\documentclass{article}
\usepackage[letterpaper, margin=1in]{geometry}
\usepackage{setspace}
\usepackage{amsmath,amssymb}
\usepackage{graphicx}
\usepackage{subcaption}
\usepackage[amsmath,thmmarks,framed]{ntheorem}
\usepackage{authblk}
\usepackage{tikz}
\usetikzlibrary{arrows.meta,calc,fit,positioning}
\usepackage{tcolorbox}
\usepackage{booktabs}
\usepackage{multirow}
\usepackage{placeins}
\usepackage{enumerate}
\usepackage{xcolor}
\definecolor{darkgreen}{rgb}{0.0,0.5,0.0}
\usepackage[colorlinks,linkcolor=red,citecolor=darkgreen,urlcolor=blue]{hyperref}
\hypersetup{
  pdftitle={Mode Coverage in Normalizing Flow Boltzmann Generators via Log-Ratio Variation},
  pdfauthor={Qi Feng, Rongjie Lai, Di Qi, Xuda Ye}
}
\usepackage[round,authoryear]{natbib}

\usepackage[ruled,linesnumbered]{algorithm2e}
\SetKwInput{KwInput}{Input}
\SetKwInput{KwOutput}{Output}

\newtheorem{remark}{Remark}
\newtheorem{lemma}{Lemma}
\newtheorem{theorem}{Theorem}
\theoremheaderfont{\normalfont\bfseries}
\theorembodyfont{\normalfont}
\newtheorem*{definition}{Definition}
\newenvironment{proof}{\par\medskip\noindent\textit{Proof.}\ }{\hfill$\square$\par\medskip}

\newcommand{\Le}{\leqslant}
\newcommand{\Ge}{\geqslant}
\newcommand{\D}{\mathrm{d}}
\newcommand{\KL}{\mathrm{KL}}
\newcommand{\X}{\mathrm{X}}
\newcommand{\ESS}{\mathrm{ESS}}

\theoremstyle{plain}

\title{Mode Coverage in Normalizing Flow Boltzmann Generators via Log-Ratio Variation}
\author[1]{Qi Feng}
\author[2]{Rongjie Lai}
\author[2]{Di Qi}
\author[2]{Xuda Ye}
\affil[1]{Department of Mathematics, Florida State University, Tallahassee, FL 32306}
\affil[2]{Department of Mathematics, Purdue University, West Lafayette, IN 47907}

\date{\today}
\begin{document}
\maketitle

\begin{abstract}
Normalizing flow Boltzmann generators retain a tractable pushforward density, but training with forward KL depends on target samples that may be biased or omit modes. As a result, a flow can miss target mass while its observed importance weights give a high effective sample size. We introduce the log-ratio variation $\X_\omega$, the mean absolute pairwise difference of the target-to-pushforward log-density ratio under a weighting measure $\omega$, and use it to define KLXX, a new loss function. Two log-ratio variations are added to the forward KL (denoted by the two X's): one weighted by the target to improve accuracy, the other by a mixture of quench and temper samples with pushforward samples to search candidate modes. We derive the Fisher--Rao gradient flow of KLXX, where both variations contribute nonpositive dissipation, and a fixed-surrogate error bound for KLXX. We use KLXX in an adaptive-staging Boltzmann generator, with importance reweighting at every stage. We bound the sampling error of its inference scheme when the stage weights are essentially bounded, and prove it asymptotically unbiased in the sample size. In the numerical tests, KLXX improves mode coverage over forward KL. It also improves the generator's per-stage diagnostics against the loss that built the schedule. The observables the generator recovers are close to independent references. The log-ratio variations thus supply information that the forward KL loss usually omits.
\end{abstract}

% ===================== Graphical abstract =====================
\par\medskip
\noindent\makebox[\textwidth][c]{%
    \resizebox{0.9\textwidth}{!}{%
        \begingroup
        \def\KLXXEmbedded{1}%
        \ifdefined\KLXXEmbedded
    \def\KLXXBeginDocument{}
    \def\KLXXEndDocument{}
\else
\documentclass[tikz,border=5pt]{standalone}

\usepackage[T1]{fontenc}
\usepackage{lmodern}
\usepackage{amsmath,amssymb}
\usepackage{xcolor}
\usepackage{tikz}
\usetikzlibrary{arrows.meta,calc,fit,positioning}

    \def\KLXXBeginDocument{\begin{document}}
    \def\KLXXEndDocument{\end{document}}
\fi

\definecolor{ink}{HTML}{26333F}
\definecolor{line}{HTML}{7A8793}
\definecolor{neutral}{HTML}{F5F7F8}
\definecolor{target}{HTML}{3B78A7}
\definecolor{targetfill}{HTML}{EAF2F8}
\definecolor{coverage}{HTML}{C98B32}
\definecolor{coveragefill}{HTML}{FCF4E8}
\definecolor{proposal}{HTML}{4E8878}
\definecolor{proposalfill}{HTML}{EBF4F1}

\tikzset{
    >={Stealth[length=2.1mm,width=1.5mm]},
    flow/.style={->, draw=line, line width=0.85pt},
    panel/.style={
        draw=line,
        fill=white,
        rounded corners=2mm,
        line width=0.7pt,
        inner sep=2.5mm,
        align=center,
        text=ink
    },
    support/.style={panel, text width=4.75cm, minimum height=1cm},
    ratio/.style={panel, fill=neutral, text width=9.6cm, minimum height=3.55cm},
    term/.style={panel, minimum height=2.65cm, inner sep=2.2mm},
    output/.style={
        panel,
        fill=neutral,
        text width=11.4cm,
        minimum height=0.85cm,
        inner sep=1.8mm
    },
    stage/.style={
        panel,
        fill=white,
        font=\small\bfseries,
        inner xsep=4mm,
        inner ysep=1.1mm
    },
    note/.style={font=\small, text=line, fill=white, inner sep=0.5pt}
}

\KLXXBeginDocument
\begin{tikzpicture}
    % math glue is fixed, so relations in text-width nodes are not stretched
    \medmuskip=4mu \thickmuskip=5mu

    \node[font=\Large\bfseries, text=ink] at (0,5.53)
        {KLXX: one forward KL, two log-ratio variations (XX)};

    \node[output] (output) at (0,4.38) {%
        \textbf{Training aim}
        \quad $\displaystyle \nu=(G^{-1})_{\#}\pi_0\approx\pi$\\[0.5mm]
        {approximate target distribution with pushforward distribution}
    };

    \node[support, draw=target, fill=targetfill] (targetbatch) at (-5.54,2.78) {%
        \textbf{Target samples $\pi$}\\[0.8mm]
        {sequential Monte Carlo}
    };

    \node[support, draw=coverage, fill=coveragefill] (qtbatch) at (-5.54,1.08) {%
        \textbf{QT coverage $\hat\pi$}\\[0.8mm]
        {\small melt+\textbf{q}uench+\textbf{t}emper}\\[-0.2mm]
        {\small search candidate modes}
    };

    \node[support, draw=proposal, fill=proposalfill, inner ysep=1.5mm] (modelbatch) at (-5.54,-0.62) {%
        \textbf{Pushforward $\bar\nu$}\\[0.6mm]
        {\small $x\sim\pi_0\xrightarrow{\;G^{-1}\;}y\sim\nu$}\\[-0.2mm]
        {\small detached from autograd}
    };

    \node[ratio] (ratio) at (3.115,1.08) {%
    	\textbf{Pushforward density} \\[5pt]
        $\displaystyle
        \nu=(G^{-1})_{\#}\pi_0
        \quad\Longrightarrow\quad
        \nu(y)=\pi_0\bigl(G(y)\bigr)\,\lvert\det J_G(y)\rvert
        $ \\[7pt]
        \textbf{Log-ratio term} \\[5pt]
        $\displaystyle
            z(y):=\log\frac{\pi(y)}{\nu(y)}
            =U_0\bigl(G(y)\bigr)-U(y)-\log\lvert\det J_G(y)\rvert$
    };

    \draw[flow, draw=target]
        (targetbatch.east) -- ([yshift=17mm]ratio.west);
    \draw[flow, draw=coverage]
        (qtbatch.east) -- (ratio.west);
    \draw[flow, draw=proposal]
        (modelbatch.east) -- ([yshift=-17mm]ratio.west);

    \draw[flow] (output.south -| ratio.north) --
        node[right=1mm,note] {train flow $G$} (ratio.north);

    \node[stage] (losslabel) at (0,-1.45)
        {KLXX loss};
    \coordinate (ratioout) at ($(ratio.south)+(0,-3mm)$);
    \coordinate (ratioturn) at (losslabel.north |- ratioout);
    \draw[draw=line, line width=0.85pt] (ratio.south) --
        node[right=1mm,note] {evaluate $z$ on each batch} (ratioout);
    \draw[draw=line, line width=0.85pt] (ratioout) -- (ratioturn);
    \draw[flow] (ratioturn) -- (losslabel.north);

    \node[term, draw=target, fill=targetfill, text width=3.20cm]
        (meanterm) at (-6.345,-3.45) {%
        \textbf{Target mean}\\[1.2mm]
        $\displaystyle \mathbb E_{y\sim\pi}[z(y)]$\\[1.2mm]
        {\small target fit}
    };

    \node[font=\Large, text=line] at (-4.27,-3.45) {$+$};

    \node[term, draw=target, fill=targetfill, text width=3.90cm]
        (targetterm) at (-1.845,-3.45) {%
        \textbf{Target variation}\\[1.1mm]
        $\displaystyle
        \mathbb E_{y,y'\sim\pi}
        \big[|z(y)-z(y')|\big]
        $\\[1.1mm]
        {\small target regularization}
    };

    \node[font=\Large, text=line] at (0.60,-3.45) {$+$};

    \node[term, text width=6.85cm] (mixterm) at (4.52,-3.45) {%
        \textbf{Mixture variation: \textcolor{coverage}{QT}
        $+$ \textcolor{proposal}{pushforward}}\\[0.7mm]
        $\displaystyle
        \mathbb E_{y,y'\sim\xi}
        \big[|z(y)-z(y')|\big],~~ \xi=\frac{\textcolor{coverage}{\hat\pi}
        	+\textcolor{proposal}{\bar\nu}}{2}
        $\\[0.9mm]
        {\small mode coverage $\cdot$ pushforward support}
    };

    \node[fit=(meanterm)(targetterm)(mixterm), inner sep=0pt] (lossrow) {};
    \coordinate (lossbus) at ($(lossrow.north)+(0,2.2mm)$);
    \draw[draw=line, line width=0.75pt] (losslabel.south) -- (lossbus);
    \draw[draw=line, line width=0.75pt]
        (meanterm.north |- lossbus) -- (mixterm.north |- lossbus);
    \draw[flow] (meanterm.north |- lossbus) -- (meanterm.north);
    \draw[flow] (targetterm.north |- lossbus) -- (targetterm.north);
    \draw[flow] (mixterm.north |- lossbus) -- (mixterm.north);
\end{tikzpicture}
\KLXXEndDocument
        \endgroup
    }%
}
\par\medskip

% ===================== Section 1: Introduction =====================
\section{Introduction}
\label{sec: intro}

Sampling from a target distribution $\pi \propto \exp(-U)$ on $\mathbb R^d$ is a recurring problem in many fields including molecular simulation, Bayesian inference, and statistical physics \citep{frenkel2002understanding,robert2004monte,landau2014guide}. When $\pi$ is high-dimensional and multimodal, its mass splits among basins of $U$ separated by barriers. Classical Monte Carlo methods are usually local. Metropolis--Hastings \citep{metropolis1953equation,hastings1970monte}, Langevin dynamics \citep{langevin1908theorie,leimkuhler2015molecular}, Nos{\'e}--Hoover dynamics \citep{nose1984unified,hoover1985canonical}, and Hamiltonian Monte Carlo \citep{duane1987hybrid} all evolve from the initial configuration, and over an accessible simulation time they can stay in the basin where they started. When this occurs, finite-time estimates are biased toward the visited basin and give no direct account of the mass that was never reached. Tempering and particle methods widen the reach by connecting a tractable distribution to $\pi$ through intermediate distributions, but a basin that none of their chains has visited still cannot enter the target ensemble \citep{woodard2009conditions,chehab2024provable}. Every one of these samplers learns the target only where it has been, and the mass it never reached is invisible to it.

A more universal strategy is to learn a global transport from an easily sampled reference to the target, which generates independent, nonlocal proposals. Boltzmann generators do this with a learned transport map \citep{noe2019boltzmann}, dropping the long-time Monte Carlo iterations entirely. Let $\pi_0\propto\exp(-U_0)$ be an easily sampled reference distribution, usually Gaussian. We train an invertible normalizing flow $G$ \citep{rezende2015variational,papamakarios2017masked,kingma2018glow} on the domain of the target, written $\mathbb R^d$ throughout and a box or a torus in the experiments below, to fit the target-to-source transport relation $G_{\#}\pi\approx\pi_0$. The inverse map then sends independent source draws toward the target, and their pushforward distribution is $\nu:=(G^{-1})_{\#}\pi_0$. Its density is
\begin{equation}
    \nu(y)=\pi_0(G(y))\left|\det J_G(y)\right|,
    \label{eq: pushforward-density}
\end{equation}
where $J_G(y)$ is the Jacobian matrix of $G$ at $y$. The importance weight $w(y):=\pi(y)/\nu(y)$ is the ratio of the target density to the pushforward density. Reweighting by $w$ corrects the estimates from pushforward samples, so an imperfect flow still gives the right answer wherever it places samples. What the weights cannot correct is a region the flow never reaches, and that is the failure this paper is about. Throughout the paper we also use the \emph{log-ratio}:
\begin{equation}
	z(y) := \log w(y) = \log \frac{\pi(y)}{\nu(y)} =  U_0(G(y)) - U(y) - \log|\det J_G(y)| + \mathrm{const}.
	\label{eq: log-ratio}
\end{equation}
The same weights define the effective sample size (ESS), a diagnostic of density agreement on the support reached by $\nu$ \citep{liu2001monte,doucet2001introduction,martino2017rethinking}. A high ESS indicates accuracy in the regions the flow reaches, but says nothing about the modes of $U$ it does not reach. We call this phenomenon a \emph{fake ESS}; it is also reported in \citet{noe2019boltzmann} and \citet{midgley2022flow}.

The fake ESS can arise from various loss functions. The original Boltzmann generator \citep{noe2019boltzmann} and its variational-inference predecessor \citep{rezende2015variational} minimize the reverse Kullback--Leibler (KL) divergence
\begin{equation}
    \text{(reverse)}\quad \KL(\nu\|\pi) = \int_{\mathbb R^d} \nu(y) \log \frac{\nu(y)}{\pi(y)} \D y \,= -\,\mathbb E_{y\sim\nu}\bigl[z(y)\bigr],
    \label{eq: reverse KL}
\end{equation}
which needs samples only from $\nu$ and evaluations of $U$. The empirical gradient is therefore accessible, but the divergence is mode-seeking \citep{wainwright2008graphical}: a mode absent from $\nu$ contributes no samples and hence no gradient. A flow can then collapse onto a subset of the target modes \citep{wu2020stochastic,felardos2023designing}. Existing approaches interleave Markov chain Monte Carlo (MCMC) kernels with deterministic layers, encode target symmetries in the architecture, reduce gradient variance near the optimum, or anneal from a smoother potential \citep{wu2020stochastic,kohler2020equivariant,vaitl2022gradients,schopmans2025temperature}. These changes can improve exploration, but their success still depends on whether the resulting sampler reaches each relevant mode.

The opposing loss minimizes the forward KL,
\begin{equation}
    \text{(forward)}\quad \KL(\pi\|\nu) = \int_{\mathbb R^d} \pi(y) \log \frac{\pi(y)}{\nu(y)} \D y \,= \mathbb E_{y\sim\pi}\bigl[z(y)\bigr],
    \label{eq: forward KL}
\end{equation}
which is mass-covering for exact distributions: it diverges if $\nu$ vanishes where $\pi$ has positive mass \citep{wainwright2008graphical}. Its empirical form, however, needs samples from the target distribution $\pi$. Existing work resolves this circularity in several ways. Maximum likelihood uses configurations from a reference thermodynamic state \citep{wirnsberger2020targeted}; adaptive flow samplers train on a running Markov chain and use the flow as a global proposal \citep{gabrie2022adaptive}; molecular flows can fit a molecular-dynamics trajectory \citep{klein2023equivariant}; and transferable architectures amortize coverage across related molecules \citep{klein2024transferable}. Persistent chains and coupled-target regression provide other approximations to the missing target expectation \citep{naesseth2020markovian,rehman2026regflow}. Because an empirical loss reflects only the regions represented by its sample source, the mass-covering property of the exact divergence does not prevent empirical mode omission. This tension motivates a loss that retains forward KL while using samples from more than one distribution.

The sampled forward KL leaves three distinct gaps. First, a mode absent from its target surrogate contributes no training points. If the flow also misses that mode, the observed weights can be nearly uniform on the reached support, so the ESS can be close to one even though target mass is absent, which is the fake ESS above. Second, a finite or biased target surrogate distorts the sampled loss, and the fitted flow can inherit that error. Third, target samples give little information where the pushforward distribution places excess mass, allowing leakage between modes or into target-poor regions. Addressing one of these failures does not by itself resolve the other two.

Previous remedies change the loss to favor discovery: flow annealed importance sampling bootstrap (FAB) \citep{midgley2022flow} minimizes the $2$-divergence, whose stronger tail penalty is estimated by annealing toward $\pi^2/\nu$. This paper overcomes mode collapse in a different way, keeping the forward KL and adding what it lacks. First, a quench and temper (QT) construction finds candidate modes of $U$ beyond the reach of the source. It takes three steps. Each source sample is melted by Gaussian noise, which carries it past the region the source covers; quenched by gradient descent on $U$, which drives it into a local basin; and tempered by a short Langevin run, which spreads it around that basin. Second, a \emph{log-ratio variation}, the mean absolute pairwise difference of the target-to-pushforward log-density ratio under a chosen weighting measure, carries the mode information into the loss: the log-ratio is constant on the support of that measure exactly when the pushforward matches the target there, and pairwise differences cancel the unknown normalizers. Two such variations are appended to the forward KL, one weighted by the target to improve accuracy, the other weighted by the mixture of the QT distribution and the detached pushforward distribution to search candidate modes and to restrain leakage between them. The result is the KLXX loss, named for its two X's; Section~\ref{sec: construction} gives the construction. Our contributions are as follows:

\begin{enumerate}
	\setlength{\itemsep}{0pt}
	\item We introduce a new KLXX loss \eqref{eq: KLXX-G}, whose components mix QT and pushforward distributions. Pairwise differences remove the unknown target normalizer, and the log-ratio variation retains $\nu=\pi$ as a minimizer. The target variation reuses the forward KL batch, while the mixture variation supplies candidate-mode and pushforward support information absent from that batch.

	\item We derive the Fisher--Rao flow of the generalized KLXX loss under a general choice of parameters, and identify the dissipation it induces: the two log-ratio variations enter as nonpositive terms, and as strictly negative ones at positive coefficients unless the importance weight is already constant. For a biased target surrogate, the two log-ratio variations improve the KL bound at the stationary distribution of the biased KLXX loss, and Theorem~\ref{thm: accuracy} exhibits that improvement under mild conditions.

	\item We demonstrate that QT can reach target basins outside the regions covered by the source and by the sequential Monte Carlo (SMC) surrogate. In the Himmelblau and Sparse tests, forward KL, $\KL$+$\X_\pi$ and FAB each miss modes, while both QT-based losses recover every mode; mixing QT and pushforward samples also reduces visible intermodal leakage (Figures~\ref{fig: himmelblau} and~\ref{fig: sparse}). We therefore report ESS together with coverage, since uniform weights on reached modes do not imply complete coverage.

	\item We build an adaptive-staging Boltzmann generator based on KLXX that uses only oracle access to the target potential $U(x)$ and its gradient $\nabla U(x)$, with SMC and QT generating the required samples on the fly. We compare KLXX against $\KL$+$\X_\pi$ on this construction, and measure the weight degeneracy by the product of its per-stage ESS factors. Each training step draws its own samples and never differentiates through an MCMC transition, since the mixture measure is detached. We also bound the error of the sample set it returns. Following the standard non-asymptotic analysis \citep{delmoral2004feynman,chopin2020introduction}, Theorem~\ref{thm: propagation} controls the error at rate $N^{-1/2}$ for every bounded observable with a constant independent of $N$, so the scheme is asymptotically unbiased, provided each stage weight is essentially bounded. Reading that constant stagewise defines the propagation factors, which summarize the weight degeneracy of the composed stage maps; the experiments report the computable factor $\hat F$, assembled from the per-stage ESS.
\end{enumerate}

%\paragraph{Organization.}
The rest of the paper is organized as follows:
Section~\ref{sec: construction} constructs the KLXX loss and compares it with related work. Section~\ref{sec: FR} analyzes the generalized KLXX loss in Fisher--Rao geometry, deriving its gradient flow and the accuracy bound it reaches under a biased target surrogate; Appendices~\ref{sec: appendix-FR} and~\ref{sec: appendix-accuracy} give the proofs. Section~\ref{sec: training} gives the training algorithm. Section~\ref{sec: boltzmann} proposes the adaptive-staging Boltzmann generator and bounds the error of the sample set it returns; Appendix~\ref{sec: appendix-propagation} gives the proof. Section~\ref{sec: experiments} reports the numerical results. Section~\ref{sec: conclusions} gives the conclusion.

\section{Construction of the KLXX loss}
\label{sec: construction}

\subsection{One forward KL, two log-ratio variations}
\label{subsec: construction}

The KLXX loss is the forward KL \eqref{eq: forward KL} augmented by two \emph{log-ratio variations}, defined for a weighting distribution $\omega$ by
\begin{equation}
	\X_\omega(\pi\|\nu) = \mathbb E_{y,y'\sim \omega} \big[|z(y)-z(y')|\big] = \iint_{\mathbb R^d\times\mathbb R^d} \omega(y)\,\omega(y')
	\biggl|\log\frac{\pi(y)}{\nu(y)} - \log\frac{\pi(y')}{\nu(y')}\biggr| \D y\D y'.
	\label{eq: X}
\end{equation}
This is the mean absolute pairwise difference of the log-ratio $z(y)$ in \eqref{eq: log-ratio}, with both points drawn independently from $\omega$; for a real random variable the functional is the Gini mean difference \citep{yitzhaki2003gini}. The pairwise difference cancels the unknown normalizers, and $\X_\omega(\pi\|\nu)=0$ exactly when $z$ is $\omega$-almost surely constant. The normalizing constants of \eqref{eq: log-ratio} are therefore omitted: they depend on neither $y$ nor $G$, so they shift the forward KL by a fixed amount. The two variations are:

\textbf{Target variation}: $\X_\omega$ with $\omega = \pi$, the target distribution. It is evaluated on the target samples the forward KL already draws.

\textbf{Mixture variation}: $\X_\omega$ with $\omega = (\hat\pi + \bar\nu)/2$, the equal mixture of the QT distribution $\hat\pi$ and the detached pushforward distribution $\bar\nu$, which is identical to $\nu$ but does not enter the differentiation graph. The QT distribution is built from the source in three steps:
	\begin{equation*}
		y\sim \pi_0 \longrightarrow \boxed{\small
		\begin{gathered}
			\text{melt: scatter with noise} \\
			y = y + m_e \zeta, \quad \zeta \sim \mathcal N(0,I_d)
		\end{gathered}
		} \longrightarrow \boxed{\small
		\begin{gathered}
			\text{quench: optimization} \\
			\dot y = -\nabla U(y)
		\end{gathered}
		} \longrightarrow \boxed{\small
		\begin{gathered}
			\text{temper: Langevin} \\
			\dot y = -\nabla U(y) + \sqrt{2}\dot B_t
		\end{gathered}
		} \longrightarrow y\sim \hat \pi
	\end{equation*}
Each sample $y$ from the source distribution $\pi_0$ is scattered by Gaussian noise at a melt scale $m_e>0$, driven into a local basin of $U$ by gradient descent, and spread around that basin by a short Langevin run. Section~\ref{subsec: QT} gives the discrete construction.

In terms of the flow map $G$, the KLXX loss reads
\begin{align}
	\mathrm{KLXX}[G] & = \KL(\pi\|\nu) + \X_\pi(\pi\|\nu) + \X_{(\hat\pi+\bar\nu)/2}(\pi\|\nu) \notag \\
	& = \mathbb E_{y\sim\pi}\big[z(y)\big] + \mathbb E_{y,y'\sim\pi}\big[|z(y)-z(y')|\big] + \mathbb E_{y,y'\sim(\hat\pi+\bar\nu)/2}\big[|z(y)-z(y')|\big],
	\label{eq: KLXX-G}
\end{align}
where the log-ratio of \eqref{eq: log-ratio}, written out again for convenience,
\begin{equation*}
	z(y) = \log\frac{\pi(y)}{\nu(y)} = U_0(G(y)) - U(y) - \log|\det J_G(y)| ,
\end{equation*}
depends explicitly on $G$, while the detached pushforward distribution $\bar\nu$ requires $G^{-1}$ evaluations that do not enter the differentiation graph.

The KLXX loss can be generalized through its parameters, by changing the coefficients of the two log-ratio variations and the proportions of the mixture. The generalized KLXX loss, determined by a set of parameters $(\lambda,\vartheta,\alpha,\beta)$, reads
\begin{align}
	\mathcal L[G] & = \KL(\pi\|\nu) + \lambda\,\X_\pi(\pi\|\nu) + \vartheta\,\X_{\alpha\hat\pi+\beta\bar\nu}(\pi\|\nu) \notag \\
	& = \mathbb E_{y\sim\pi}\big[z(y)\big] +  \lambda\, \mathbb E_{y,y'\sim\pi}\big[|z(y)-z(y')|\big] + \vartheta\, \mathbb E_{y,y'\sim \alpha\hat\pi+\beta\bar\nu}\big[|z(y)-z(y')|\big],
	\label{eq: KLXX-general}
\end{align}
with $\lambda,\vartheta,\alpha,\beta\Ge0$ and $\alpha+\beta=1$, so that the mixture $\xi = \alpha\hat\pi+\beta\bar\nu$ is again a probability distribution. Taking $(\lambda,\vartheta,\alpha,\beta)=(1,1,\tfrac12,\tfrac12)$ recovers \eqref{eq: KLXX-G}; $(\lambda,\vartheta)=(0,0)$ gives the plain forward KL, $(1,0)$ gives $\KL$+$\X_\pi$, and $(1,1,1,0)$ weights the second variation by the QT distribution alone.

\subsection{Comparison to related works}
\label{subsec: loss-related}

On the Monte Carlo side, the tempering and particle methods named in Section~\ref{sec: intro} are built on a sequence of intermediate distributions. Parallel tempering (PT) and replica exchange run copies of the system at several temperatures and exchange neighboring configurations, so a basin reached at high temperature can enter the low-temperature ensemble \citep{swendsen1986replica,geyer1991markov,hukushima1996exchange}. Annealed importance sampling (AIS) moves along a bridge of intermediate distributions while accumulating importance weights \citep{neal2001annealed}, and SMC uses a similar bridge but resamples and rejuvenates a weighted particle set at each level \citep{doucet2001introduction,delmoral2006sequential}. Other classical routes bias a chosen reaction coordinate \citep{torrie1977umbrella,kastner2011umbrella} or rewrite the normalizing integral as a one-dimensional quadrature \citep{skilling2006nested}.

Recent work improves PT by better swap schedules and temperature grids \citep{syed2022nonreversible}, by adapting the reference endpoint and the connecting path \citep{surjanovic2022variational}, by neural transports between adjacent distributions \citep{zhang2025accelerated}, and by carrying replica exchange to stochastic-gradient samplers \citep{chen2019accelerating,deng2020nonconvex,deng2021variance,dong2021replica}. These accelerations move a discovered basin through the temperature grid more efficiently, but two limits remain: a basin that no chain has visited cannot enter the target ensemble \citep{woodard2009conditions,chehab2024provable}, and the mass of a visited basin is still estimated through occupation frequencies rather than reweighting.

On the normalizing flow side, the highlight of the KLXX loss \eqref{eq: KLXX-G} is the mixture of the QT distribution $\hat\pi$ with the detached pushforward distribution $\bar\nu$, which remedies mode collapse. The following ingredients draw on several ideas already present in normalizing flow training, and we compare them here.

The closest related loss is that of FAB \citep{midgley2022flow}, which couples AIS and flow training but replaces the forward KL by the $\alpha$-divergence at $\alpha=2$. Its integrand is proportional to $\pi^2/\nu$, and thus the training minimizes the importance-weight variance. FAB targets $\pi^2/\nu$ with AIS and reuses those samples through a prioritized replay buffer that must be maintained across training; its MCMC transitions carry the gradient of the flow density. KLXX instead retains the forward KL and adds nonnegative log-ratio variations, so its additional cost is the QT construction and the variation evaluations under the mixture measure. The $2$-divergence gives the stronger tail penalty, whereas KLXX obtains its coverage information from QT.

Centered log-ratio penalties also arise in log-variance losses for learned diffusion samplers \citep{richter2024improved} and in log-dispersion regularization (LDR) losses for Boltzmann generators \citep{schopmans2026ldr}. The present construction therefore belongs to a broader family of log-ratio regularizers. LDR-L1 uses the centered penalty $\mathbb E_\pi|z-\mathbb E_\pi z|$, and Jensen's inequality with the triangle inequality gives $\mathbb E_\pi|z-\mathbb E_\pi z|\Le\X_\pi(\pi\|\nu)\Le2\mathbb E_\pi|z-\mathbb E_\pi z|$, so the two agree within a factor of two at the same coefficient. They are not equivalent: $\X_\pi$ pairs the forward KL integrand at two independent target points, and KLXX adds a second variation under the QT and pushforward mixture rather than under fixed energy-labeled data.

A second family trains one flow per annealing step inside an SMC sweep. Annealed flow transport (AFT) \citep{arbel2021annealed} trains each flow by the KL divergence between the transported distribution and the next annealed target. That KL is estimated on the current weighted particles, which are then transported through the flow and resampled. Its continual repeated variant, continual repeated annealed flow transport (CRAFT) \citep{matthews2022craft}, repeats the training across successive sweeps instead of fitting each flow once. These are the closest prior works to the staged construction of Section~\ref{sec: boltzmann}, whose stage loss carries the same forward KL term; what KLXX adds at each stage are the two log-ratio variations, and in particular the mixture of QT and pushforward samples that a per-stage KL does not supply.

Constrained mass transport keeps the overlap of neighboring intermediate distributions high by constraining their KL divergence and entropy decay \citep{vonklitzing2026constrained}; the stage selection of Section~\ref{sec: boltzmann} addresses the same difficulty through the ESS gate. Sequential Boltzmann generators fit one flow to biased molecular dynamics samples and anneal toward the target at inference time \citep{tan2025sequential}, so their transport acts after training, whereas the surrogate of Section~\ref{subsec: annealing} is built during it.

A separate family replaces the loss altogether. Energy-only methods such as iDEM alternate a sampling loop with a score- or velocity-fitting loop and a replay buffer \citep{akhound2024iterated,woo2024iefm}. Adjoint sampling learns a diffusion sampler from the energy alone by stochastic optimal control \citep{havens2025adjoint}. Controlled stochastic differential equations, denoising diffusions, and flow-matching models learn a drift or velocity rather than a bijection \citep{zhang2022path,vargas2023denoising,berner2024optimal,nusken2021solving,richter2024improved}. Annealing enters this family as well: NETS augments AIS with a learned drift \citep{albergo2024nets}, and a temperature sequence of diffusion models can be annealed at inference time \citep{akhoundsadegh2025progressive}. These stochastic transports can be more expressive than the invertible architectures used here, but their terminal density often requires per-sample ODE integration \citep{he2025notrick}. We retain a bijection because the exact pushforward density, importance ratio, ESS, and reweighting are central to every term of \eqref{eq: KLXX-G} and to the final estimator. Table~\ref{tab: loss-methods} summarizes these distinctions.

A further family keeps an exact likelihood but drops the bijection. Autoregressive Boltzmann generators (ArBG) factorize the density as $p(x)=\prod_j p(x_j\mid x_{<j})$, so the likelihood needs neither a Jacobian determinant nor an ODE solve, and no map is inverted \citep{rehman2026autoregressive}. They are trained by likelihood on molecular dynamics data, whereas the generator of Section~\ref{sec: boltzmann} uses no external samples.

\begin{table}[htb]
\centering
{\small
\begin{tabular}{@{}lll@{}}
\toprule
method & loss & training mechanism \\
\midrule
reverse KL & $\KL(\nu\|\pi)$ & pushforward samples; no target sampler \\
forward KL & $\KL(\pi\|\nu)$ & SMC target surrogate in this work \\
FAB & $2$-divergence ($\pi^2/\nu$) & annealing on $\pi^2/\nu$; MCMC needs $\nabla_y\log\nu$; replay buffer \\
LDR & $\KL(\pi\|\nu)+\mathbb E_\pi[|z-\mathbb E_\pi z|^p]$ & fixed energy-labeled data; no extra pushforward samples \\
AFT, CRAFT & $\KL$ per annealing step & one flow per step, trained on the weighted particles \\
iDEM & energy matching & iterative sampler; stochastic transport; replay buffer \\
ArBG & maximum likelihood & autoregressive model; molecular dynamics data \\
\multirow{2}{*}{KLXX} &
\multirow{2}{*}{$\KL(\pi\|\nu)$+$\X_\pi$+$\X_{(\hat\pi+\bar\nu)/2}$} &
target annealing, QT, and detached pushforward samples \\
& & no $\nabla_y\log\nu$ in the chains \\
\bottomrule
\end{tabular}
}
\caption{Losses and training mechanisms for the methods discussed above. Here $z=\log(\pi/\nu)$; the LDR row gives its target-weighted form with $p=1$ for LDR-L1 and $p=2$ for LDR-L2. The forward KL row uses the SMC target surrogate of this paper but may instead use external target data. The FAB row describes the reference implementation; the runs of Section~\ref{sec: experiments} keep no replay buffer. The table compares how samples enter each loss; it does not assume equal computational budgets.}
\label{tab: loss-methods}
\end{table}

Of the methods above, we run FAB alongside forward KL and the log-ratio variants; the other families are left to future work, for the following reasons. Training without external data is itself a constraint: much forward KL work fits a flow to target samples produced beforehand, so a method given those samples and one given only $U$ and $\nabla U$ are solving different problems. Molecular flows are fitted to molecular dynamics trajectories \citep{klein2023equivariant,klein2024transferable}, sequential Boltzmann generators to biased molecular dynamics samples \citep{tan2025sequential}, and autoregressive Boltzmann generators by likelihood on peptide simulation datasets \citep{rehman2026autoregressive}; LDR is trained either on simulation datasets or purely variationally \citep{schopmans2026ldr}. The remaining candidates differ in what their training loop must carry. FAB estimates the $2$-divergence, whose gradient has high variance when it is estimated from flow samples. Its annealed importance sampling and prioritized replay buffer are designed to control that variance \citep{midgley2022flow}, and both add per-step cost. Its chains also target $\pi^2/\nu$, so every MCMC step needs the positional gradient of the flow density. We therefore run FAB on the 2D benchmarks, the product multi-well and the $\phi^4$ lattices, without the replay buffer and with a rejuvenation kernel that leaves each annealing level invariant, as Section~\ref{subsec: annealing} describes. An independent study reports FAB reaching comparable accuracy on alanine di- and tetrapeptide at roughly three times the target energy evaluations, and resolving the metastable states of its largest system only partially \citep{schopmans2025temperature}. iDEM carries a replay buffer of the same kind and learns a diffusion sampler, so no exact pushforward density is available for importance reweighting \citep{akhound2024iterated}. By the two-sided bound derived above, LDR-L1 is within a factor of two of $\KL$+$\X_\pi$ at the same coefficient. Two losses with comparable values can still have different gradients and reach different flows, so the bound is not a substitute for running LDR-L1; Section~\ref{subsec: highd} runs it on the product multi-well. The mixture variation is in any case not a competitor to them: it is a way of feeding known candidate modes into a loss, and it can be added to any flow-based sampler for a given target, whatever the architecture and whatever the principal loss.

\section{Analysis of the generalized KLXX loss}
\label{sec: FR}

This section analyzes the generalized KLXX loss \eqref{eq: KLXX-general}. For convenience, we write it as a functional of the pushforward distribution $\nu$ rather than of the flow map $G$:
\begin{align}
	\mathcal L[\nu] & =  \KL(\pi\|\nu) + \lambda\,\X_\pi(\pi\|\nu) + \vartheta\,\X_{\alpha\hat\pi+\beta\bar\nu}(\pi\|\nu) \notag \\
	& = \mathbb E_{y\sim\pi}\big[z(y)\big] +  \lambda\, \mathbb E_{y,y'\sim\pi}\big[|z(y)-z(y')|\big] + \vartheta\, \mathbb E_{y,y'\sim \alpha\hat\pi+\beta\bar\nu}\big[|z(y)-z(y')|\big],
	\label{eq: KLXX-general-nu}
\end{align}
the same loss as \eqref{eq: KLXX-general} with the same conditions on the parameters. Note that the log-ratio $z = \log(\pi/\nu)$ in \eqref{eq: log-ratio} carries the entire dependence on $\nu$. Throughout the paper, we use one Greek symbol ($\pi$, $\nu$, $\mu$, \emph{etc.}) to denote both a probability distribution and its density with respect to Lebesgue measure on $\mathbb R^d$.

The three subsections below take the generalized loss \eqref{eq: KLXX-general-nu} in turn: an account of how the mixture variation contributes to mode discovery, the Fisher--Rao flow of the loss and the dissipation it carries, and an accuracy bound under a biased target surrogate.

\subsection{Principle of mixture variation}
\label{subsec: mixture-variation}

The mixture distribution $\xi = \alpha\hat\pi+\beta\bar\nu$ mainly relies on the QT distribution $\hat\pi$ to search the candidate modes of the potential $U$, as described in Section~\ref{sec: intro}. The QT distribution $\hat\pi$ is generated from the source distribution $\pi_0$ in three steps: for each $y\sim \pi_0$,
\begin{enumerate}
	\setlength{\itemsep}{0pt}
	\item \textbf{Melt}: scatter with Gaussian noise, $y = y+m_e\zeta$, with $\zeta\sim \mathcal N(0,I_d)$, at a melt scale $m_e>0$ that carries $y$ beyond the region the source covers;
	\item \textbf{Quench}: descend the potential, $\dot y = -\nabla U(y)$, until $y$ settles at a local minimizer of $U$, that is, at the center of whichever basin the melted particle landed in;
	\item \textbf{Temper}: run a short Langevin trajectory, $\dot y = -\nabla U(y)+\sqrt 2\,\dot B_t$, which spreads $y$ around that center at the target temperature.
\end{enumerate}
The three names are taken directly from industrial steelmaking, where a workpiece is melted, quenched, and then tempered in that order. The same three operations appear in the energy-landscape literature: melt and quench together are basin hopping \citep{wales1997global}, and the quench is the steepest-descent map carrying a configuration to the local minimum known as its inherent structure \citep{stillinger1982hidden}.

Only melting carries the coverage: a large $m_e>0$ carries the particles beyond the source and into basins the source never visits, and every basin the melted cloud touches receives particles. Quenching then uses a local gradient descent to drive these particles into the potential basins. Tempering finally restores a spread around each center, so that $\hat\pi$ is a distribution over configurations rather than point masses at the minimizers.

We illustrate the construction on a one-dimensional example. Define the 1D Rastrigin potential
\begin{equation*}
	U(x) = \frac12x^2+4\cos(2\pi x), \qquad x\in \mathbb R,
\end{equation*}
take the source distribution $\pi_0 = \mathcal N(0,1)$ and the melt scale $m_e=2.0$, and temper for a Langevin time of $0.1$ to obtain the QT distribution $\hat\pi$. Figure~\ref{fig: 1d-qt} shows $\hat\pi$ against the target $\pi\propto\exp(-U)$.

Every mode of $\pi$ is covered by $\hat\pi$, including the outer wells the source never reaches, but the probability assigned to each differs from the target one: $\hat\pi$ weights a well by the melted mass falling into it, so its outer wells carry far more than $\pi$ gives them. This is exactly complementary to the forward KL, which is accurate where its target samples land and blind elsewhere, whereas $\hat\pi$ is inaccurate everywhere and blind only where the melted cloud never lands. The difference form of the log-ratio variation \eqref{eq: X} is what lets $\hat\pi$ be used without corrupting the fit to $\pi$: a variation asks only that the log-ratio be flat under its weighting measure, so $\hat\pi$ enters as the points at which flatness is demanded, never as an estimate of the mass they carry.

\begin{figure}[htb]
    \centering
    \includegraphics[width=0.55\textwidth]{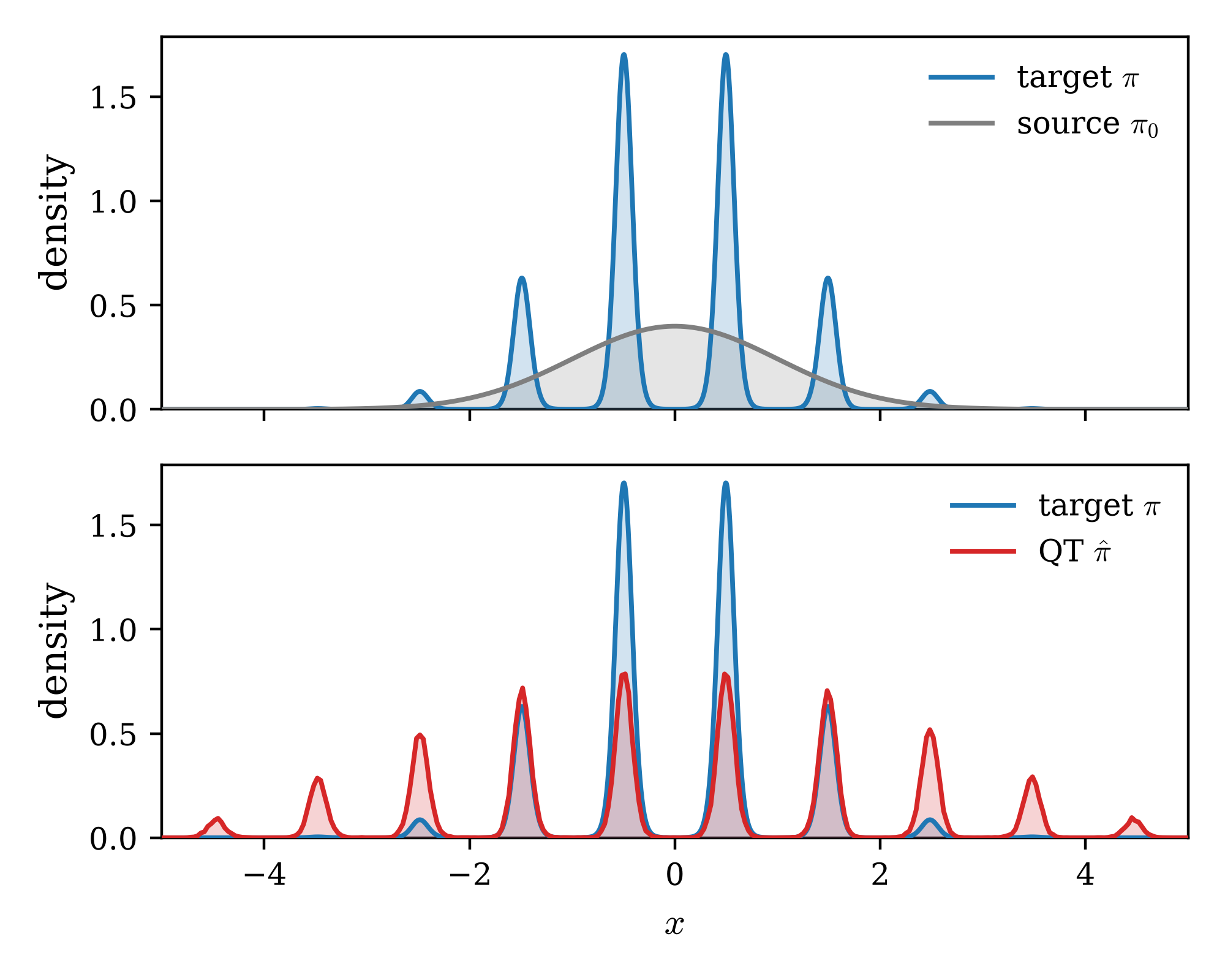}
    \caption{Quench and temper on the 1D Rastrigin target $\pi\propto\exp(-U)$ with $U(x)=\frac12x^2+4\cos(2\pi x)$. Top: the Gaussian source $\pi_0=\mathcal N(0,1)$. Bottom: the QT distribution $\hat\pi$, built from $10^5$ source samples at melt scale $m_e=2.0$ and tempered over a Langevin time of $0.1$.}
    \label{fig: 1d-qt}
\end{figure}%

However, using the QT distribution $\hat\pi$ alone in the log-ratio variation may cause \emph{intermodal leakage}: the trained flow places spurious mass between the target modes. The Sparse target shows this directly, where $\KL$+$\X_\pi$+$\X_{\hat\pi}$ recovers all four modes but leaves visible mass in the intermodal regions (Figure~\ref{fig: sparse}). The reason is that $\X_{\hat\pi}$ weights only the points $\hat\pi$ produces, and those sit in the basins of $U$; it therefore says nothing about the regions the pushforward $\nu$ occupies and $\hat\pi$ does not, which are exactly the intermodal ones. Combining it with the detached pushforward distribution $\bar\nu$ is exactly the treatment. In the log-ratio variation with the mixture distribution $\xi =\alpha\hat\pi+\beta\bar\nu$, the log-ratio $z=\log\frac{\pi}{\nu}$ is penalized to be a constant under both $\hat\pi$ and $\bar\nu$, hence during training we have approximately
\begin{equation}
	z(y) = \log\frac{\pi(y)}{\nu(y)} \approx z(y') = \log\frac{\pi(y')}{\nu(y')}, \qquad
	y\sim \hat\pi,~~y'\sim\nu.
	\label{eq: approx} 
\end{equation}
The relation \eqref{eq: approx} reflects a fundamental collision during training. On the one hand, $y$ lies in a target mode, one of those $\hat\pi$ supplies, and secures a large density $\pi(y)$, so $z(y)$ grows to $+\infty$ if the pushforward $\nu$ does not reach that mode. On the other hand, $y'$ lies wherever the flow sends mass and secures a positive density $\nu(y')$, so $z(y')$ decreases to $-\infty$ if the pushforward reaches the intermodal regions, where $\pi$ is negligible. Demanding that the two agree sets the divergences against each other, and the flow can satisfy both only by reaching every mode and staying out of the intermodal regions. The log-ratio variation with the mixture $\xi = \alpha\hat\pi+\beta\bar\nu$ resolves the collision: $\hat\pi$ supplies the points at which $z$ must not run high, $\bar\nu$ supplies the points at which it must not run low, and the variation balances the log-ratio across both, thus reducing the intermodal leakage.

\begin{remark}
	We do not train with the log-ratio variation alone. A variation constrains only the flatness of $z$ pointwise, while the forward KL \eqref{eq: forward KL} drives the flow globally: by Jensen's inequality it is nonnegative and zero only at $\nu=\pi$. Appendix~D.2 of the LDR study \citep{schopmans2026ldr} reports the analogous failure: training with LDR-L1 or LDR-L2 alone, without a KL term, produces qualitatively incorrect distributions.
\end{remark}

\begin{remark}
	The melting step is a design choice. Gaussian noise suits a target with Gaussian tails, and on a compact domain the same scatter is used with the measure of that domain. It can be wasteful when the target concentrates near a low-dimensional manifold, since most of the scattered samples land in the ambient space. Melt and quench can then be replaced by a particle dynamics carrying a pairwise repulsion, which spreads the particles apart instead of scattering each one on its own. The repulsive term of Stein variational gradient descent \citep{liu2016stein} is of this kind, and the history-dependent bias of metadynamics \citep{laio2002escaping} serves the same purpose in molecular simulation.
\end{remark}

\subsection{Fisher--Rao gradient of the KLXX loss}
\label{subsec: FR-convergence}

For an arbitrary parameter choice $(\lambda,\vartheta,\alpha,\beta)$, the generalized KLXX loss \eqref{eq: KLXX-general-nu} attains its minimum $0$ at the exact target $\pi$, so we ask what the gradient flow it induces dissipates, and how the two variations enter that dissipation.
In the Wasserstein-2 geometry, the forward KL is not displacement convex, so the Otto--Villani and Bakry--{\'E}mery arguments used for the reverse KL do not apply \citep{otto2000generalization,bakry2014analysis}. We therefore work in the Fisher--Rao (FR) geometry, where the gradient flow has an explicit form. Assume all densities are positive and regular enough for differentiation.

Let $\omega$ be a weighting distribution that is fixed under the variation, and define the \emph{nonlocal correlation} of the importance weight
\begin{equation}
	S^{\omega}(x) := \mathbb E_{y\sim\omega}\bigl[\operatorname{sgn}\bigl(w(x)-w(y)\bigr)\bigr] \in [-1,1],
	\label{eq: FR-S}
\end{equation}
which reports where $w(x)$ sits among the weights $\omega$ carries. Here $\operatorname{sgn}(\cdot)\in\{0,\pm1\}$ is the sign of a real number.

The FR flow of the generalized KLXX loss \eqref{eq: KLXX-general-nu} is a differential equation for a time-dependent pushforward density $\nu_t$, satisfying
\begin{equation}
	\partial_t\nu_t(x) = -\,\mathrm{grad}^{\mathrm{FR}}\mathcal L(\nu_t) = \pi(x)-\nu_t(x) + 2\lambda\,\pi(x)S^{\pi}_t(x) + 2\vartheta\,\xi(x)S^{\xi}_t(x),
	\label{eq: FR-flow}
\end{equation}
where the nonlocal correlations are evaluated with the importance weight $w_t=\pi/\nu_t$,
\begin{equation}
	S^{\pi}_t(x) = \mathbb E_{y\sim\pi}\bigl[\operatorname{sgn}\bigl(w_t(x)-w_t(y)\bigr)\bigr],
	\qquad
	S^{\xi}_t(x) = \mathbb E_{y\sim\xi}\bigl[\operatorname{sgn}\bigl(w_t(x)-w_t(y)\bigr)\bigr].
	\label{eq: FR-S-t}
\end{equation}
The mixture $\xi=\alpha\hat\pi+\beta\bar\nu$ is held fixed when the gradient is taken, as Appendix~\ref{sec: appendix-FR} carries through the derivation, and then follows $\nu_t$ along the flow. For $\beta>0$, \eqref{eq: FR-flow} is therefore not the gradient flow of one fixed functional; the dissipation identity \eqref{eq: combined-rate} below holds for it as stated.
Differentiating $\KL(\pi\|\nu_t)$ along \eqref{eq: FR-flow} gives the dissipation relation of the forward KL
\begin{equation}
	\frac{\D}{\D t}\KL(\pi\|\nu_t) = -\chi^2(\pi\|\nu_t)
	- \lambda\,\mathbb E_{x,y\sim\pi}\bigl[|w_t(x)-w_t(y)|\bigr]
	- \vartheta\,\mathbb E_{x,y\sim\xi}\bigl[|w_t(x)-w_t(y)|\bigr],
	\label{eq: combined-rate}
\end{equation}
where $\chi^2(\pi\|\nu):=\mathbb E_{\nu}\bigl[(w-1)^2\bigr]$ is the chi-squared divergence. The last two terms of \eqref{eq: combined-rate} are what the two variations contribute, and both are nonpositive. At $\lambda=\vartheta=0$ they vanish, the flow \eqref{eq: FR-flow} becomes the linear equation $\partial_t\nu_t=\pi-\nu_t$, and the only dissipation left is $-\chi^2(\pi\|\nu_t)$. Since $\chi^2(\pi\|\nu_t)\Ge\KL(\pi\|\nu_t)$, that term alone gives the decay $\KL(\pi\|\nu_t)\Le e^{-t}\KL(\pi\|\nu_0)$, which the forward KL already carries on its own. Adding the two variations back only makes the right-hand side more negative, so they do not slow convergence in FR geometry. Appendix~\ref{sec: appendix-FR} gives the derivation of the gradient flow \eqref{eq: FR-flow} and the proof details.

\begin{remark}
	Identity \eqref{eq: combined-rate} does not say how fast training converges. It describes a flow in density space. Training instead descends in the parameters of the flow map $G$, and that descent moves the density along a different path. The flow time $t$ is also not a step count. How far one step advances depends on the magnitude of the parameter gradient, and \eqref{eq: combined-rate} does not control it. Fisher--Rao geometry works in unconstrained density space and therefore cannot see mode collapse, which is a pathology of the sampled surrogate, the parametrization, and the optimizer; nothing in this subsection uses any property of the QT construction of Section~\ref{subsec: QT}.
\end{remark}

\subsection{Accuracy under biased target surrogate}

In training the generalized KLXX loss \eqref{eq: KLXX-general-nu}, SMC approximates the target distribution $\pi$ (Section~\ref{subsec: annealing}), so the target surrogate is a biased distribution $\tilde\pi$. The corresponding biased loss is given by
\begin{equation}
	\tilde{\mathcal L}[\nu] = \mathbb E_{y\sim\tilde\pi}\big[z(y)\big] +  \lambda\, \mathbb E_{y,y'\sim\tilde\pi}\big[|z(y)-z(y')|\big] + \vartheta\, \mathbb E_{y,y'\sim \alpha\hat\pi+\beta\bar\nu}\big[|z(y)-z(y')|\big],
	\label{eq: KLXX-biased}
\end{equation}
where the log-ratio $z=\log(\pi/\nu)$ is still evaluated against the exact target, and only the outer measures of the first two terms are replaced  by the biased $\tilde\pi$. The mixture term carries no surrogate bias: $\xi=\alpha\hat\pi+\beta\bar\nu$ approximates nothing, $\hat\pi$ being supplied by QT and $\bar\nu$ being the detached copy of $\nu$.

Following the derivation of \eqref{eq: FR-flow} for the unbiased loss, the FR gradient of the biased KLXX loss \eqref{eq: KLXX-biased} with respect to $\nu$ has a similar form,
\begin{equation}
	\mathrm{grad}^{\mathrm{FR}}[\tilde{\mathcal L}](x) = \nu(x)-\tilde\pi(x) - 2\lambda\,\tilde\pi(x)S^{\tilde\pi}(x) - 2\vartheta\,\xi(x)S^{\xi}(x),
	\label{eq: FR-grad-biased}
\end{equation}
where $S^{\tilde\pi}$ and $S^{\xi}$ are the nonlocal correlations defined in \eqref{eq: FR-S}.
Let $\nu_\star$ be a stationary density of $\tilde {\mathcal L}$, at which the gradient \eqref{eq: FR-grad-biased} vanishes. Collecting the two $\tilde\pi$ terms gives
\begin{equation}
	\nu_\star(x) = \tilde\pi(x)\bigl(1 + 2\lambda\,S^{\tilde\pi}_\star(x)\bigr) + 2\vartheta\,\xi(x)\,S^{\xi}_\star(x),
	\label{eq: surrogate-stationary}
\end{equation}
where the correlation terms are evaluated with the importance weight $w_\star=\pi/\nu_\star$,
\begin{equation}
	S^{\tilde\pi}_\star(x) = \mathbb E_{y\sim\tilde\pi}\bigl[\operatorname{sgn}\bigl(w_\star(x)-w_\star(y)\bigr)\bigr],
	\qquad
	S^{\xi}_\star(x) = \mathbb E_{y\sim\xi}\bigl[\operatorname{sgn}\bigl(w_\star(x)-w_\star(y)\bigr)\bigr].
\end{equation}
Since $\bar\nu$ tracks $\nu$ as training updates it, \eqref{eq: surrogate-stationary} characterizes the densities that annihilate the gradient rather than the minimizers of a functional. At $\lambda=\vartheta=0$ it reduces to $\nu_\star=\tilde\pi$, whose error is $\KL(\pi\|\tilde\pi)$; the following theorem bounds how far the two log-ratio variations move that error.

\begin{theorem}
\label{thm: accuracy}
Fix the coefficients $(\lambda,\vartheta,\alpha,\beta)$ and let $\tilde\pi$ be a positive probability density. Assume the fixed-point equation \eqref{eq: surrogate-stationary} has a positive solution $\nu_\star$, the mixture $\xi$ on its right-hand side being evaluated at that same $\nu_\star$, and write $w_\star = \pi/\nu_\star$ for the importance weight. Here $\xi=\alpha\hat\pi+\beta\nu_\star$ is the mixture measure at that stationary point, and we assume $\mathbb E_{\tilde\pi}[w_\star]<\infty$ and $\mathbb E_{\xi}[w_\star]<\infty$. Then
\begin{equation}
	\KL(\pi\|\nu_\star) \Le \KL(\pi\|\tilde\pi)
	- \lambda\,\mathbb E_{x,y\sim\tilde\pi}\bigl[|w_\star(x)-w_\star(y)|\bigr]
	- \vartheta\,\mathbb E_{x,y\sim\xi}\bigl[|w_\star(x)-w_\star(y)|\bigr] .
	\label{eq: accuracy-bound}
\end{equation}
\end{theorem}
Both subtracted discrepancies are nonnegative, so the two log-ratio variations improve the error that a fixed biased surrogate leaves behind, under two different weightings of the same ratio $w_\star$.

The size of that improvement is capped. Appendix~\ref{sec: appendix-accuracy} shows the two subtracted terms sum to $1-\mathbb E_{\tilde\pi}[w_\star]$, and nonnegativity forces $\mathbb E_{\tilde\pi}[w_\star]\in(0,1]$, so the subtraction is at most one whatever the coefficients and the dimension. The bound is therefore informative mainly when the surrogate is already close to the target. 

\begin{remark}
	A capped improvement in the KL loss does not cap the improvement KLXX delivers in practice, for two reasons. First, \eqref{eq: accuracy-bound} is an upper bound on the error, so a small reduction of the bound is not a small reduction of the error. Second, in forward KL training a small decrease of the loss can carry a large increase of the ESS. Table 2 of \citet{midgley2022flow} reports flows whose ESS differs by $52.2\%$ against $92.8\%$ while their forward KL values differ by less than $0.4$. The loss and the sampling quality it is meant to serve are on different scales, and the diagnostics of Section~\ref{subsec: eval} are what the experiments report.
\end{remark}

Theorem~\ref{thm: accuracy} holds a single $\tilde\pi$ fixed, while Algorithm~\ref{alg: training} rebuilds the target surrogate at every gradient step from the current flow. The two meet at a stationary point. The SMC samples enter the loss as data and are never differentiated through, so the gradient that vanishes at a stationary flow $\nu_\star$ is the Fisher--Rao gradient of \eqref{eq: KLXX-biased} with $\tilde\pi$ frozen at the surrogate Algorithm~\ref{alg: annealing} produces when initialized from $\nu_\star$, namely the marginal law $\tilde\pi_\star$ of one output particle. Identity \eqref{eq: surrogate-stationary} therefore holds with $\tilde\pi=\tilde\pi_\star$, and the proof of Theorem~\ref{thm: accuracy} uses nothing about $\tilde\pi$ beyond that identity, so the theorem applies with that surrogate, just as it applies with $\xi$ containing the detached copy of $\nu_\star$. The empirical variation pairs two particles of one SMC batch, which resampling leaves exchangeable rather than independent, so $\tilde\pi_\star\otimes\tilde\pi_\star$ is the independent-pair idealization of that batch.

% ===================== Sections 3-4: training + generator =====================
\section{Flow training with generalized KLXX loss}
\label{sec: training}

In this section we give the training algorithm for the generalized KLXX loss \eqref{eq: KLXX-general}. Recall that the source and target distributions are $\pi_0\propto \exp(-U_0)$ and $\pi\propto \exp(-U)$, and that the flow map $G$ is trained to fit the transport relation $\nu=(G^{-1})_{\#}\pi_0\approx\pi$, so that the log-ratio $z$ of \eqref{eq: log-ratio} carries the whole dependence on $G$.

The training algorithm uses two parameters throughout: the sample size $N$, the size of the sample set that enters and leaves the training and carries the diagnostics; and the batch size $B$, the number of samples per gradient step.
The generalized KLXX loss \eqref{eq: KLXX-general} draws on three sample sets: one from the target distribution $\pi$, one from the QT distribution $\hat\pi$, and one from the detached pushforward distribution $\bar\nu$. The QT set is constructed once before training, and its size is $N$ unless a smaller subset is used. The pushforward samples number $B$ and are drawn at each gradient step, where the target surrogate also approximates $\pi$ with a batch of $B$ samples.

The rest of the section is organized as follows. Section~\ref{subsec: training-algorithm} gives the KLXX training algorithm, which alternates between preparing samples and taking gradient steps; Sections~\ref{subsec: QT} and~\ref{subsec: annealing} then construct the QT and SMC samples it consumes; and Section~\ref{subsec: eval} gives two complementary diagnostics for the training quality.

\subsection{KLXX training algorithm}
\label{subsec: training-algorithm}

We begin from a sample set $\mathcal Y_0$ of size $N$ approximating the source distribution $\pi_0$, initialize the flow $G$ as an identity map, and run QT (Algorithm~\ref{alg: QT}) once to generate a QT sample set $\hat{\mathcal Y}$ of size $N$. Each gradient step then builds the target surrogate and mixture samples of size $B$ as follows:
\begin{enumerate}[(i)]
	\setlength{\itemsep}{0pt}
	\item Draw a random batch $\mathcal X_0\subset \mathcal Y_0$ of size $B$ and apply the pushforward to obtain $\mathcal X_{\nu} = G^{-1}(\mathcal X_0)$;
	\item Apply SMC (Algorithm~\ref{alg: annealing}) to $\mathcal X_\nu$ to obtain the target surrogate samples $\mathcal X_{\pi}$;
	\item Draw $B$ subsamples $\mathcal X_{\hat\pi}\subset \hat{\mathcal Y}$ from the QT samples, and mix $\mathcal X_{\hat\pi}$ and $\mathcal X_{\nu}$ with probabilities $\alpha$ and $\beta$ to obtain the mixture set $\mathcal X_{\xi}$;
	\item Evaluate the empirical loss on the target surrogate samples $\mathcal X_{\pi}$ and the mixture samples $\mathcal X_{\xi}$,
	\begin{equation}
	\widehat{\mathcal L}[G]=\underbrace{\widehat{\mathbb E}_{y\in\mathcal X_\pi}\bigl[z(y)\bigr]}_{\text{forward KL}} +
	\lambda\underbrace{\widehat{\mathbb E}_{y,y'\in\mathcal X_\pi}\bigl[|z(y)-z(y')|\bigr]}_{\text{target variation}}
	+\vartheta\underbrace{\widehat{\mathbb E}_{y,y'\in\mathcal X_{\xi}}\bigl[|z(y)-z(y')|\bigr]}_{\text{mixture variation}},
	\label{eq: KLXX-empirical}
	\end{equation}
	and take one gradient step on $G$, in which only $z(y) $ is differentiated.
\end{enumerate}
By convention $\mathcal X$ always has size $B$ and $\mathcal Y$ has size $N$. The sets $\mathcal Y_0$ and $\hat{\mathcal Y}$ are determined before training, and $\mathcal X_{\pi}$ and $\mathcal X_{\xi}$ are updated at each gradient step since their construction depends on $G$.
For a batch sample set $\mathcal X = \{y_1,\dots,y_B\}$, the empirical averages of $z(y)$ and $|z(y)-z(y')|$ appearing in \eqref{eq: KLXX-empirical} are defined by
\begin{equation}
	\widehat{\mathbb E}_{y\in \mathcal X} \big[z(y)\big] := \frac1B \sum_{i=1}^B z(y_i), \qquad \widehat{\mathbb E}_{y,y'\in \mathcal X} \big[|z(y)-z(y')|\big] := \frac{2}{B(B-1)}\sum_{1\Le i<j\Le B} |z(y_i) - z(y_j)|,
	\label{eq: empirical-variation}
\end{equation}
where the second sum runs over all $\binom{B}{2}$ pairs of distinct batch samples. Each $z(y_i)$ is a scalar, so sorting the $B$ values $z(y_1),\dots,z(y_B)$ computes this sum exactly in $O(B\log B)$ operations, without forming all $\binom{B}{2}$ differences. The estimator is unbiased for the log-ratio variation when the batch points are independent draws from the weighting distribution. The target variation therefore adds only this sort to the forward KL step, and the additional cost of KLXX comes from constructing the QT samples and evaluating the mixture variation. Algorithm~\ref{alg: training} collects steps (i)--(iv).

\begin{algorithm}[htb]
\setstretch{1.2}
\caption{KLXX training algorithm}
\label{alg: training}
\KwInput{source and QT sample sets $\mathcal Y_0,\hat{\mathcal Y}$; source and target potentials $U_0,U$; sizes $N,B$; loss coefficients $(\lambda,\vartheta,\alpha,\beta)$}
\KwOutput{trained flow map $G$}
Initialize $G$ to the identity map\;
\For{each gradient step}{
    \tcp{(i) pushforward}
    Draw a random batch $\mathcal X_0\subset\mathcal Y_0$ of size $B$ and set $\mathcal X_{\nu}\gets G^{-1}(\mathcal X_0)$\;
    \tcp{(ii) target surrogate}
    Apply SMC (Algorithm~\ref{alg: annealing}) to $\mathcal X_\nu$ to obtain $\mathcal X_{\pi}$\;
    \tcp{(iii) mixture}
    Draw $B$ subsamples $\mathcal X_{\hat\pi}\subset\hat{\mathcal Y}$, and mix $\mathcal X_{\hat\pi}$ and $\mathcal X_{\nu}$ with probabilities $\alpha$ and $\beta$ to form $\mathcal X_{\xi}$\;
    \tcp{(iv) gradient step}
    Take one gradient step on $G$ for the empirical loss
    \[
        \widehat{\mathcal L}[G]=\widehat{\mathbb E}_{y\in\mathcal X_\pi}\bigl[z(y)\bigr]
        +\lambda\,\widehat{\mathbb E}_{y,y'\in\mathcal X_\pi}\bigl[|z(y)-z(y')|\bigr]
        +\vartheta\,\widehat{\mathbb E}_{y,y'\in\mathcal X_{\xi}}\bigl[|z(y)-z(y')|\bigr],
    \]
    where the log-ratio $z(y) =  U_0(G(y)) - U(y) - \log|\det J_G(y)|$.
}
\end{algorithm}

\subsection{Constructing \texorpdfstring{$\hat\pi$}{pi-hat} by quench and temper}
\label{subsec: QT}

Here $\hat\pi$ is the QT distribution for the target $\pi$, built by the melt, quench and temper construction of Section~\ref{subsec: mixture-variation}. In the discrete setting we implement the quench step with limited-memory BFGS (L-BFGS) \citep{liu1989limited} and the temper step with the Metropolis-adjusted Langevin algorithm (MALA). Algorithm~\ref{alg: QT} writes the construction as the map
\begin{equation*}
    \hat{\mathcal Y} = \mathrm{QT}(\mathcal Y_0,U)
\end{equation*}
from the source samples $\mathcal Y_0=\{y_1,\dots,y_N\}$ and the target potential $U$ to the QT samples $\hat{\mathcal Y}=\{\hat y_1,\dots,\hat y_N\}$.

When the potential function $U$ is possibly singular, we prefer L-BFGS and MALA over plain gradient descent and unadjusted Langevin dynamics, which are less stable near the singularity. The QT samples can also be built from a subset of $\mathcal Y_0$, at lower cost.

\begin{remark}
	L-BFGS is a quasi-Newton method \citep{liu1989limited} and is not a discretization of the quench flow $\dot y = -\nabla U(y)$; its iterates do not follow that trajectory. The quench asks only for the local minimizer at the end of it, and in as few iterations as possible, so the better performance of L-BFGS over gradient descent is worth the departure from the flow.
\end{remark}

\begin{algorithm}[htb]
\setstretch{1.2}
\caption{QT: wide-coverage construction}
\label{alg: QT}
\KwInput{source sample set $\mathcal Y_0 = \{y_1,\dots,y_N\}$, target potential $U$, melt scale $m_e$}
\KwOutput{wide-coverage set $\hat{\mathcal Y} = \{\hat y_1,\dots,\hat y_N\}$ forming $\hat\pi$, written $\hat{\mathcal Y} = \mathrm{QT}(\mathcal Y_0,U)$}
\tcp{melt}
Scatter the samples $y_i \gets y_i + m_e\zeta_i$ with $\zeta_i \sim \mathcal N(0, I_d)$ i.i.d. and melt scale $m_e > 0$\;
\tcp{quench}
Drive each sample toward a mode center of $\pi$ by L-BFGS on $U$, updating $y_1,\dots,y_N$\;
\tcp{temper}
Spread the samples around each mode, $\hat y_{1:N} \gets \mathrm{MALA}(y_{1:N}, U)$\;
\end{algorithm}

The three steps treat every basin the melted cloud reaches alike, whatever its depth. When the target potential is singular, or when the modes carry very unequal weights, this leaves the QT set with particles at very high energy and with a mode population far from the target one. A reweighting step then follows the temper: compute the weights $\exp(-c_{\mathrm{QT}}U(\hat y_i))$ on the QT samples, resample the set by those weights, and rejuvenate the result by MALA under $U$ again, replacing the original QT set. The coefficient $c_{\mathrm{QT}}\Ge0$ sets how far the reweighting pulls the set toward the target; $c_{\mathrm{QT}}=0$ leaves the tempered set unchanged, and $c_{\mathrm{QT}}=1$ would weight it by $\pi$ itself. Values well below one keep the wide coverage the construction is for while removing the particles the singular core would otherwise contribute.

\subsection{Approximating $\pi$ by sequential Monte Carlo}
\label{subsec: annealing}

Step (ii) of the KLXX training algorithm (Algorithm~\ref{alg: training}) builds the target surrogate samples $\mathcal X_{\pi}$ on the fly, at every gradient step, since samples from $\pi$ are not directly available. We form them by reweighting the pushforward samples $y_1,\dots,y_B$ drawn from $\nu=(G^{-1})_{\#}\pi_0$. The importance weight is
\begin{equation}
	w(y)
	= \exp\Big( U_0(G(y)) - U(y) - \log|\det J_G(y)| \Big)
	= \frac{\pi(y)}{\nu(y)},
	\qquad y \in \mathbb R^d.
	\label{eq: weight}
\end{equation}
In high dimensions, however, $\pi$ and $\nu$ can differ greatly, so the weights can collapse onto a few outliers. We therefore reweight through $M+1$ bridging distributions
\begin{equation}
	\mu_m(y) \propto \nu(y)^{1-\frac{m}{M}} \pi(y)^{\frac{m}{M}} \propto ((G^{-1})_{\#}\pi_0)(y)\, w(y)^{\frac{m}{M}}, \qquad m = 0,1,\dots,M,
	\label{eq: mu-m}
\end{equation}
so that $\mu_0 = \nu$ and $\mu_M = \pi$. Each level reweights by only $w(y)^{1/M}$, reducing the degeneracy of direct importance sampling; at $M=1$ the ladder is a single reweighting followed by rejuvenation.
After each reweighting, we resample and rejuvenate the duplicates before the next level with MALA \citep{roberts1996exponential}, which leaves $\pi$ invariant.

\begin{algorithm}[htb]
	\setstretch{1.2}
	\caption{SMC for a biased target surrogate}
	\label{alg: annealing}
	\KwInput{pushforward samples $y_1,\dots,y_B$ from $\nu = (G^{-1})_{\#}\pi_0$, target distribution $\pi$, ladder length $M$}
	\KwOutput{biased target-surrogate batch $y_1,\dots,y_B$ approximating $\pi$}
	
	\For{$m = 1,\dots,M$}{
		Compute the incremental weights $w(y_i)^{1/M}=\exp(z(y_i)/M)$ on $y_1,\dots,y_B$\;
		\tcp{resample}
		Resample $y_1',\dots,y_B'$ from $y_{1:B}$ with weights $w(y_i)^{1/M}$\;
		\tcp{rejuvenate under the target}
		Diffuse the resampled duplicates by a short MALA chain on $\pi$, $y_{1:B} \gets \mathrm{MALA}(y_{1:B}', U)$\;
	}
\end{algorithm}
\medskip
\noindent
An SMC sampler that leaves every level of the path \eqref{eq: mu-m} invariant would rejuvenate at level $m$ with a Markov kernel invariant for $\mu_m$. For $m<M$, a MALA kernel for $\mu_m$ would require $\nabla_y\log\nu(y)$ and hence positional derivatives of the flow density, including $\nabla_y\log|\det J_G(y)|$; at $m=M$, this term disappears because $\mu_M=\pi$. Algorithm~\ref{alg: annealing} instead applies MALA directly under $\pi$, using only $-\nabla U$ in its proposal. The price is that it preserves $\pi$ rather than the intermediate $\mu_m$, so the output is a biased target surrogate.

The SMC construction in Algorithm~\ref{alg: annealing} differs from the AIS of FAB \citep{midgley2022flow}. FAB places the flow inside the chains, and its MCMC transitions target $\pi^2/\nu$, whose log-density gradient is
\begin{equation*}
	\nabla_y \log \frac{\pi^2(y)}{\nu(y)} = -2\nabla U(y) - \nabla_y \log\nu(y),
\end{equation*}
and therefore contains the positional gradient of the pushforward density at every MCMC step. The SMC of Algorithm~\ref{alg: annealing} instead constructs samples for $\pi$: every MALA chain targets $\pi$ and uses only $-\nabla U$, and the flow enters through the initial pushforward samples and evaluations of $z$. Each gradient step starts from a new source batch. The original FAB trains from a prioritized replay buffer of weighted AIS samples with a staleness correction, which reuses the annealing work of earlier steps. The FAB runs reported in Section~\ref{sec: experiments} keep no buffer: for simplicity they construct the samples for $\pi^2/\nu$ by an SMC sequence from $\nu$ to $\pi^2/\nu$ in which the rejuvenation kernel at each level leaves that level invariant, which needs the gradient of the flow density there. Each gradient step then starts from a fresh batch, and the two methods differ only in the objective and in the measure the chains target.

The exploration costs also differ. Relative to forward KL, KLXX adds the QT construction and the pairwise log-ratio evaluations. The target variation $\X_\pi$ reuses the target batch, and the mixture term evaluates one additional batch. Coverage then comes from the fixed $N$-sample QT set of Section~\ref{subsec: QT}, whereas FAB relies on its AIS budget for exploration and on the stronger tail penalty of its $2$-divergence.

\begin{remark}
	In principle, an SMC sampler can propagate from the source $\pi_0$ to the target $\pi$ on its own, with no flow at all. That choice has two drawbacks. First, such a sampler explores only where its own trajectories go, so nothing guarantees coverage of the target modes; the QT set of Section~\ref{subsec: QT} is what supplies candidate modes in KLXX. Second, with no flow to deform the proposal toward the next distribution, each level reweights and resamples at a much lower ESS, and repeated resampling of a degenerate weight set raises the Monte Carlo error of the returned samples.
	
	The Boltzmann generator of Section~\ref{sec: boltzmann} compares each trained flow with the identity proposal and keeps whichever has the larger sample ESS, so an accepted stage is never worse than the identity in that diagnostic, up to the finite sample set. The identity proposal itself propagates by reweighting, resampling, and MALA alone. Section~\ref{subsec: clock} reports the accepted stages this selection produces, and the propagation factor that follows from them.
\end{remark}

\subsection{Evaluation: ESS, coverage, and importance sampling}
\label{subsec: eval}

We assess the quality of the pushforward samples through two complementary diagnostics. The ESS measures density agreement through the observed importance weights on support reached by the flow. Coverage instead checks whether the pushforward samples reach the regions represented by a fixed QT reference set. A high ESS can coexist with a missing mode, while high coverage does not establish the correct relative density within the reached regions. We therefore report the two diagnostics together.

At evaluation time the trained flow $G$ supplies the pushforward distribution $\nu=(G^{-1})_{\#}\pi_0$ and its tractable density $\nu(y)=\pi_0(G(y))|\det J_G(y)|$. The importance weight $w(y)$ in \eqref{eq: weight} can therefore be evaluated on pushforward samples without fitting another density model. Let $x_1,\dots,x_N$ be samples from the source $\pi_0$ and let $y_i = G^{-1}(x_i)$, so that $\mathcal Y=\{y_1,\dots,y_N\}$ are pushforward samples with weights $w_i=w(y_i)$.

The first diagnostic is the ESS, normalized to $(0,1]$ by dividing the classical ESS by the sample count $N$. For the importance weights $(w_1,\dots,w_N)$, define
\begin{equation}
    \ESS(w) = \frac{\big(\sum_{i=1}^{N} w_i\big)^2}{N \sum_{i=1}^{N} w_i^2} \in (0,1].
    \label{eq: ESS}
\end{equation}
The upper bound follows from Cauchy--Schwarz and is attained when all weights are equal. The factor $N$ distinguishes this normalized ESS from the classical count $(\sum_iw_i)^2/\sum_iw_i^2\in[1,N]$ \citep{martino2017rethinking}. A high ESS indicates uniform observed weights only on support reached by $\nu$; a target mode absent from the sample cannot lower it.

To detect such missing modes, we pair the ESS with the coverage diagnostic of \citet{naeem2020reliable}. Let $\hat{\mathcal Y} = \{\hat y_1,\dots,\hat y_P\}$ be a QT sample set produced by Algorithm~\ref{alg: QT}, drawn independently of the pushforward samples $\mathcal Y$ and typically much smaller. For each $1 \Le i \Le P$, the $5$-nearest-neighbor radius of $\hat y_i$ is
\begin{equation}
    \mathrm{NND}(\hat y_i;{\hat{\mathcal Y}}) = \inf\Big\{ r \Ge 0 :
    \# \big\{1\Le j \Le P: 0 < |\hat y_i - \hat y_j| < r\big\} \Ge 5
    \Big\}.
\end{equation}
The coverage of $\mathcal Y$ in $\hat{\mathcal Y}$ is then
\begin{equation}
    \mathrm{Coverage}({\mathcal Y;\hat{\mathcal Y}}) =
    \frac{1}{P} \sum_{i=1}^{P} \mathbf 1\Bigl\{
        \exists  1\Le j\Le N : |y_j - \hat y_i| < \mathrm{NND}(\hat y_i;\hat{\mathcal Y})
    \Bigr\},
\end{equation}
i.e., the fraction of points in $\hat{\mathcal Y}$ whose $5$-nearest-neighbor ball contains a point of $\mathcal Y$. A coverage close to $1$ indicates that the finite QT reference set is well covered by the flow, while a coverage well below $1$ shows missing regions that the ESS alone cannot detect. Because the reference set is itself produced by QT, the diagnostic checks only modes found by that reference construction.

The same weights select the particles that the resampling step of Section~\ref{subsec: annealing} keeps. Normalize them as $p_i=w_i/\sum_{j=1}^{N}w_j$. Draw $N$ samples by multinomial resampling from the importance-weighted pushforward samples $(y_i,w_i)_{i=1}^{N}$ with probabilities $p_i$, and apply a short MALA chain under $\pi$ to rejuvenate the duplicates that resampling produces. The resampled samples form an unweighted approximation of the target-weighted empirical measure, while MALA retains $\pi$ as its invariant distribution.

\begin{remark}
	MALA or another local Monte Carlo method is needed to keep the samples distinct after resampling: even uniform weights can produce duplicates, and a low ESS produces more. The QT subsample of step~(iii) of the KLXX training algorithm (Algorithm~\ref{alg: training}) is rejuvenated the same way: the points drawn from $\hat{\mathcal Y}$ are advanced by a short MALA chain under the target before they enter the mixture. The source batch of step~(i) is not.
\end{remark}

Under standard regularity conditions, the target-weighted empirical measure and its resampled particle approximation converge to $\pi$ as $N\to\infty$ \citep{liu2001monte,doucet2001introduction}. The $\pi$-invariant MALA rejuvenation preserves this target limit. We call samples with this property \emph{asymptotically unbiased}: their bias against the target vanishes as $N\to\infty$.

\begin{remark}
Asymptotic unbiasedness is stated here informally. Theorem~\ref{thm: propagation} proves it rigorously for the inference scheme of Algorithm~\ref{alg: inference}, which replays the trained stage maps on fresh particles, with MALA as the rejuvenation kernel.
\end{remark}

\section{Adaptive-staging Boltzmann generator}
\label{sec: boltzmann}

For high-dimensional distributions, a single flow from the source $\pi_0$ to the target $\pi$ can be inadequate when the source-to-target deformation is too large, so we divide the training into stages that interpolate between the two potentials. Each stage should change the potential by a controlled amount and train the flow map with the generalized KLXX loss \eqref{eq: KLXX-general} at a fixed choice of $(\lambda,\vartheta,\alpha,\beta)$; the KLXX Boltzmann generator is the instance at the default parameters $(\lambda,\vartheta,\alpha,\beta)=(1,1,\frac12,\frac12)$. The stage schedule is chosen from sampling diagnostics computed on the sample set.

\subsection{Why staging is adaptive}

Two obstacles limit what a single flow can do in carrying $\pi_0$ to $\pi$ in one map.

\begin{itemize}
    \item \emph{Curse of dimensionality.} In low dimension a unimodal source can often be reshaped into a multimodal target by one map, but as $d$ grows the target landscape develops more basins, separated by higher barriers and occupying a smaller fraction of the source support. A normalizing flow of fixed depth and width can represent only a bounded change of shape, so the effect of a training run is limited by that capacity. A single map from source to target can therefore leave part of the target unfitted when the problem is high-dimensional and multimodal. Splitting the interpolation into stages keeps each stage within the capacity of one flow, and the number of stages needed grows as the source and the target move apart.

    \item \emph{Structural limitation of a diffeomorphism.} A trained flow is a diffeomorphism of the domain, so the pushforward $\nu$ inherits the connected support of $\pi_0$: the map deforms that support continuously and cannot split it into disconnected components. Mass therefore remains on the paths joining the target basins even when the basins themselves are well fitted, an effect visible between the isolated modes of the Sparse target in Figure~\ref{fig: sparse}. Reweighting and MALA rejuvenation at the interpolating distributions remove the leaked samples at every stage, so the support changes shape gradually across the schedule.
\end{itemize}

The stage schedule cannot be fixed in advance, because we assume no external information about the target. Training uses no target dataset, no reference trajectory, and no prior knowledge of the number, location, or depth of the potential wells. Under this restriction the difficulty of a candidate stage cannot be read off the potential before attempting it. The algorithm therefore proposes a stage increment, trains the stage flow, and accepts the increment only when the resulting proposal clears the ESS gate on the sample set; otherwise it shrinks the increment and retries. Easy interpolations can then take large stages and hard interpolations take smaller ones, using only oracle access to $U(x)$ and $\nabla U(x)$.

\subsection{Construction of the adaptive stage schedule}

To interpolate from the source $\pi_0 \propto \exp(-U_0)$ to the target $\pi \propto \exp(-U)$, we seek an increasing sequence
\begin{equation*}
    0 = t_0 < t_1 < \cdots < t_K = 1
\end{equation*}
defining the interpolating potential and measure
\begin{equation*}
    U_t(x)=(1-t)U_0(x)+tU(x),
\qquad \pi_t(x)\propto\exp(-U_t(x)),
    \qquad t\in[0,1].
\end{equation*}
Here $t$ is a dimensionless interpolation parameter. The discrete values $t_k$ are stage points, and stage $k$ is the transition from $t_{k-1}$ to $t_k$. At the accepted stage points, write $\pi_k:=\pi_{t_k}$. Stage $k$ trains a target-to-source map $G_k$, whose inverse induces the pushforward distribution $\nu_k=(G_k^{-1})_{\#}\pi_{k-1}$ for $\pi_k$. The number of stages $K$ is selected during the run. The central scheduling choice is the stage increment $\Delta t_k=t_k-t_{k-1}$: a large increment can be hard to fit, whereas a small increment adds stages and thus computational cost.

Training $G_k$ requires the target surrogate for $\pi_k$ defined in Section~\ref{subsec: annealing}. The stage point $t_k$ is proposed by extrapolating the last accepted increment and is accepted or shrunk by the stage ESS gate of Algorithm~\ref{alg: boltzmann}, so the schedule is built from the diagnostic the run already computes. Throughout this section, the sample set is carried across stages and used for training batches, QT, and stage selection; QT may also act on a separately sized subset.

The SMC target surrogate (Algorithm~\ref{alg: annealing}) involves only the potentials $U_0$ and $U$ and uses no QT samples, so it can miss a mode its own chains never reach; the stage ESS of step~(iv) below is likewise a difficulty diagnostic and not a mode-coverage guarantee, and QT supplies the separate wide-coverage set. As $t_k$ approaches $t_{k-1}$, the stage becomes trivial and its ESS approaches $1$, so a gate below $1$ is attainable under repeated shrinking in exact arithmetic.

\subsection{The full training procedure}

Algorithm~\ref{alg: boltzmann} combines the stage-selection rule with the training procedure. It carries across stages a sample set $\mathcal Y_k$ of size $N$, whose empirical measure approximates the per-stage target $\pi_k$. The $k$-th candidate stage collects the sample set $\mathcal Y_{k-1}$ from the last stage, builds its QT set $\hat{\mathcal Y}_k$ of size $N$, trains a flow $G_k$ with the selected loss weights at batch size $B$, and compares that flow with the identity proposal. If neither proposal reaches the ESS gate $\tau_{\mathrm v}\in[0,1)$, the algorithm shrinks the stage increment and retries. The loop ends when $t_K=1$.

\begin{enumerate}[(i)]
\item \emph{Adaptive stage selection.} The next stage point $t_k$ is proposed first, then tested. The first stage starts from $t_k=t_{\mathrm{safe}}$, with default $0.2$. Later stages extrapolate the last accepted stage increment, $t_k=\min\bigl(1,t_{k-1}+\Gamma(t_{k-1}-t_{k-2})\bigr)$, with enlarge factor $\Gamma=1.5$. The proposal is tested by training and by the stage ESS gate of step~(iv), which shrinks it when the stage is not accepted. The selected stage target is $\pi_k\propto\exp(-U_{t_k})$, with $U_{t_k}=(1-t_k)U_0+t_kU$.

\item \emph{QT set construction}. QT (Algorithm~\ref{alg: QT}) acts on the sample set $\mathcal Y_{k-1}$ under $U_{t_k}$; the algorithm is stated for the source set, and applies unchanged to any set carried into a stage. The resulting $N$ samples $\hat{\mathcal Y}_k=\mathrm{QT}(\mathcal Y_{k-1},U_{t_k})$ are held fixed for the attempt, and rebuilt when a shrunk increment changes $U_{t_k}$.

\item \emph{Flow training}. Each attempt initializes $G_k$ to the identity map and trains it by the KLXX training algorithm (Algorithm~\ref{alg: training}), applied to the stage $\pi_{k-1}\to\pi_k$: the batches of size $B$ are drawn from $\mathcal Y_{k-1}$, the target surrogate samples come from SMC (Algorithm~\ref{alg: annealing}) with $G=G_k$, the QT sample set is $\hat{\mathcal Y}_k$, and the stage log-ratio is
\begin{equation*}
    z_k(y) = U_{t_{k-1}}(G_k(y)) - U_{t_k}(y) - \log|\det J_{G_k}(y)|.
\end{equation*}
The coefficients $(\lambda,\vartheta,\alpha,\beta)$ are fixed for the run. Taking $(\lambda,\vartheta)=(0,0)$ gives forward KL, $(\lambda,\vartheta)=(1,0)$ gives $\KL$+$\X_{\pi}$, $(1,1,1,0)$ gives $\KL$+$\X_\pi$+$\X_{\hat\pi}$, and $(1,1,\frac12,\frac12)$ gives the default KLXX.

\item \emph{Sample set update.} Once the candidate $G_k$ is trained, each $y \in \mathcal Y_{k-1}$ is pushed through its inverse to $\tilde y = G_k^{-1}(y)$. This gives the trained set of pushforward samples $\tilde{\mathcal Y}_k^{\mathrm{tr}} = G_k^{-1}(\mathcal Y_{k-1})$ from $\nu_k = (G_k^{-1})_{\#}\pi_{k-1}$, together with the importance weights
\begin{equation*}
    w_k^{\mathrm{tr}}(\tilde y) := \exp\bigl(z_k(\tilde y)\bigr),
    \qquad \tilde y \in \tilde{\mathcal Y}_k^{\mathrm{tr}} .
\end{equation*}
The identity proposal keeps $\mathcal Y_{k-1}$ unchanged and has weights $w_k^{\mathrm{id}}(y)=\exp\big(U_{t_{k-1}}(y)-U_{t_k}(y)\big)$. We select the pair with the larger ESS over the sample set and denote it by $(\tilde{\mathcal Y}_k,w_k)$, taking the stage map to be $G_k$ or the identity accordingly. If $\ESS(w_k) < \tau_{\mathrm v}$, with default $\tau_{\mathrm v} = 0.4$, we discard the candidate, shrink to $t_k \gets t_{k-1} + \gamma\,(t_k - t_{k-1})$ with shrink factor $\gamma=0.7$, and repeat steps (ii)--(iv) at the reduced stage increment. Otherwise we draw $N$ samples from $\tilde{\mathcal Y}_k$ with weights $w_k$ and rejuvenate them by MALA under $U_{t_k}$, giving the next sample set $\mathcal Y_k\approx\pi_k$.
\end{enumerate}

\begin{algorithm}
    \setstretch{1.1}
\caption{Adaptive-staging Boltzmann generator training}
\label{alg: boltzmann}
\small
\KwInput{potentials $U_0,U$; sizes $N$ and $B\Ge2$, the batch needing two points for a pairwise difference; loss coefficients $(\lambda,\vartheta,\alpha,\beta)$; ladder length $M$; ESS gate $\tau_{\mathrm v}=0.4$; factors $\gamma=0.7$, $\Gamma=1.5$; safe start $t_{\mathrm{safe}}=0.2$}
\KwOutput{completed stage schedule $0=t_0<\cdots<t_K=1$ and selected stage maps $\{G_k\}_{k=1}^{K}$}

\tcp{sample set, $\mathcal Y_k \approx \pi_k$}
$\mathcal Y_0 \gets N$ samples from $\pi_0$, \quad $t_0 \gets 0$, \quad $k \gets 0$\;
\While{$t_k < 1$}{
    $k \gets k + 1$\;
    Set $U_{t_{k-1}} \gets (1-t_{k-1}) U_0 + t_{k-1} U$\;
    \tcp{(i) adaptive stage selection}
    \eIf{$k=1$}{
        Set $t_k \gets t_{\mathrm{safe}}$\;
    }
    {
        Set $t_k \gets \min\big(1, t_{k-1} + \Gamma(t_{k-1} - t_{k-2})\big)$\;
    }
    Set $U_{t_k} \gets (1-t_k) U_0 + t_k U$\;
    \Repeat{$\ESS(w_k) \Ge \tau_{\mathrm v}$}{
        Initialize $G_k$ to the identity map\;
        \tcp{(ii) QT set construction}
        Set the QT sample set $\hat{\mathcal Y}_k \gets \mathrm{QT}(\mathcal Y_{k-1},U_{t_k})$ by Algorithm~\ref{alg: QT}\;
        \tcp{(iii) flow training}
        Train $G_k$ by the KLXX training algorithm (Algorithm~\ref{alg: training}) on the sample set
        $\mathcal Y_{k-1}$ and the QT sample set $\hat{\mathcal Y}_k$, from $U_{t_{k-1}}$ to $U_{t_k}$,
        at the coefficients $(\lambda,\vartheta,\alpha,\beta)$\;
        \tcp{(iv) sample set update}
        Form the trained and identity pushforward samples with their weights\;
        Keep the larger-ESS pair as $(\tilde{\mathcal Y}_k,w_k)$ and set $G_k$ to the corresponding map\;
        \eIf{$\ESS(w_k) < \tau_{\mathrm v}$}{
            \tcp{reject the stage and shrink toward $t_{k-1}$}
            $t_k \gets t_{k-1} + \gamma\,(t_k - t_{k-1})$, then update $U_{t_k} \gets (1-t_k) U_0 + t_k U$\;
        }{
            Resample $N$ samples from $\tilde{\mathcal Y}_k$ with weights $w_k$\;
            Rejuvenate the resampled set by MALA under $U_{t_k}$ and set $\mathcal Y_k$ to the result\;
        }
    }
}
$K \gets k$\;
\Return $\{t_k\}_{k=0}^{K}$ and $\{G_k\}_{k=1}^{K}$\;
\end{algorithm}
\medskip
\noindent
Two losses can be compared on one schedule. Algorithm~\ref{alg: boltzmann} is run once with the first loss, which accepts a schedule $t_0<\cdots<t_K$. The second loss is then trained on those same accepted values of $t$, one flow per stage, with steps (ii) to (iv) unchanged and no shrinking, since the stage points are already fixed. Every stage of the two runs then faces the same pair of distributions, so their stage ESS values are comparable. Section~\ref{subsec: clock} reports such a comparison.

The accepted stage schedule defines a Boltzmann generator for $\pi$. The process starts with samples from $\pi_0$. At stage $k$, it applies the selected inverse map or identity, computes $w_k$, resamples, and rejuvenates under $U_{t_k}$. The resulting sample set supplies stage $k+1$.

\subsection{Propagation factor analysis}
\label{sec: propagation-analysis}

Algorithm~\ref{alg: boltzmann} traverses the interpolation distributions $\{\pi_k\}_{k=0}^K$, training a flow map $G_k$ for each stage from $\pi_{k-1}$ to $\pi_k$. We next ask how the Monte Carlo sampling error accumulates across these stages. Let $w_k = \pi_k/\nu_k$ be the importance weight at stage $k$, written as the ratio of the target density $\pi_k$ and the pushforward density $\nu_k$. Assume $w_k$ is well defined so that importance sampling is available at every stage. Algorithm~\ref{alg: inference} states the inference scheme. Here $Q_k$ is any Markov kernel leaving $\pi_k$ invariant; the MALA of step (iv) in Algorithm~\ref{alg: boltzmann} is one such kernel, and nothing below uses more than that invariance.

\begin{algorithm}[htb]
\setstretch{1.2}

\caption{Inference scheme of the Boltzmann generator}
\label{alg: inference}
\KwInput{source distribution $\pi_0$, stage maps $G_1,\dots,G_K$, stage weights $w_1,\dots,w_K$, rejuvenation kernels $Q_1,\dots,Q_K$, sample size $N$}
\KwOutput{sample set $X_K^1,\dots,X_K^N$ for the target distribution $\pi_K$}
Draw $X_0^1,\dots,X_0^N$ independently from the source distribution $\pi_0$\;
\For{$k=1,\dots,K$}{
	\tcp{flow pushforward}
	Apply the inverse stage map to each particle, $Y_k^i = G_k^{-1}(X_{k-1}^i)$\;
	\tcp{importance sampling}
	Evaluate the stage weight at each pushforward particle, $W_k^i = w_k(Y_k^i)$\;
	Draw independent $Z_k^1,\dots,Z_k^N$ by multinomial resampling from $\sum_{i=1}^{N}\bigl(W_k^i/\sum_{l=1}^{N} W_k^l\bigr)\,\delta_{Y_k^i}$\;
	\tcp{rejuvenation}
	Rejuvenate each resampled particle, $X_k^i \sim Q_k(Z_k^i,\cdot)$, independently across $i$\;
}
\end{algorithm}

After stage $K$ the scheme returns the empirical measure
\begin{equation}
	\pi_K^N = \frac1N\sum_{i=1}^{N} \delta_{X_K^i}
\end{equation}
as the estimate of $\pi_K$. For a general distribution $\mu$ and a function $\varphi$, we write $\mu\langle \varphi\rangle := \int_{\mathbb R^d} \varphi \, \D\mu$ for the integral. For a bounded observable $\varphi$ we write $\operatorname{osc}(\varphi) = \operatorname{ess\,sup}\varphi - \operatorname{ess\,inf}\varphi$, and for a function $h$ write $\lVert h \rVert_\infty$ for the essential supremum of $|h|$ on $\mathbb R^d$, taken for $\varphi$ with respect to $\pi_K$ and for $w_k$ with respect to $\nu_k$.

\begin{theorem}
\label{thm: propagation}
Suppose that, for every stage $k=1,\dots,K$, $w_k>0$, the $\nu_k$-essential supremum $\lVert w_k \rVert_\infty$ is finite, and $Q_k$ is a Markov kernel leaving $\pi_k$ invariant. Then for every bounded observable $\varphi$ and every $N \Ge 1$,
\begin{equation}
    \mathbb E\Bigl[\bigl(\pi_K^N\langle \varphi\rangle - \pi_K\langle \varphi\rangle\bigr)^{2}\Bigr]
    \Le \frac{C}{N},
    \qquad
    C = \frac{\operatorname{osc}(\varphi)^{2}}{4}
    \left(\sum_{k=0}^{K}\ \prod_{j=k+1}^{K} 3\lVert w_j \rVert_\infty\right)^{\!2},
    \label{eq: propagation-bound}
\end{equation}
with the empty product equal to one.
\end{theorem}
\noindent
Appendix~\ref{sec: appendix-propagation} gives the proof of Theorem~\ref{thm: propagation}. The constant is explicit, independent of $N$, and depends on the stage maps only through the weight suprema. Rejuvenation does not enlarge it: given the pushforward particles, the resampling and the kernel draw amount to a single draw per particle, and a Markov kernel cannot widen the range of a bounded function.

\begin{remark}
The theorem describes Algorithm~\ref{alg: inference}, the inference scheme with the trained maps held fixed. The direct output of Algorithm~\ref{alg: boltzmann} is not covered: its samples can be correlated with the schedule, since the same particles select the stages and the flow-versus-identity proposal before being propagated. The fresh-rebuild tests of Section~\ref{subsec: clock} replay the frozen stage flows on independent particles.
\end{remark}

Theorem~\ref{thm: propagation} also settles how the scheme behaves in $N$. Writing $\lVert\cdot\rVert_2$ for the $L^2$ norm over all the randomness of the scheme, the bound gives $\lVert \pi_K^N\langle \varphi\rangle - \pi_K\langle \varphi\rangle \rVert_2 \Le \sqrt{C/N}$ for every bounded observable, which also bounds the bias and vanishes as $N$ grows: the inference is \emph{asymptotically unbiased}, and its error is removed by increasing the sample size alone. Classical Monte Carlo sampling of $\pi$ instead requires the simulation time to grow and, for unadjusted integrators, the time step to vanish; a Metropolis-adjusted kernel such as MALA removes the second requirement but not the first \citep{metropolis1953equation,hastings1970monte,roberts1996exponential,delmoral2004feynman}. The rejuvenation also keeps the particles distinct: resampling only copies and the stage maps are bijections, so without it the number of distinct configurations could only fall from stage to stage, and the independent noise scatters the copies apart again.

The constant $C$ in Theorem~\ref{thm: propagation} is a square up to the observable, and its square root is defined as the \emph{propagation factor} of the training schedule,
\begin{equation}
    F_\Sigma
    := \sum_{k=0}^{K}\ \prod_{j=k+1}^{K} 3\lVert w_j \rVert_\infty ,
    \label{eq: propagation-factor-sigma}
\end{equation}
so that the inequality \eqref{eq: propagation-bound} equivalently reads
\begin{equation*}
    \bigl\lVert \pi_K^N\langle \varphi\rangle - \pi_K\langle \varphi\rangle \bigr\rVert_2
    \Le \frac{\operatorname{osc}(\varphi)}{2}\, \frac{F_\Sigma}{\sqrt N}.
\end{equation*}

$F_\Sigma$ is built from the essential supremum of each stage weight, which makes it impractical. One rare configuration of large weight fixes the supremum, so $\lVert w_j \rVert_\infty$ can be enormous at a stage that otherwise behaves well; for a singular potential it is difficult to estimate accurately, and need not be finite at all. We therefore estimate the training difficulty by a second propagation factor, assembled from the ESS of each accepted stage, a quantity the run computes anyway. For the stage weight, the ESS of \eqref{eq: ESS} reads
\begin{equation}
    \ESS(w_k) := \frac{1}{\int_{\mathbb R^d} w_k(x)^2 \, \nu_k(x) \, \D x} \in (0,1] .
    \label{eq: propagation-ESS}
\end{equation}
The factor $F$ keeps the leading product of \eqref{eq: propagation-factor-sigma}, the $k=0$ summand, drops the factors $3$, which arise from the recurrence in the proof and carry no information about the problem, and replaces each supremum by the reciprocal square root of the stage ESS,
\begin{equation}
    F
    := \prod_{j=1}^{K} \frac{1}{\sqrt{\ESS(w_j)}} .
    \label{eq: propagation-factor-F}
\end{equation}
Each substitution only lowers the factor, since $\int_{\mathbb R^d} w_j^2 \, \D\nu_j \Le \lVert w_j \rVert_\infty \Le \lVert w_j \rVert_\infty^{2}$ by $\lVert w_j \rVert_\infty \Ge \int_{\mathbb R^d} w_j \, \D\nu_j = 1$. Hence $F \Le F_\Sigma$, and $F$ does not bound the constant of \eqref{eq: propagation-bound}. Every factor in \eqref{eq: propagation-factor-F} is at least one, so $F \Ge 1$, with equality exactly when every stage proposal equals its target.

The calculation of $F$ requires the $\nu_j$-integral $\int w_j^2 \, \D\nu_j$. A Boltzmann generator run already evaluates the stage weight vectors $W_j = (W_j^1,\dots,W_j^N)$ with $W_j^i = w_j(Y_j^i)$, and the sample ESS of each vector is read off them at no further cost, giving the propagation factor that the experiments report,
\begin{equation}
    \hat F := \prod_{j=1}^{K} \frac{1}{\sqrt{\ESS(W_j)}} ,
    \label{eq: propagation-factor-hat}
\end{equation}
the sample counterpart of $F$, invariant to the scale of the weights and therefore available when the target is known only up to a constant. Being a sample quantity, it is not a bound on $F$ at finite $N$.
$F$ summarizes the weight degeneracy of the composed flow along one accepted schedule, and does not measure the difficulty of the target. The sample size required by importance sampling grows exponentially in the KL divergence between target and proposal \citep{chatterjee2018sample}, and with the intrinsic dimension of the problem \citep{agapiou2017importance}; neither dependence appears in $F$.

The argument behind Theorem~\ref{thm: propagation} already appears in the particle-system literature. A sum over the stage at which the error enters, times a product over the stages it must still traverse, is the standard form for Feynman--Kac particle systems \citep{delmoral2004feynman,chopin2020introduction}, and is the form of the variance recurrence AFT \citep{arbel2021annealed} and CRAFT \citep{matthews2022craft} obtain as a central limit theorem. Our scheme is such a system, with the rejuvenation kernel of one stage composed with the trained map of the next as the mutation and the stage weight $w_k$ as the potential, so \eqref{eq: propagation-bound} is that recurrence unrolled, bounded rather than exact, and explicit at finite $N$.

In the variance recurrence of those analyses, the term injected at a stage is a second moment of the stage weight, the quantity the stage ESS of \eqref{eq: propagation-ESS} measures. This is why $F$ is assembled from stage ESS values. Being a product of factors at least one, $F$ grows with every extra stage, so two losses are comparable only on a shared schedule, as in Section~\ref{subsec: clock}. The experiments print its sample value $\hat F$, since $F$ itself needs the exact integrals $\int w_j^2\D\nu_j$.

% ===================== Section 5: model problems =====================
\section{Numerical experiments}
\label{sec: experiments}

We test forward KL, and on the product multi-well also the LDR-L1 loss of \citet{schopmans2026ldr}, against three choices of the generalized KLXX loss \eqref{eq: KLXX-general}: $\KL$+$\X_\pi$, $\KL$+$\X_\pi$+$\X_{\hat\pi}$, and $\KL$+$\X_\pi$+$\X_{(\hat\pi+\bar\nu)/2}$, corresponding to different coefficients $(\lambda,\vartheta,\alpha,\beta)$. The last is KLXX, whose mixture variation weights the QT distribution $\hat\pi$ and the pushforward $\bar\nu$ equally. The tests include the 2D benchmark distributions, the high-dimensional product multi-well, the coefficient sweep of the target variation at $d=100$, the lattice $\phi^4$ field theory, the $p$-state clock model, and three achiral molecules. They vary in dimension and in how the modes are distributed. Each subsection below gives the sizes and the training length of its runs. The step sizes of MALA and of the quench, the quench budget, and the learning rate are those of the drivers released with this paper. When a fake ESS appears, we report the ESS together with a diagnostic that sees the missing mass: the coverage of Section~\ref{subsec: eval} on the 2D targets, and the minority-well weight on the $\phi^4$ lattices. On the product multi-well the wells of the double-well coordinates label the mode exactly, and every method finds every mode, so no fake ESS arises there. We also compare against FAB \citep{midgley2022flow} on the 2D benchmarks, the product multi-well, and the $\phi^4$ lattices; the coefficient sweep of Section~\ref{subsec: hd100} varies one loss, and the clock model of Section~\ref{subsec: clock} and the molecules of Section~\ref{subsec: achiral} train staged generators, so none of them carries a FAB run.

\subsection{2D benchmark distributions}
\label{subsec: 2d-benchmark}

For every target, we initialize the same neural spline flow (NSF) \citep{durkan2019neural} at the identity and compare the four losses above, together with FAB. Each training step constructs a target surrogate through the current flow on a ladder of length $M=1$, followed by MALA rejuvenation. SMC and QT both act on a fixed sample set of $5\times10^4$ to $8\times10^4$ points, and QT uses melt scale $m_e=2$. Coverage, against an independently generated QT reference set, measures the fraction of reference configurations represented by the pushforward samples through the nearest-neighbor test of Section~\ref{subsec: eval}. Because coverage tests occupancy of the reference balls, a collapsed flow can have ESS near one but cannot have full coverage. Each loss--target pair is one fixed-seed training realization, so the numbers below are for these matched runs, not variation across independent trainings.

\begin{figure}[htb]
    \centering
    \includegraphics[width=\textwidth]{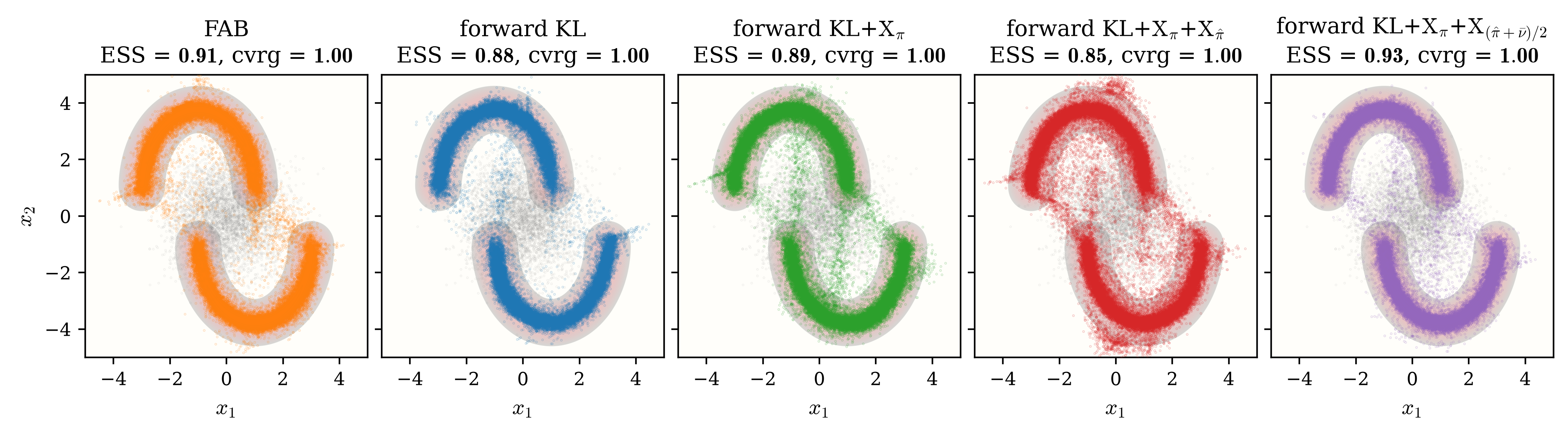}
    \caption{Two-Moon target. Columns, from left to right, are FAB, forward KL, $\KL$+$\X_\pi$, $\KL$+$\X_\pi$+$\X_{\hat\pi}$, and $\KL$+$\X_\pi$+$\X_{(\hat\pi+\bar\nu)/2}$. Each panel overlays the pushforward samples on the target energy and reports final ESS and coverage. Gray points show the source.}
    \label{fig: two-moon}
\end{figure}

\paragraph{Two-Moon.}
The source is $\mathcal N(0,I_2)$ and both crescents lie within its reach, so all five methods attain coverage $1$ (Figure~\ref{fig: two-moon}). Training uses a sample set of $5\times10^4$, batch size $500$, and $200$ steps. KLXX has the highest ESS of the five and $\KL$+$\X_\pi$+$\X_{\hat\pi}$ the lowest, since QT supplies no mode that the target surrogate is missing.

\begin{figure}[htb]
    \centering
    \includegraphics[width=\textwidth]{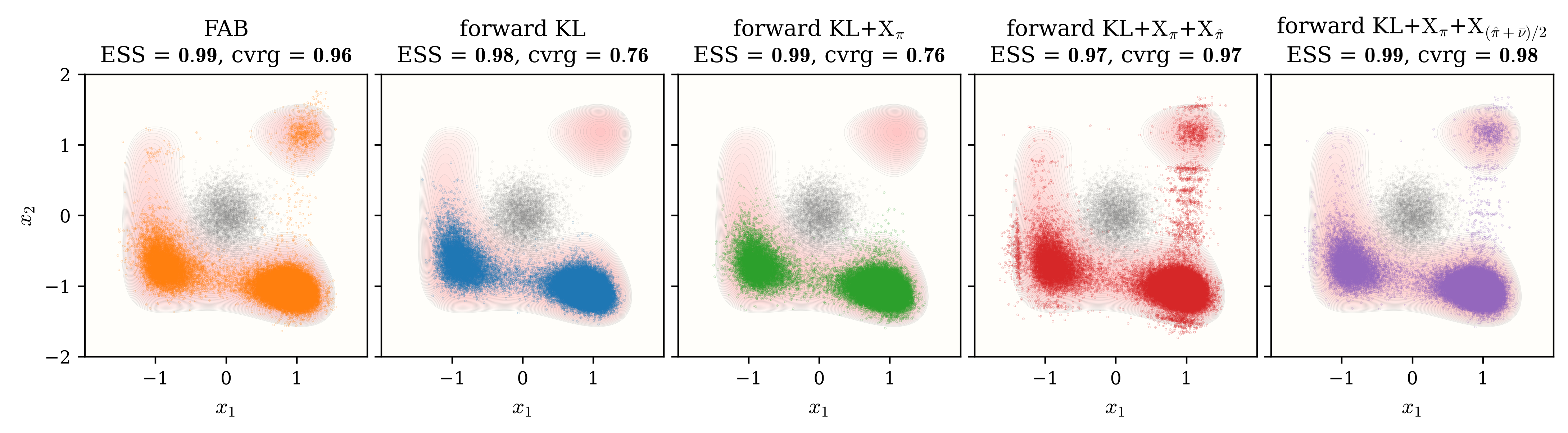}
    \caption{Three-Well target. From left to right, the panels show FAB, forward KL, $\KL$+$\X_\pi$, $\KL$+$\X_\pi$+$\X_{\hat\pi}$, and $\KL$+$\X_\pi$+$\X_{(\hat\pi+\bar\nu)/2}$. Each panel overlays the pushforward samples on the target energy and reports final ESS and coverage; gray points show the source.}
    \label{fig: threewell}
\end{figure}

\paragraph{Three-Well.}
The source is $\mathcal N(0,0.25^2 I_2)$, narrower than the well separation and initially covering no basin (Figure~\ref{fig: threewell}). Training uses a sample set of $8\times10^4$, batch size $1000$, and $500$ steps. Forward KL and $\KL$+$\X_\pi$ populate the two favored wells and leave the third uncovered. The QT set also represents the suppressed third well, and both losses that use it cover all three wells, at ESS values within $0.03$ of each other.

\begin{figure}[htb]
    \centering
    \includegraphics[width=\textwidth]{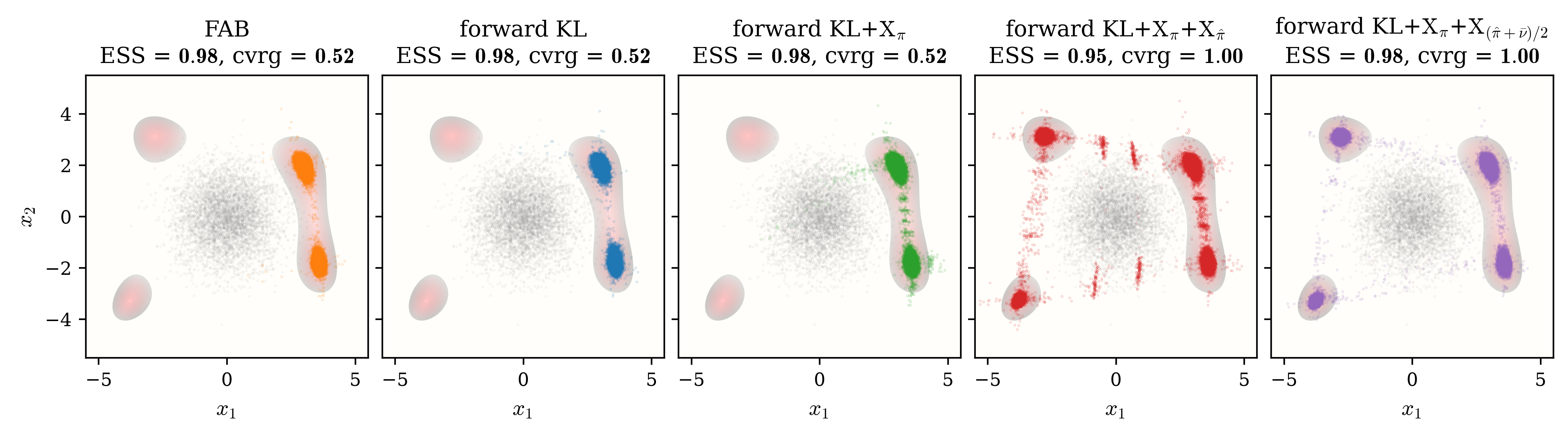}
    \caption{Himmelblau target. Columns show FAB, forward KL, $\KL$+$\X_\pi$, $\KL$+$\X_\pi$+$\X_{\hat\pi}$, and $\KL$+$\X_\pi$+$\X_{(\hat\pi+\bar\nu)/2}$; each panel shows the pushforward samples over the target energy with final ESS and coverage, and gray points show the source.}
    \label{fig: himmelblau}
\end{figure}

\paragraph{Himmelblau.}
The source is $\mathcal N(0,I_2)$, whereas the four wells lie at radius approximately $3.6$--$5.0$. Training uses a sample set of $5\times10^4$, batch size $1000$, and $1000$ steps. Forward KL and $\KL$+$\X_\pi$ capture only two wells (Figure~\ref{fig: himmelblau}), so their coverage falls to about half at ESS above $0.97$. Both losses using $\hat\pi$ attain coverage $1$, and relative to $\KL$+$\X_\pi$+$\X_{\hat\pi}$, KLXX shows less intermodal mass at higher ESS.
\begin{figure}[htb]
    \centering
    \includegraphics[width=\textwidth]{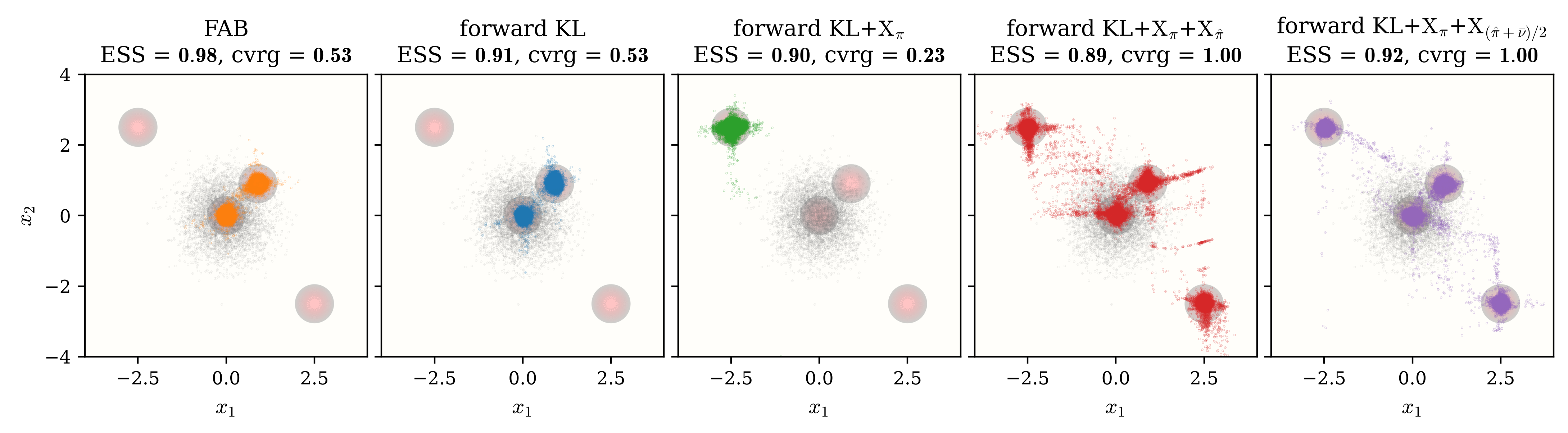}
    \caption{Sparse target. Columns show FAB, forward KL, $\KL$+$\X_\pi$, $\KL$+$\X_\pi$+$\X_{\hat\pi}$, and $\KL$+$\X_\pi$+$\X_{(\hat\pi+\bar\nu)/2}$, with the pushforward samples overlaid on the target energy; panel titles report final ESS and coverage, and gray points show the source.}
    \label{fig: sparse}
\end{figure}

\paragraph{Sparse.}
The target is a four-component Gaussian mixture with component standard deviation $0.08$: two modes lie near the source center and two lie in opposite corners, more than five source standard deviations away. With source standard deviation $0.6$, a sample set of $8\times10^4$, batch size $1000$, and $1000$ steps, forward KL covers only the central pair, and $\KL$+$\X_\pi$ is worse still: it settles on a single far mode. The loss $\KL$+$\X_\pi$+$\X_{\hat\pi}$ restores all four modes but places visible mass between them, while KLXX attains coverage $1$ with less intermodal mass and higher ESS (Figure~\ref{fig: sparse}).

We also run FAB \citep{midgley2022flow} on all four targets, with the same flow, source, target, batch size and step count as the losses above. The objectives differ in what each training step must carry, as Section~\ref{subsec: annealing} describes, so a matched batch and step count is not a matched budget of energy evaluations; we report no timing comparison. The FAB runs of this paper carry no replay buffer, so they are not the reference implementation, and ``FAB'' in every table and figure below means this buffer-free version. Its annealing follows the geometric path from the pushforward distribution to $\pi^2/\nu$ through $\pi$, with a rejuvenation kernel that leaves each level invariant and no replay buffer, as described in Section~\ref{subsec: annealing}. FAB discovers the third Three-Well basin, which forward KL does not, and the mass-covering $2$-divergence accounts for this. It does not reach every mode: on Himmelblau and on Sparse its coverage equals the forward KL value, so it misses the same wells. Since the replay buffer is part of what controls the gradient variance of the $2$-divergence, this deficit belongs to the buffer-free version run here and not necessarily to the reference implementation. Minimizing the $2$-divergence is minimizing importance weight variance, and FAB does hold a high ESS on every target; on Himmelblau and Sparse that ESS is earned on the modes it reached, and is a fake ESS in the sense of Section~\ref{sec: intro}. 

\begin{table}[htb]
    \centering
    \begin{tabular}{lccccc}
        \toprule
        target & FAB & KL & $\KL$+$\X_\pi$ & $\KL$+$\X_\pi$+$\X_{\hat\pi}$ & $\KL$+$\X_\pi$+$\X_{(\hat\pi+\bar\nu)/2}$ \\
        \midrule
        Two-Moon   & $0.906\,(1.000)$ & $0.884\,(1.000)$ & $0.894\,(1.000)$ & $0.849\,(1.000)$ & $\mathbf{0.931}\,(1.000)$ \\
        Three-Well & $0.988\,(0.962)$ & $0.977\,(0.760)$ & $0.991\,(0.760)$ & $0.968\,(0.970)$ & $\mathbf{0.989}\,(0.977)$ \\
        Himmelblau & $0.984\,(0.516)$ & $0.978\,(0.516)$ & $0.978\,(0.516)$ & $0.949\,(1.000)$ & $\mathbf{0.979}\,(1.000)$ \\
        Sparse     & $0.982\,(0.534)$ & $0.914\,(0.534)$ & $0.896\,(0.233)$ & $0.892\,(1.000)$ & $\mathbf{0.916}\,(1.000)$ \\
        \bottomrule
    \end{tabular}
    \caption{Final ESS, with coverage in parentheses, on the four 2D targets. Bold marks the highest ESS among the methods with coverage at least $0.9$, which are those whose panels occupy every target mode. The threshold is needed because the third Three-Well mode is narrow, so the reference balls there are only partly filled even when the mode is occupied.}
    \label{tab: 2d-benchmark-summary}
\end{table}

Table~\ref{tab: 2d-benchmark-summary} collects the four targets. The four rows support the division of labor the loss is built for: the target variation improves accuracy where the flow already places samples, the QT set supplies the missing modes, and the pushforward half of the mixture holds back the mass between them. KLXX reaches every mode on every target, and among the mode-complete methods it has the highest ESS on all four; on Three-Well its margin over FAB is within the fluctuation of one realization. The mechanism is visible term by term. Adding $\X_\pi$ leaves ESS near the forward KL value on Two-Moon, where coverage is already complete, and on Sparse it can even concentrate the flow further, onto one mode. QT samples supply information about modes absent from the target surrogate. Adding the pushforward samples to that mixture reduces visible intermodal mass relative to $\KL$+$\X_\pi$+$\X_{\hat\pi}$, and raises ESS on all four targets.

\subsection{High-dimensional product multi-well}
\label{subsec: highd}

We next isolate accuracy scaling over the dimensions
$d=16,32,48,\ldots,256$. Letting $k=\lfloor\log_2 d\rfloor$, we use
\begin{equation}
    U(\mathbf{x})=\frac12\lVert\mathbf{x}\rVert^2
    +12\sum_{i=1}^{k}\exp(-x_i^2).
    \label{eq: highd-potential}
\end{equation}
The first $k$ coordinates are symmetric double wells with minima at $\pm\sqrt{\ln 24}$ and barrier approximately $9.9$, while the remaining $d-k$ coordinates are standard Gaussian. The target therefore has exactly $2^k$ equal-weight modes, where $2^k\Le d<2^{k+1}$. Multimodality occupies only a $k$-dimensional subspace, but the log-weight error accumulates across all $d$ coordinates.

All five losses use the same identity-initialized NSF on $[-4,4]^d$. The fifth loss is LDR-L1, the forward KL with the centered penalty $\mathbb E_\pi|z-\mathbb E_\pi z|$ at coefficient $1$, written $\KL$+L1 below. At every dimension, training uses batch size $300$, $2000$ steps, and a fixed sample set of size $5000d$, of which QT uses $1000d$ samples at melt scale $m_e=2$; the SMC ladder has $M=2$. FAB is run at every dimension with the same flow, source, batch size and step count. Final sample ESS is evaluated on $80000$ fresh source samples. Each loss--dimension pair is one fixed-seed training realization.

\begin{figure}[htb]
    \centering
    \includegraphics[width=\textwidth]{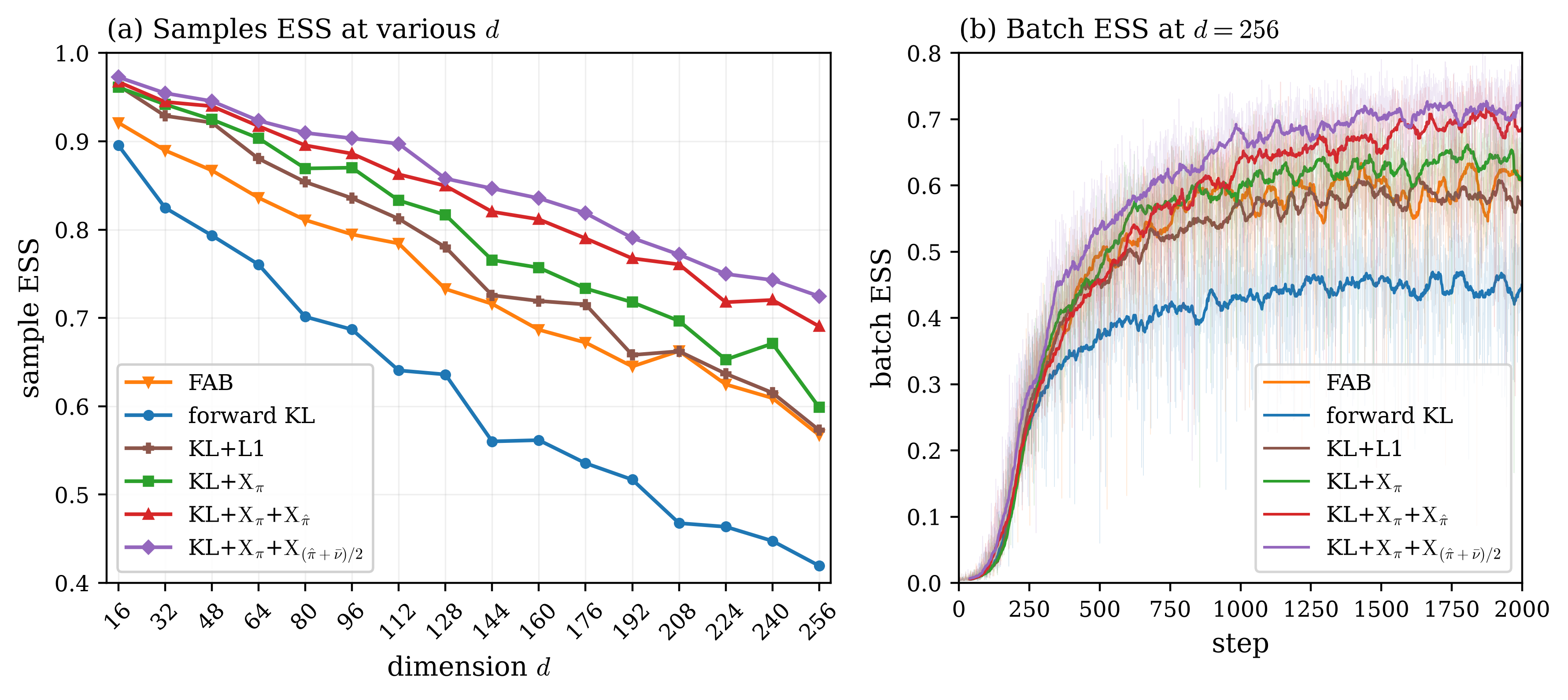}
    \caption{ESS on the product multi-well target \eqref{eq: highd-potential}. (a) Final sample ESS over sixteen tested dimensions. (b) Batch ESS during the $d=256$ training; bold curves are $40$-step moving averages of the raw values.}
    \label{fig: highd-ess}
\end{figure}

Coverage is not used here, because nearest-neighbor radii concentrate as the dimension grows. This target needs no such test: each of the first $k$ coordinates sits in one of two wells, which label the mode exactly. In each of the $96$ pairs of method and dimension, all $2^k$ modes occur among the fresh evaluation samples. Figure~\ref{fig: highd-ess} reports the final sample ESS. The three losses with log-ratio variation stay above forward KL at all sixteen dimensions, and so do $\KL$+L1 and FAB. The order is nearly the same at every dimension: KLXX first, then $\KL$+$\X_\pi$+$\X_{\hat\pi}$, then $\KL$+$\X_\pi$, $\KL$+L1, FAB, and forward KL. Two of the $80$ adjacent comparisons go the other way, at $d=16$ and $d=208$, by margins far below the run-to-run spread of Section~\ref{sec: conclusions}. The gap widens with dimension: at $d=16$ the six values lie within $0.08$ of each other, and at $d=256$ they span nearly four times that. During the $d=256$ run the batch ESS is evaluated after each gradient step, and the three losses with log-ratio variation keep improving after forward KL has begun to plateau.

\subsection{Coefficient of the target variation at \texorpdfstring{$d=100$}{d=100}}
\label{subsec: hd100}

The preceding test fixes the coefficient $\lambda=1$ of the target variation. We now vary it on the same
potential \eqref{eq: highd-potential} at $d=100$, where $k=6$ and the target has $64$ equal-weight modes.
The loss is $\KL$+$\lambda\X_\pi$ at $\lambda=0,0.5,1,1.5,2,2.5,3$ and at the integers from $4$ to $10$,
with $\lambda=0$ being forward KL; KLXX at the coefficients used everywhere else is trained on the same
target for comparison.

All runs share an identity-initialized NSF on $[-4,4]^{100}$. Training uses batch size $300$, $2000$ steps, and a fixed
sample set of $5\times10^{5}$ source samples, of which QT uses $10^{5}$ samples at melt scale $m_e=2$; the SMC
ladder has length $M=2$. Final sample ESS is evaluated on
$80000$ fresh source samples, and each coefficient is one fixed-seed training realization.

\begin{figure}[htb]
    \centering
    \includegraphics[width=\textwidth]{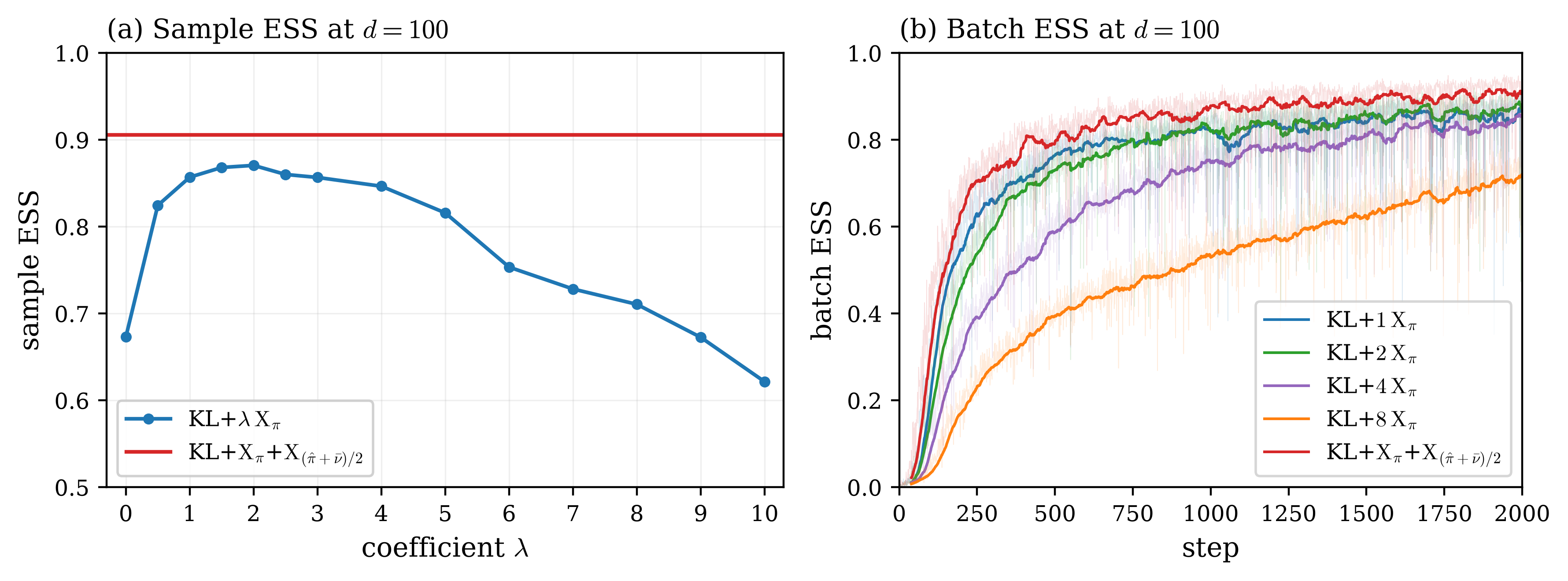}
    \caption{Effect of the target-variation coefficient on the product multi-well target
    \eqref{eq: highd-potential} at $d=100$. (a) Final sample ESS of $\KL$+$\lambda\X_\pi$ against $\lambda$;
    the horizontal line is the KLXX value $0.906$. (b) Batch ESS during training at $\lambda=1,2,4,8$ and for KLXX;
    bold curves are $40$-step moving averages of the raw values.}
    \label{fig: hd100-lambda}
\end{figure}

Figure~\ref{fig: hd100-lambda} shows a single interior optimum. Forward KL ends at sample ESS $0.673$; half a unit of
$\lambda$ already recovers most of the gain, and the curve peaks at
$\lambda=2$. The optimum is flat: every tested coefficient from $0.5$ to $5$ stays above $0.81$. Beyond it the
ESS declines steadily and ends below the forward KL value at $\lambda=10$: the variation term is
measured on the target surrogate alone, and weighting it far above the forward KL term removes the pressure
that places the pushforward distribution on the target in the first place.

The value of the target variation is thus insensitive to its coefficient over an order of magnitude, which
is why $\lambda=1$ is used untuned elsewhere in this paper. KLXX sits above every value the
sweep attains, so the mixture variation contributes something that tuning the single coefficient does not
reproduce. Panel (b) shows where the difference is made: the KLXX batch ESS separates from the
$\KL$+$\lambda\X_\pi$ curves within the first few hundred steps and stays above them for the rest of the
run, while raising $\lambda$ to $4$ or $8$ slows the early rise instead of accelerating it.

\subsection{Lattice \texorpdfstring{$\phi^4$}{phi4} field theory: fake ESS at a collective barrier}
\label{subsec: phi4}
On a periodic $L\times L$ lattice with $d=L^2$, the tilted scalar $\phi^4$ action is
\begin{equation*}
U(\boldsymbol\phi)
=-0.8\sum_{\langle jl\rangle}\phi_j\phi_l
+\sum_{j=1}^{d}\left[\phi_j^2+\frac12(\phi_j^2-1)^2\right]
+h\sum_{j=1}^{d}\phi_j.
\end{equation*}
The on-site term alone has a single minimum at $0$; the nearest-neighbor coupling renders a uniform configuration bistable. The magnetization of a configuration is
\begin{equation*}
    m(\boldsymbol\phi)=\frac1d\sum_{j=1}^{d}\phi_j .
\end{equation*}
We take $h=0.0257$ at $L=6$ and $h=0.0144$ at $L=8$, matching the scaling $hL^2\approx0.92$ to produce comparable well imbalance in the targets of the two settings.
A collective domain-wall barrier separates the two wells, and the minority weight $p_+=\mathbb P(m>0)$ must be learned rather than fixed by symmetry.

PT provides the references $p_+=0.1256\pm0.0001$ for $L=6$ and $0.1265\pm0.0002$ for $L=8$. Each reference uses $20$ replicas on a geometric inverse-temperature grid from $1$ down to $0.05$, hot enough to melt the domain-wall barrier; the quoted uncertainty is the spread over three independent runs at each size. Each loss uses an identity-initialized NSF on $[-3,3]^d$. The source is $\mathcal N(0,0.5^2 I_d)$; training uses a fixed sample set of $10^5$ source samples, batch size $1000$, and $2000$ steps. SMC and QT both act on that set, with ladder length $M=1$ and melt scale $m_e=2$. FAB is run on both lattices with the same flow, source, batch size and step count.

\begin{table}[htb]
    \centering
    \begin{tabular}{lcccc}
        \toprule
        & \multicolumn{2}{c}{$L=6$} & \multicolumn{2}{c}{$L=8$} \\
        \cmidrule(lr){2-3}\cmidrule(lr){4-5}
        & ESS & $p_+$ & ESS & $p_+$ \\
        \midrule
        FAB
        & $0.932/0.928/0.927$ & $0/0/0$
        & $0.880/0.877/0.870$ & $0/1/0$ \\
        forward KL
        & $0.854/0.843/0.872$ & $1/0/0$
        & $0.742/0.760/0.772$ & $0/0/0$ \\
        $\KL$+$\X_\pi$
        & $0.971/0.968/0.977$ & $1/0/0$
        & $0.922/0.907/0.920$ & $0/0/0$ \\
        $\KL$+$\X_\pi$+$\X_{\hat\pi}$
        & $0.940/0.946/0.942$ & $0.127/0.126/0.127$
        & $0.816/0.819/0.788$ & $0.126/0.127/0.124$ \\
        $\KL$+$\X_\pi$+$\X_{(\hat\pi+\bar\nu)/2}$
        & $0.946/0.952/0.957$ & $0.128/0.127/0.127$
        & $0.853/0.844/0.842$ & $0.132/0.129/0.131$ \\
        \midrule
        PT reference & --- & $0.1256\pm0.0001$ & --- & $0.1265\pm0.0002$ \\
        \bottomrule
    \end{tabular}
    \caption{Final ESS and reweighted minority-well weight $p_+$ for three seeds on the two tilted $\phi^4$ targets, $(L,h)=(6,0.0257)$ and $(8,0.0144)$. Integer values $0$ and $1$ indicate collapse onto one well.}
    \label{tab: phi4}
\end{table}

FAB, forward KL and $\KL$+$\X_\pi$ collapse at every size and seed (Table~\ref{tab: phi4}). Their target surrogate starts from the pushforward samples and does not cross the collective barrier within the run, so the second well remains absent and reweighting gives $p_+=0$ or $1$. Nevertheless, $\KL$+$\X_\pi$ holds the highest ESS at both sizes, and it is a fake ESS: it is earned on the single well the flow reaches, and the missing one does not enter it. KLXX populates both wells, so its ESS is measured on a support that carries both, and the number is no longer earned on half the target.

\begin{figure}[htb]
    \centering
    \includegraphics[width=\textwidth]{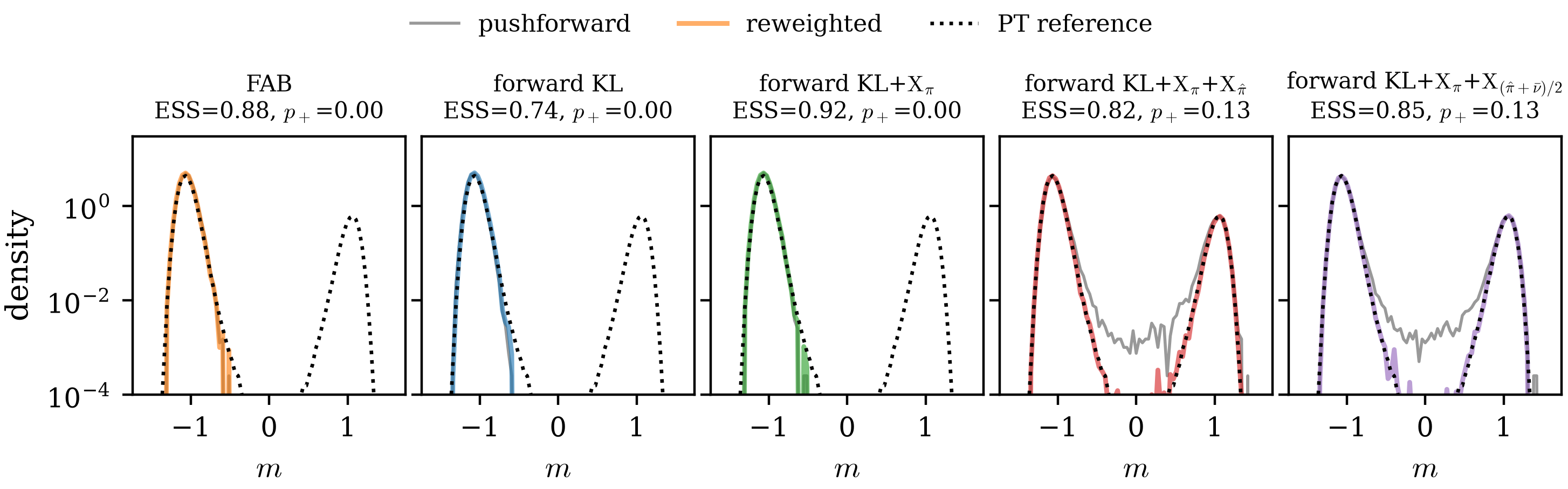}
    \caption{Magnetization density on the $L=8$ lattice. Each panel shows the pushforward density (gray), the reweighted density (color), and the PT reference density (dotted); the title reports ESS and reweighted $p_+$.}
    \label{fig: phi4-methods}
\end{figure}

The two losses that use the QT samples $\hat\pi$ populate both wells at every seed, with recovered $p_+$ within $0.006$ of the reference at the corresponding lattice size. On this observable the two are not equally accurate at $L=8$: $\KL$+$\X_\pi$+$\X_{\hat\pi}$ brackets the reference across its three seeds, whereas all three KLXX seeds lie above it. The offset is about $0.004$, larger than the spread over seeds, so it reflects the fit rather than sampling noise. Both QT-based losses still recover the minority well, which the collapsed fits never populate. In these finite-capacity fits the pushforward density places a little mass in the suppressed region between the wells, and the ESS stays below the value $\KL$+$\X_\pi$ reaches on its single well. Reweighting removes almost all of that intermodal mass, and both QT-based losses then match the reference there. KLXX reaches the higher ESS of the two at both sizes (Figure~\ref{fig: phi4-methods}).

\subsection{Boltzmann generator on \texorpdfstring{$p$}{p}-state clock model}
\label{subsec: clock}

The preceding benchmarks use a single flow. We now test the full adaptive-staging construction, on a target where one flow from the source to $\pi$ is not enough: the $p=6$ clock model on a periodic $8\times8$ lattice,
\begin{equation}
U(\boldsymbol\theta)
=-\sum_{\langle jl\rangle}\cos(\theta_j-\theta_l)
-\frac12\sum_{j=1}^{d}\cos(6\theta_j),
\qquad \boldsymbol\theta\in[-\pi,\pi)^{d},\qquad d=64,
\label{eq: clock-potential}
\end{equation}
where $\langle jl\rangle$ runs over nearest-neighbor pairs. The source is uniform on the torus and the interpolation is $U_t=tU$. The target distribution $\pi\propto\exp(-U)$ separates into six symmetry-related clock sectors.

Each stage uses an identity-initialized neural circular spline flow. The sample set contains $5\times10^5$ samples, every stage takes $3000$ steps, and both SMC and QT act directly on this set, with ladder length $M=4$. The melt of QT uses $m_e=2\pi$ on the torus, which spreads each particle over nearly the whole circle at every site. The stage schedule uses $t_{\mathrm{safe}}=0.25$ and the ESS gate at its default $\tau_{\mathrm v}=0.4$. We report batch sizes $B=2000$, $1500$, $1000$, and $500$.

KLXX uses the same stage schedule as $\KL$+$\X_\pi$ for fair comparison. That schedule is constructed with $\KL$+$\X_\pi$ under the ESS gate $\tau_{\mathrm v}=0.4$, and KLXX is trained on the identical accepted values of $t$. Every configuration is a single run, so individual stage ESS values carry single-run fluctuation.

The principal diagnostic is the sample ESS of each accepted stage. The propagation factor $\hat F$ of \eqref{eq: propagation-factor-hat} summarizes stagewise weight degeneracy across the stage sequence; smaller is better. It is not a full-chain ESS or an endpoint error estimate. Table~\ref{tab: clock-summary} reports every accepted stage and the resulting factor.

\begin{table}[htb]
    \centering
    \begin{tabular}{lcccccccc}
        \toprule
        stage $k$ & 1 & 2 & 3 & 4 & 5 & 6 & 7 & $\hat F$ \\
        \midrule
        $t_k$ ($B=2000$) & 0.250 & 0.625 & 0.754 & 0.889 & 1.000 & --- & --- & \\
        $\KL$+$\X_\pi$ & 0.885 & 0.431 & 0.492 & 0.433 & $\mathbf{0.666}$ & --- & --- & 4.30 \\
        $\KL$+$\X_\pi$+$\X_{(\hat\pi+\bar\nu)/2}$ & $\mathbf{0.937}$ & $\mathbf{0.553}$ & $\mathbf{0.554}$ & $\mathbf{0.456}$ & 0.662 & --- & --- & $\mathbf{3.40}$ \\
        \midrule
        $t_k$ ($B=1500$) & 0.250 & 0.512 & 0.648 & 0.789 & 0.937 & 1.000 & --- & \\
        $\KL$+$\X_\pi$ & 0.855 & 0.613 & 0.614 & 0.404 & 0.412 & $\mathbf{0.838}$ & --- & 4.72 \\
        $\KL$+$\X_\pi$+$\X_{(\hat\pi+\bar\nu)/2}$ & $\mathbf{0.919}$ & $\mathbf{0.708}$ & $\mathbf{0.681}$ & $\mathbf{0.440}$ & $\mathbf{0.419}$ & 0.832 & --- & $\mathbf{3.83}$ \\
        \midrule
        $t_k$ ($B=1000$) & 0.250 & 0.512 & 0.648 & 0.747 & 0.851 & 1.000 & --- & \\
        $\KL$+$\X_\pi$ & 0.813 & 0.530 & 0.562 & 0.573 & 0.507 & $\mathbf{0.448}$ & --- & 5.63 \\
        $\KL$+$\X_\pi$+$\X_{(\hat\pi+\bar\nu)/2}$ & $\mathbf{0.888}$ & $\mathbf{0.636}$ & $\mathbf{0.627}$ & $\mathbf{0.617}$ & $\mathbf{0.528}$ & 0.441 & --- & $\mathbf{4.43}$ \\
        \midrule
        $t_k$ ($B=500$) & 0.250 & 0.434 & 0.569 & 0.668 & 0.772 & 0.882 & 1.000 & \\
        $\KL$+$\X_\pi$ & 0.715 & 0.607 & 0.566 & 0.593 & 0.490 & 0.443 & $\mathbf{0.529}$ & 7.73 \\
        $\KL$+$\X_\pi$+$\X_{(\hat\pi+\bar\nu)/2}$ & $\mathbf{0.816}$ & $\mathbf{0.696}$ & $\mathbf{0.637}$ & $\mathbf{0.642}$ & $\mathbf{0.502}$ & $\mathbf{0.460}$ & 0.516 & $\mathbf{6.01}$ \\
        \bottomrule
    \end{tabular}
    \caption{Stage ESS, the sample ESS of the map selected at a stage, on every accepted stage of the schedule-matched clock-model comparison between $\KL$+$\X_\pi$ and $\KL$+$\X_\pi$+$\X_{(\hat\pi+\bar\nu)/2}$. Within each batch size block, both losses traverse the listed stage schedule. The final column is the propagation factor $\hat F$ in \eqref{eq: propagation-factor-hat}; bold marks the better loss at each stage and the smaller propagation factor.}
    \label{tab: clock-summary}
\end{table}

KLXX has higher sample ESS on $20$ of the $24$ shared stages, the exception being the last stage of each schedule, and it raises the geometric mean stage ESS at every reported batch size, by between $0.04$ and $0.06$. It reduces the propagation factor $\hat F$ at every reported batch size. As the batch shrinks, the schedule lengthens, and both factors grow with it, since each stage contributes a factor of at least one. Only the comparison within a batch size is a statement about the losses.

\begin{figure}[htb]
    \centering
    \includegraphics[width=0.7\textwidth]{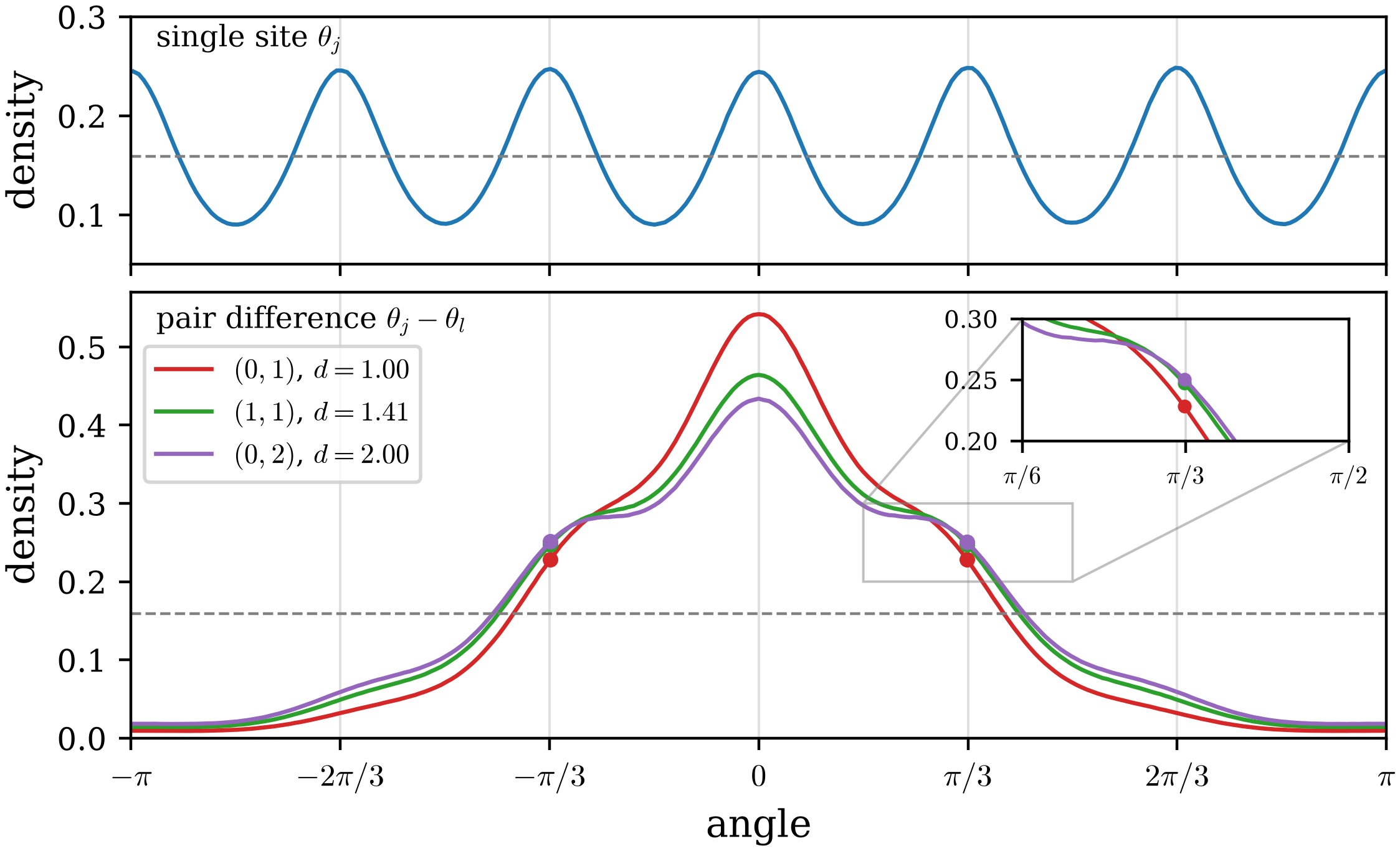}
    \caption{Angle densities from $2\times10^6$ fresh samples advanced through the trained $B=2000$ KLXX stage sequence, the single run of Table~\ref{tab: clock-summary}, by map, reweighting, resampling, and MALA. The upper panel is the single-site marginal. The lower panel gives pair differences at offsets $(0,1)$, $(1,1)$, and $(0,2)$; dots and the inset resolve the one-clock-step angle $\pi/3$. Dashed lines show the uniform density.}
    \label{fig: clock-marginals}
\end{figure}

The single-site marginal exhibits all six clock wells with nearly equal peak heights. Pair differences are most concentrated at zero for nearest neighbors; as the lattice separation increases, the alignment peak falls and the shoulders near $\pm\pi/3$ gain mass. This is the expected decay of local angular order with distance.

To quantify whether staged resampling preserves the sixfold symmetry, we perform two fresh-rebuild tests through the trained stage sequences. The batch size test fixes $N=6.4\times10^5$ and averages over four matched sampling seeds for each method and batch size. The particle-count test fixes $B=2000$ and varies $N$ at equal total work. For each sample, let $s$ be the clock sector nearest to the argument of the complex magnetization $d^{-1}\sum_{j=1}^d e^{\mathrm i\theta_j}$, and let $p_s$ be the empirical fraction assigned to sector $s$. We define
\[
\operatorname{occupancy\ bias}:=\frac16\sum_{s=0}^{5}\left|p_s-\frac16\right|.
\]

\begin{figure}[htb]
    \centering
    \includegraphics[width=0.9\textwidth]{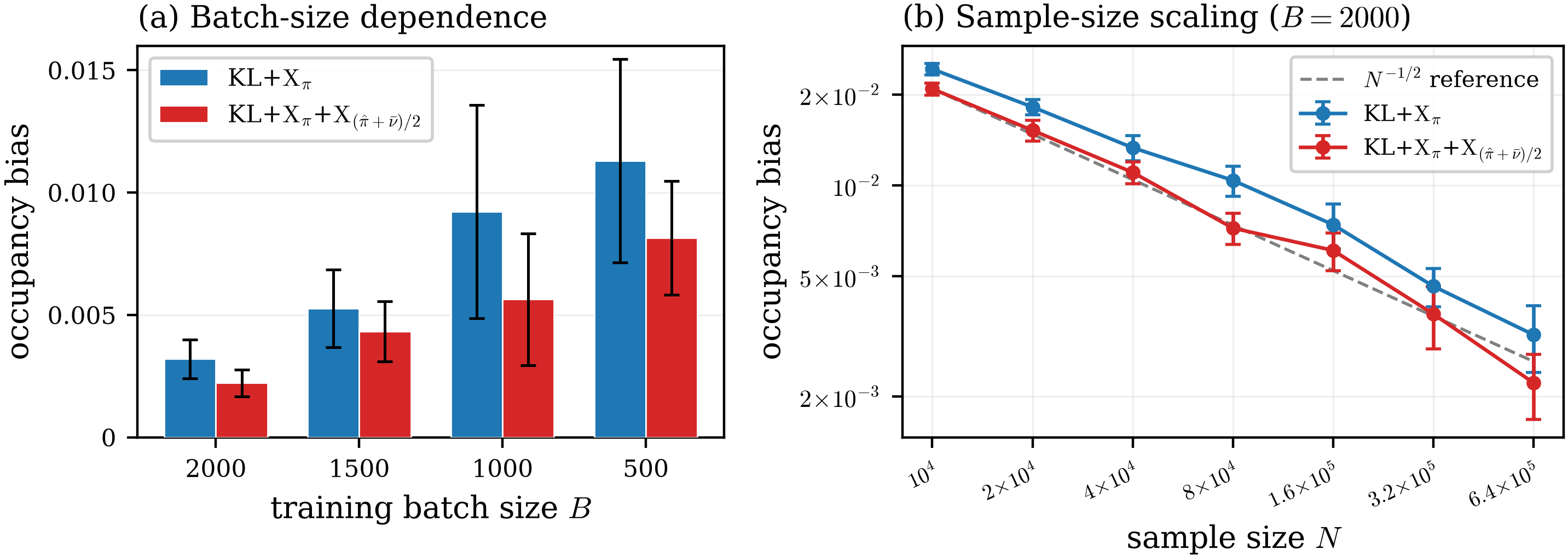}
    \caption{Clock-sector occupancy bias under two fresh-rebuild tests. (a) Mean bias at fixed $N=6.4\times10^5$ across training batch sizes; each bar averages four matched sampling seeds. (b) Mean bias at fixed $B=2000$ versus particle count; at $N=2^k\times10^4$, each mean uses $2^{8-k}$ independent rebuilds, keeping the total work per point at $2.56\times10^6$ particles. Error bars in both panels are $\pm2$ standard errors. The dashed line in (b) is an $N^{-1/2}$ reference, and the fitted log--log slopes are $-0.484$ for $\KL$+$\X_\pi$ and $-0.521$ for KLXX, and refitting at the fixed rate $\operatorname{err}=C/\sqrt N$ gives $C=2.67$ and $2.11$. The two panels' numbers align at $B=2000$, $N=6.4\times10^5$.}
    \label{fig: clock-occupancy}
\end{figure}

At fixed $N$, panel (a) of Figure~\ref{fig: clock-occupancy} shows that the mean occupancy bias increases as the training batch shrinks for both methods, and that KLXX has the lower mean at every reported batch size.

At fixed $B=2000$, panel (b) follows the Monte Carlo rate, with log--log slopes close to $-1/2$ for both losses. The absence of a visible bias floor is consistent with sampling fluctuation rather than a particle-count-independent sector imbalance over the tested range. The KLXX curve is lower at all seven reported sizes. Together with the six peaks in Figure~\ref{fig: clock-marginals}, these tests show that the staged samples are not confined to one global sector.

\subsection{Boltzmann generator on achiral molecules}
\label{subsec: achiral}

The targets above are analytic or lattice systems. We close with three molecules in implicit solvent at $300\,\mathrm K$, in the mixed internal coordinates of \citet{noe2019boltzmann}, whitened as that work does: $N$-methylacetamide (NMA, $d=30$), glycerol ($d=36$), and neutral diethanolamine ($d=48$). Here $d$ counts the internal degrees of freedom: a molecule of $N_{\mathrm a}$ atoms has $d=3N_{\mathrm a}-6$ once global translation and rotation are removed, so the three carry $12$, $14$ and $18$ atoms. NMA uses the Amber ff96 force field and the other two use GAFF2 with AM1-BCC charges; all three carry the OBC1 implicit solvent with the ACE nonpolar term, with no cutoff and no bond constraints.

The Lennard--Jones core grows as $r^{-12}$, so an early flow proposal can land at energies far outside the thermally occupied region. The stages therefore follow a regularized path instead of the direct interpolation used above. Flooring every nonbonded pair distance at $r$, and capping the reference-relative energy $\Delta E$ above a threshold $e$ by $e\bigl[1+\log(\Delta E/e)\bigr]$, gives a regularized potential $U^{\rho}$ with $\rho=(e,r)$, measured in $\mathrm{kJ\,mol^{-1}}$ and $\mathrm{nm}$. With $\rho_t=(1-t)\rho_0+t\rho_1$, the stage target is
\begin{equation}
	U_t=(1-t)U_0+t\,U^{\rho_t},
	\label{eq: achiral-diagonal}
\end{equation}
so the single parameter $t$ advances the interpolation and relaxes the regularization together. The path ends at $U^{\rho_1}$ rather than at $U$: the regularization is weakened along the schedule but not removed. The reference of Figure~\ref{fig: achiral-dihedrals} is generated under $U$ itself. Each run uses one seed, a sample set of $2\times10^5$, batch size $10^4$, and $500$ training steps per stage, $400$ for diethanolamine. The schedule uses $t_{\mathrm{safe}}=0.1$ and the ESS gate at its default $\tau_{\mathrm v}=0.4$, and the KLXX runs use $c_{\mathrm{QT}}=0.2$ in the QT reweighting of Section~\ref{subsec: QT}. On these targets a fraction $10^{-4}$ of the largest importance weights is dropped from the loss batch and given zero weight before every stage ESS and every resampling, so that a single hole of the pushforward density cannot dominate a batch or a stage; the reported $\hat F$ is assembled from those screened values. Both losses build their own adaptive schedule here, so unlike Section~\ref{subsec: clock} the two are not compared stage by stage.

\begin{table}[htb]
    \centering
    \begin{tabular}{lcccccc}
        \toprule
        \multirow{2}{*}{molecule ($d$)} & \multirow{2}{*}{$\rho_0$} & \multirow{2}{*}{$\rho_1$}
        & \multicolumn{2}{c}{$\KL$+$\X_\pi$} & \multicolumn{2}{c}{KLXX} \\
        \cmidrule(lr){4-5}\cmidrule(l){6-7}
        & & & $\hat F$ & stages & $\hat F$ & stages \\
        \midrule
        NMA ($30$) & $(50,0.20)$ & $(100,0.15)$ & $2.06$ & $6$ & $\mathbf{1.77}$ & $5$ \\
        glycerol ($36$) & $(50,0.20)$ & $(100,0.10)$ & $3.77$ & $7$ & $\mathbf{3.28}$ & $6$ \\
        neutral diethanolamine ($48$) & $(50,0.20)$ & $(100,0.10)$ & $19.90$ & $10$ & $\mathbf{15.84}$ & $10$ \\
        \bottomrule
    \end{tabular}
    \caption{Propagation factor $\hat F$ of \eqref{eq: propagation-factor-hat} and the number of accepted stages for the three achiral molecules, with the regularization path endpoints $\rho_0$ and $\rho_1$ of \eqref{eq: achiral-diagonal}. Each entry is one seed on that loss's own adaptive schedule; bold marks the smaller factor for the molecule.}
    \label{tab: achiral-summary}
\end{table}

Table~\ref{tab: achiral-summary} shows that KLXX has the smaller propagation factor on all three targets, and that it reaches $t=1$ in one fewer stage on NMA and on glycerol. The factors grow steeply with dimension. On diethanolamine the last five stages of both runs selected the identity proposal over the trained flow, so over the second half of both schedules the generators advance by reweighting, resampling and MALA alone, and the factor there describes the schedule rather than either loss.

Figure~\ref{fig: achiral-stage-ess} resolves each factor into the accepted stages that produce it. Because the two schedules differ, the arcs of one molecule are read as two separate progressions to $t=1$ rather than as matched transitions.

\begin{figure}[!tp]
    \centering
    \includegraphics[width=\textwidth]{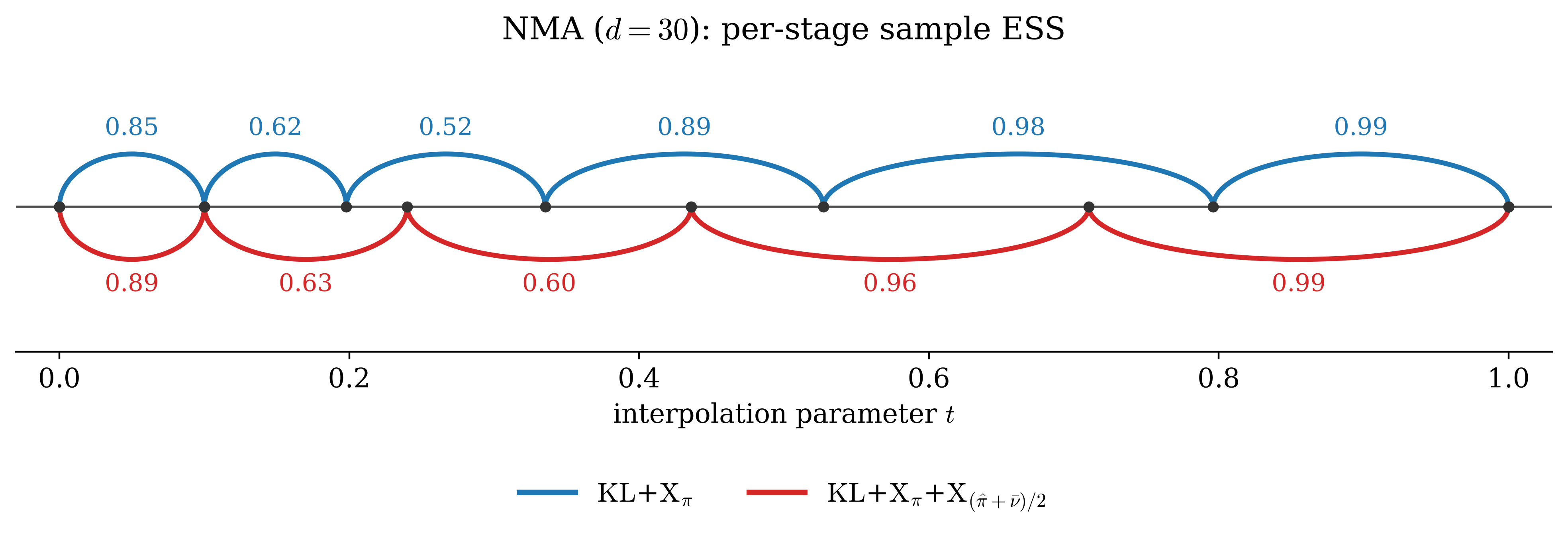}
    \par\smallskip
    \includegraphics[width=\textwidth]{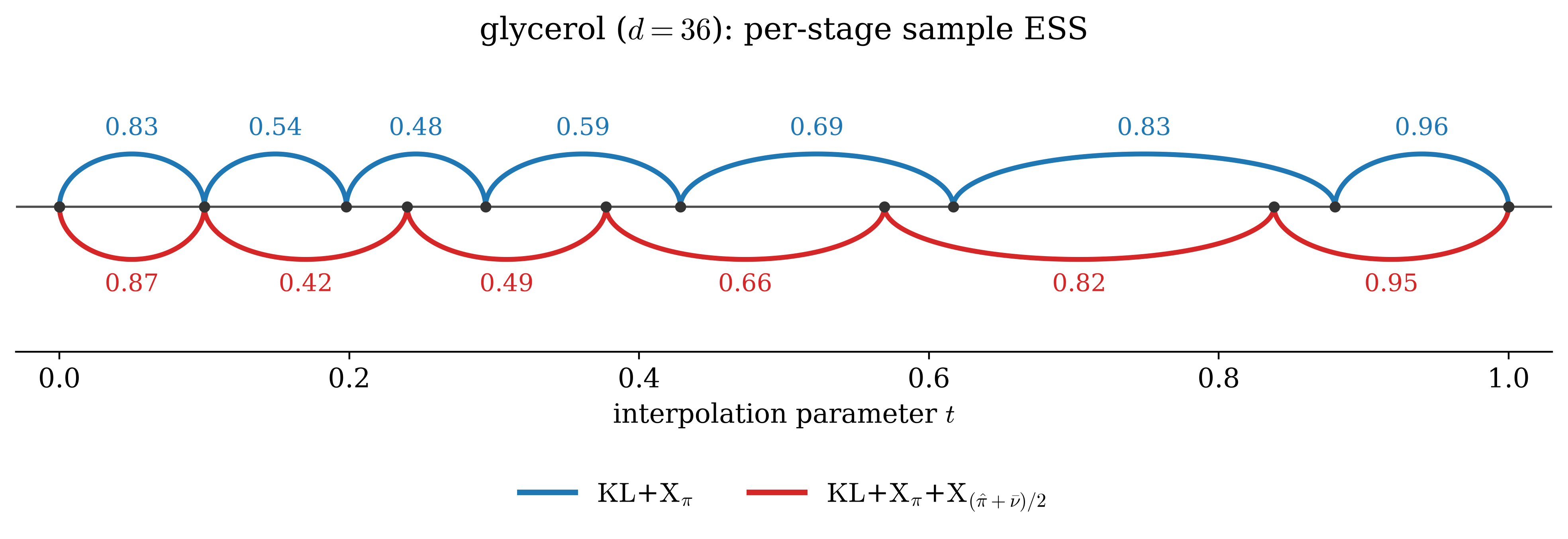}
    \par\smallskip
    \includegraphics[width=\textwidth]{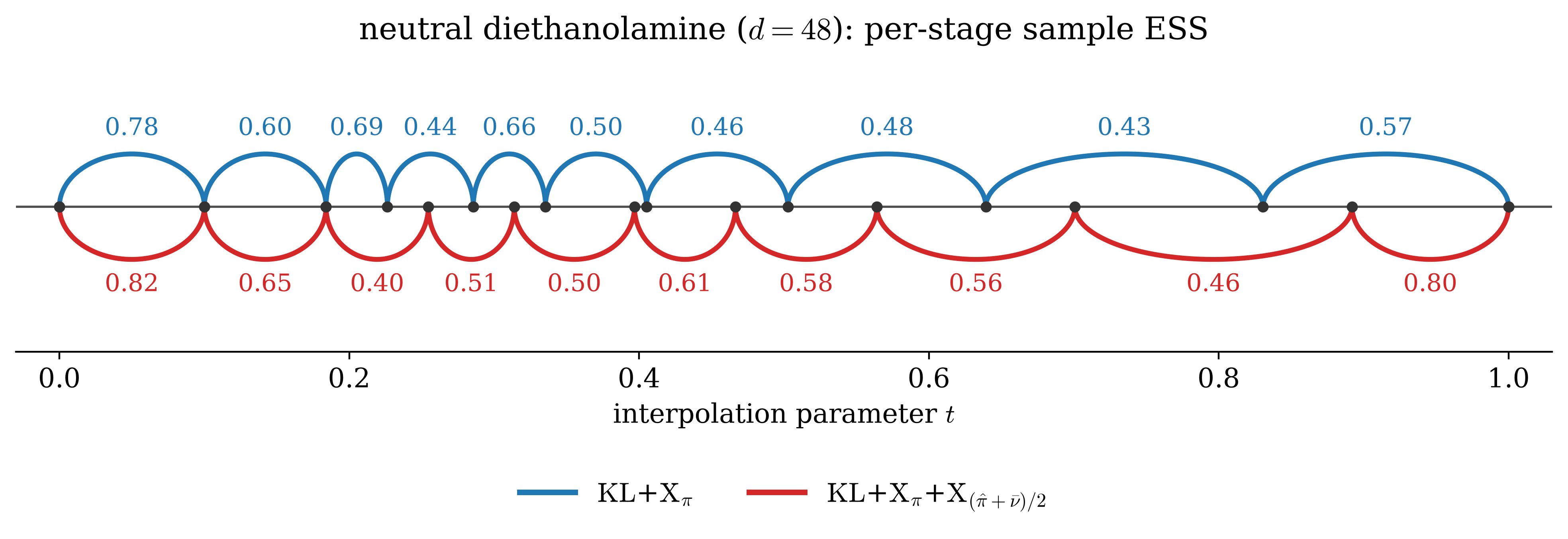}
    \caption{Per-stage sample ESS for NMA (top), glycerol (middle), and neutral diethanolamine (bottom). Each arc spans one accepted stage $t_{k-1}\to t_k$ and carries the sample ESS of the map selected there. Blue arcs above the axis are $\KL$+$\X_\pi$; red arcs below are KLXX. The horizontal coordinate is the interpolation parameter $t\in[0,1]$.}
    \label{fig: achiral-stage-ess}
\end{figure}

The samples at $t=1$ are compared with an independent PT reference in Figure~\ref{fig: achiral-dihedrals}. Each reference propagates six OpenMM replicas from $300$ to $800\,\mathrm K$ and retains the $300\,\mathrm K$ replica after $200$ equilibration rounds, giving $3300$, $5500$ and $9700$ frames for NMA, glycerol and neutral diethanolamine. One torsion is shown per molecule: the NMA amide torsion, and the C-C-C-O and C-C-N-C torsions carrying the trans and gauche wells of the other two.

\begin{figure}[htb]
    \centering
    \includegraphics[width=\textwidth]{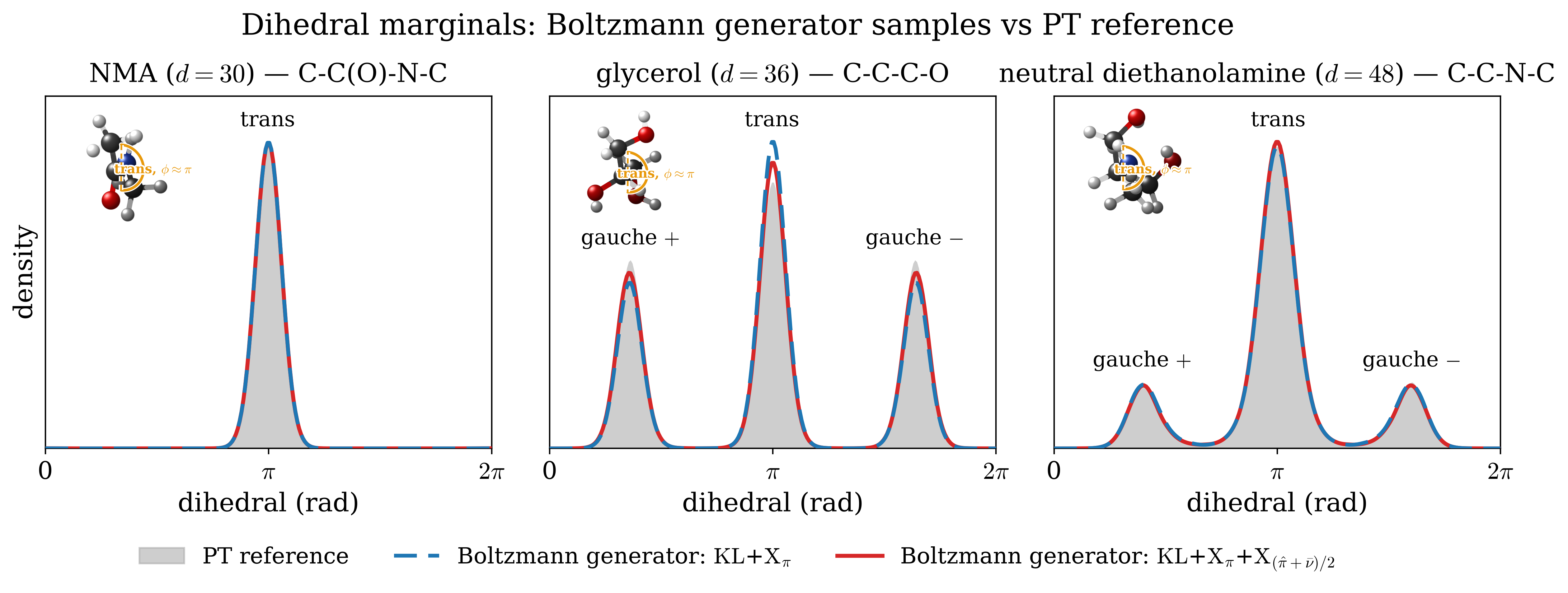}
    \caption{Dihedral marginals from the Boltzmann generator samples at $t=1$ and the PT references, for NMA, glycerol, and neutral diethanolamine. Gray is the reference, dashed blue is $\KL$+$\X_\pi$, and solid red is KLXX; the dashed curve is drawn over the solid one so that agreement between the two remains visible. Every density is symmetrized under $\varphi\mapsto-\varphi$, as the achiral targets allow, and smoothed, with the reference and both generators treated alike. The insets show a representative trans conformer.}
    \label{fig: achiral-dihedrals}
\end{figure}

Both losses populate every well the reference visits. On NMA that is the single trans well; its gauche population is too small to appear. On neutral diethanolamine they also match its trans fraction, which the symmetrization of the figure leaves untouched. On glycerol both place more mass in trans than the reference does, KLXX by about half of $\KL$+$\X_\pi$'s excess, which is the largest visible departure in the figure. The reference frames are correlated, and every generator configuration is a single seed.

% ===================== Section 7: Conclusions + Acknowledgement =====================
\section{Conclusions}
\label{sec: conclusions}

Forward KL is mass-covering for exact distributions, but its sampled loss sees only the target regions supplied to it. This gap can produce a fake ESS, transfer surrogate error to the flow, and leave target-poor pushforward regions weakly constrained. Log-ratio variation addresses these failures by comparing the exact log-density ratio at pairs of samples. The weighting measure determines which regions are compared: target samples measure log-ratio variation on reached target regions, QT samples discover candidate modes, and pushforward samples expose regions occupied by the pushforward distribution. KLXX combines these log-ratio variations while retaining the exact target as a minimizer and a tractable pushforward density.

The theory covers the generalized KLXX loss at any coefficients $(\lambda,\vartheta,\alpha,\beta)$. In FR geometry the loss induces the flow \eqref{eq: FR-flow}, whose dissipation \eqref{eq: combined-rate} carries both log-ratio variations as nonpositive terms, with no log-concavity, spectral gap, or growth assumption on $\pi$. Under a fixed biased target surrogate, Theorem~\ref{thm: accuracy} subtracts one nonnegative discrepancy for each variation from the stationary error bound, requiring a positive solution of its fixed-point equation and two finite importance-weight expectations. Theorem~\ref{thm: propagation} is separate: it bounds the $L^2$ error of the inference scheme of Algorithm~\ref{alg: inference} at rate $N^{-1/2}$ for every bounded observable, with a constant independent of $N$, so the scheme is asymptotically unbiased, provided each stage weight is $\nu_k$-essentially bounded. The reported $\hat F$ summarizes stagewise weight degeneracy and is a computable substitute for that constant rather than a bound on it.

The model problems are consistent with the roles of the target, QT, and pushforward samples. On Three-Well, Himmelblau and Sparse, forward KL and $\KL$+$\X_\pi$ miss modes while reporting a high ESS, and on the $\phi^4$ lattices they collapse to one well with $\KL$+$\X_\pi$ holding the highest ESS at both sizes; both QT-based losses recover every mode on these targets, and adding the pushforward samples to the mixture reduces the visible intermodal mass on Himmelblau and Sparse. The product multi-well test shows higher sample ESS than forward KL for every loss with log-ratio variation at all sixteen tested dimensions, and places $\KL$+L1 between forward KL and $\KL$+$\X_\pi$ at every dimension but the smallest. At $d=100$, the sweep of the target-variation coefficient has a single flat interior optimum near $\lambda=2$, and KLXX exceeds every value it attains. On the shared clock stage schedules, KLXX has higher sample ESS than the loss that built the schedule on all but its last stage, and a smaller propagation factor at every reported batch size. On the three achiral molecules, KLXX reaches $t=1$ with the smaller propagation factor at every target. The $\phi^4$, clock and molecular experiments carry the accuracy evidence: the $\phi^4$ minority weight is compared against PT, the clock sector occupancy against the exact value $1/6$ that the symmetry of the target fixes, and the molecular dihedral marginals against PT references.

Every configuration outside the $\phi^4$ tests is a single training realization. The three $\phi^4$ seeds give the only estimate of run-to-run spread in the paper, up to $0.031$ in final ESS, which exceeds several of the margins reported elsewhere, so those ESS orderings are indicative rather than measured differences.

The principle those results support is not an ordering of methods: what the mixture variation contributes is a way past mode collapse, and it does so through an explicitly constructed QT set, so the candidate modes enter the loss itself rather than being left to the sampler to find. Coverage and the ESS then measure what that training achieved. Future work includes tuning the KLXX loss for better performance, larger molecular targets, and more challenging cases such as a Lennard--Jones system and distributions on low-dimensional manifolds. A replay buffer would let the QT set and the target surrogate of earlier steps be reused, as FAB reuses its annealed samples \citep{midgley2022flow}, and adding one to KLXX is left to future work. The comparison with LDR-L1 \citep{schopmans2026ldr} is confined to the product multi-well here, and we will extend it to the other targets, add temperature-annealed Boltzmann generators \citep{schopmans2025temperature}, and run FAB with the replay buffer of its reference implementation. 

\subsection*{Data Availability}

The normalizing flow computations use the JAX backend through two Python packages maintained by Xuda Ye. The packages \texttt{jflows} and \texttt{jflows-md} can be installed directly from PyPI with:
\begin{tcolorbox}[colback=gray!12,boxrule=0pt,arc=2pt,left=6pt,right=6pt,top=4pt,bottom=4pt]
\texttt{pip install jflows jflows-md}
\end{tcolorbox}
\noindent
The source files are available at:
\begin{tcolorbox}[colback=gray!12,boxrule=0pt,arc=2pt,left=6pt,right=6pt,top=4pt,bottom=4pt]
\url{https://github.com/xuda-ye-math/KLXX}
\end{tcolorbox}

\subsection*{Declaration of AI use}
The authors used Claude Opus 5 (Anthropic) to assist with writing and testing the codes, with drafting and checking proofs, and with reviewing and editing the manuscript, including the checking of numerical claims quoted from other publications. All data and numerical results reported in this paper are generated directly by our Python scripts and are presented without any embellishment. The authors have reviewed all AI-assisted content. The authors assume responsibility for all content.

{
	\bibliographystyle{abbrvnat}
	\bibliography{references}
}

\appendix

\section{Proof of the Fisher--Rao flow and its dissipation}
\label{sec: appendix-FR}

This appendix derives the Fisher--Rao gradient flow \eqref{eq: FR-flow} of the generalized KLXX loss \eqref{eq: KLXX-general-nu} and proves the dissipation identity \eqref{eq: combined-rate} stated in Section~\ref{subsec: FR-convergence}. We view $\nu$ as a generic probability density evolving under a Riemannian metric on the space of densities \citep{ambrosio2008gradient,villani2009optimal} and ask how fast $\nu_t$ approaches the target in $\KL(\pi\|\nu_t)$.
Throughout, we assume the densities are positive and regular enough for differentiation, and that the expectations displayed below are finite.

The FR gradient of a functional $\mathcal F[\nu]$ and the flow it generates are given by
\begin{equation}
    \mathrm{grad}^{\mathrm{FR}}\mathcal F(\nu) = \nu\biggl(\frac{\delta\mathcal F}{\delta\nu} - \mathbb E_{\nu}\Big[\frac{\delta\mathcal F}{\delta\nu}\Big]\biggr),
    \qquad
    \partial_t\nu_t = -\,\mathrm{grad}^{\mathrm{FR}}\mathcal F(\nu_t).
    \label{eq: appendix-FR-gradient}
\end{equation}
The mean subtraction enforces $\int_{\mathbb R^d}\mathrm{grad}^{\mathrm{FR}}\mathcal F\D x = 0$ at unit mass, so the flow preserves probability mass.
In the generalized KLXX loss \eqref{eq: KLXX-general-nu},
the forward KL contributes a local term: its first variation is $\delta\KL(\pi\|\nu)/\delta\nu = -w$, and $\mathbb E_\nu[w] = \int_{\mathbb R^d}\pi\D x = 1$, so \eqref{eq: appendix-FR-gradient} gives
\begin{equation}
    \mathrm{grad}^{\mathrm{FR}}\bigl[\KL(\pi\|\nu)\bigr](x) = \nu(x)-\pi(x).
    \label{eq: appendix-FR-grad-KL}
\end{equation}
The two log-ratio variations are handled at once, because neither weighting distribution is varied and nothing in the calculation uses that the weighting distribution and the target coincide.

\begin{lemma}
\label{lem: appendix-variation}
Let $\omega$ be a weighting distribution that does not depend on $\nu$, and let $S^{\omega}$ be its nonlocal correlation \eqref{eq: FR-S}. Then
\begin{equation}
    \frac{\delta \X_\omega(\pi\|\nu)}{\delta\nu}(x) = -2\,\frac{\omega(x)}{\nu(x)}\,S^{\omega}(x),
    \qquad
    \mathrm{grad}^{\mathrm{FR}}\bigl[\X_\omega(\pi\|\nu)\bigr](x) = -2\,\omega(x)\,S^{\omega}(x).
    \label{eq: appendix-FR-grad-X}
\end{equation}
\end{lemma}

\begin{proof}
Since $z = \log\pi-\log\nu$, the variation of the inner difference at an evaluation point $r$ is
\begin{equation*}
    \frac{\delta}{\delta\nu(r)}\bigl[z(x)-z(y)\bigr]
    = -\frac{\delta(r-x)}{\nu(x)} + \frac{\delta(r-y)}{\nu(y)},
\end{equation*}
with $\delta$ the Dirac distribution, so integration against $\delta(r-x)$ evaluates the remaining integrand at $x=r$. Because $\log$ is increasing, $\operatorname{sgn}(z(x)-z(y)) = \operatorname{sgn}(w(x)-w(y))$. Differentiating \eqref{eq: X} and using that $\omega$ is fixed,
\begin{align*}
    \frac{\delta \X_\omega(\pi\|\nu)}{\delta\nu}(r)
    &= \iint_{\mathbb R^d\times\mathbb R^d} \omega(x)\omega(y)\operatorname{sgn}\bigl(w(x)-w(y)\bigr)
       \left[-\frac{\delta(r-x)}{\nu(x)} + \frac{\delta(r-y)}{\nu(y)}\right] \D x \D y \\
    &= -2\,\frac{\omega(r)}{\nu(r)}\int_{\mathbb R^d} \omega(y)\operatorname{sgn}\bigl(w(r)-w(y)\bigr)\D y ,
\end{align*}
the two terms being equal after exchanging $x$ and $y$ and using the antisymmetry of the sign. That antisymmetry also gives
\begin{equation}
    \mathbb E_{\omega}\bigl[S^{\omega}\bigr] = \mathbb E_{x,y\sim\omega}\bigl[\operatorname{sgn}\bigl(w(x)-w(y)\bigr)\bigr] = 0 ,
    \label{eq: appendix-S-mean-zero}
\end{equation}
so the mean subtraction in \eqref{eq: appendix-FR-gradient} contributes $\mathbb E_\nu\bigl[-2(\omega/\nu)S^{\omega}\bigr] = -2\int_{\mathbb R^d}\omega S^{\omega}\D x = 0$, and the FR gradient is $\nu\cdot\bigl(-2(\omega/\nu)S^{\omega}\bigr)$.
\end{proof}

The mixture $\xi=\alpha\hat\pi+\beta\bar\nu$ depends on $\nu$ through $\bar\nu$, but $\bar\nu$ is detached and takes no part in the gradient calculation, so $\xi$ is held fixed when the variation is taken. The lemma therefore applies at $\omega=\xi$ as it does at $\omega=\pi$. Adding \eqref{eq: appendix-FR-grad-KL} to the two variations, the FR gradient of the generalized loss is
\begin{equation*}
    \mathrm{grad}^{\mathrm{FR}}[\mathcal L](x) = \nu(x)-\pi(x) - 2\lambda\,\pi(x)S^{\pi}(x) - 2\vartheta\,\xi(x)S^{\xi}(x) ,
\end{equation*}
and the flow \eqref{eq: FR-flow} is $\partial_t\nu_t=-\,\mathrm{grad}^{\mathrm{FR}}\mathcal L(\nu_t)$. It is a nonlocal differential equation: its right-hand side at $x$ needs the entire profile of the importance weight, under $\pi$ through $S^{\pi}$ and under $\xi$ through $S^{\xi}$. Integrating it and using \eqref{eq: appendix-S-mean-zero} at both weighting distributions, the two correlation terms drop out and
\begin{equation*}
    \frac{\D}{\D t}\!\int_{\mathbb R^d}\!\nu_t(x)\D x = 1 - \!\int_{\mathbb R^d}\!\nu_t(x)\D x ,
\end{equation*}
so a unit mass at $t=0$ is preserved for all $t$. Mass neutrality of the correlation terms follows from antisymmetry alone and requires no normalization. The condition $\alpha+\beta=1$ enters elsewhere: it makes $\xi$ a probability density, hence $S^{\xi}\in[-1,1]$, and it is what lets the mixture integrals below be read as expectations under $\xi$.

\begin{remark}
\label{rem: appendix-subgradient}
The integrand $|z(x)-z(y)|$ has no derivative on the tie set $\{z(x)=z(y)\}$, where its subdifferential is $[-1,1]$ and the convention $\operatorname{sgn}(0)=0$ selects the element of least norm. This costs differentiability of $\nu\mapsto\X_\omega(\pi\|\nu)$ only where the tie set carries positive $\omega\otimes\omega$ mass, the diagonal alone being null; it does so at $\nu=\pi$, where $w\equiv1$ and every pair ties. Writing $u=\log\nu$ makes $z(x)-z(y)$ affine in $u$, so $\X_\omega$ is convex in $u$ and the second identity of \eqref{eq: appendix-FR-grad-X} is a subgradient there. Accordingly \eqref{eq: FR-flow} is a differential inclusion, and we still write gradient flow for simplicity.
\end{remark}

Lemma~\ref{lem: appendix-variation} gives the FR gradient of each log-ratio variation, and those terms already stand in the flow \eqref{eq: FR-flow}. We differentiate $\KL(\pi\|\nu_t)$ along that flow. The contribution of each variation to $\frac{\D}{\D t}\KL(\pi\|\nu_t)$ is then an inner product of the importance weight $w_t$ against $S^{\omega}$, and the following identity determines its sign. The mixture term needs the identity at the weighting distribution $\xi=\alpha\hat\pi+\beta\bar\nu$.

\begin{lemma}
\label{lem: appendix-correlation}
Let $\omega$ be a weighting distribution with $\mathbb E_{\omega}[w]<\infty$. Then
\begin{equation}
    \mathbb E_{x\sim\omega}\bigl[w(x)S^{\omega}(x)\bigr]
    = \frac12\,\mathbb E_{x,y\sim\omega}\bigl[|w(x)-w(y)|\bigr] \Ge 0 ,
    \label{eq: appendix-correlation}
\end{equation}
with equality if and only if $w$ is $\omega$-almost-everywhere constant.
\end{lemma}

\begin{proof}
By the definition \eqref{eq: FR-S} of $S^{\omega}$,
\begin{equation*}
    \mathbb E_{x\sim\omega}\bigl[w(x)S^{\omega}(x)\bigr]
    = \mathbb E_{x,y\sim\omega}\bigl[w(x)\operatorname{sgn}\bigl(w(x)-w(y)\bigr)\bigr] ,
\end{equation*}
which is finite because $\mathbb E_\omega[w]<\infty$. Symmetrizing in $x\leftrightarrow y$ and using the antisymmetry of the sign,
\begin{equation*}
    \mathbb E_{x\sim\omega}\bigl[w(x)S^{\omega}(x)\bigr]
    = \frac12\,\mathbb E_{x,y\sim\omega}\bigl[(w(x)-w(y))\operatorname{sgn}\bigl(w(x)-w(y)\bigr)\bigr]
    = \frac12\,\mathbb E_{x,y\sim\omega}\bigl[|w(x)-w(y)|\bigr] .
\end{equation*}
The last expression vanishes exactly when $w(x)=w(y)$ for $\omega\otimes\omega$-almost every pair.
\end{proof}

Differentiating $\KL(\pi\|\nu_t)$ along \eqref{eq: FR-flow} with $\delta\KL/\delta\nu = -w_t$ and the correlations \eqref{eq: FR-S-t}, we obtain the dissipation relation as in \eqref{eq: combined-rate}:
\begin{align*}
    \frac{\D}{\D t}\KL(\pi\|\nu_t)
    &= -\int_{\mathbb R^d} w_t\,\bigl[\pi-\nu_t+2\lambda\,\pi S^{\pi}_t{+2\vartheta\,\xi S^{\xi}_t}\bigr]\D x \\
    &= -\biggl[\int_{\mathbb R^d}\frac{\pi^2}{\nu_t}\D x - 1\biggr]
       - 2\lambda\,\mathbb E_{x\sim\pi}\bigl[w_t S^{\pi}_t\bigr]
       - 2\vartheta\,\mathbb E_{x\sim\xi}\bigl[w_t S^{\xi}_t\bigr] \\
    & = -\chi^2(\pi\|\nu_t) - \lambda\,\mathbb E_{x,y\sim\pi}\bigl[|w_t(x)-w_t(y)|\bigr]
    - \vartheta\,\mathbb E_{x,y\sim\xi}\bigl[|w_t(x)-w_t(y)|\bigr],
\end{align*}
where the last equality holds by applying Lemma~\ref{lem: appendix-correlation} at $\omega = \pi$ and $\omega = \xi$, respectively.

The elementary inequality $u\log u\Le u^2-u$ for $u>0$, applied at $u=w_t$ and integrated against $\nu_t$, gives $\chi^2(\pi\|\nu_t)\Ge\KL(\pi\|\nu_t)$. Discarding the two nonpositive variation terms of \eqref{eq: combined-rate} therefore leaves
\begin{equation*}
    \frac{\D}{\D t}\KL(\pi\|\nu_t) \Le -\chi^2(\pi\|\nu_t) \Le -\KL(\pi\|\nu_t) ,
\end{equation*}
and Gr\"onwall's inequality yields $\KL(\pi\|\nu_t)\Le e^{-t}\KL(\pi\|\nu_0)$, without assuming log-concavity, a spectral gap, or a growth condition on $\pi$, and uniformly in $(\lambda,\vartheta,\alpha,\beta)$. This is a statement in density space: the rate is the one forward KL already carries, and what the variations change is the structure of the flow rather than its speed. This contrasts with the Wasserstein-2 geometry, where the forward KL is not displacement convex and an analogous exponential rate is unavailable without further assumptions.

\section{Proof of accuracy under biased target surrogate}
\label{sec: appendix-accuracy}

This appendix derives the stationarity identity \eqref{eq: surrogate-stationary} and proves Theorem~\ref{thm: accuracy} on the accuracy of the biased KLXX loss $\tilde{\mathcal L}$ in \eqref{eq: KLXX-biased} under a biased target surrogate $\tilde\pi$. Applying \eqref{eq: appendix-FR-grad-KL} with $\tilde\pi$ in place of $\pi$ to the first term of \eqref{eq: KLXX-biased}, and Lemma~\ref{lem: appendix-variation} to the other two at $\omega=\tilde\pi$ and at $\omega=\xi$, gives
\begin{equation*}
    \mathrm{grad}^{\mathrm{FR}}\bigl[\tilde{\mathcal L}\bigr](x) = \nu(x) - \tilde\pi(x)\bigl(1+2\lambda S^{\tilde\pi}(x)\bigr) - 2\vartheta\,\xi(x)S^{\xi}(x) .
\end{equation*}
Hence the stationary density $\nu_\star$ of $\tilde{\mathcal L}$ satisfies \eqref{eq: surrogate-stationary}, that is
\begin{equation*}
	\nu_\star(x) = \tilde\pi(x)\bigl(1 + 2\lambda\,S^{\tilde\pi}_\star(x)\bigr) + 2\vartheta\,\xi(x)\,S^{\xi}_\star(x) .
\end{equation*}
Integrating this identity and using \eqref{eq: appendix-S-mean-zero} at $\omega=\tilde\pi$ and at $\omega=\xi$,
\begin{equation*}
	\int_{\mathbb R^d} \nu_\star(x)\D x = \int_{\mathbb R^d} \tilde\pi(x)\D x
	+2\lambda\int_{\mathbb R^d} \tilde\pi(x)S^{\tilde\pi}_\star(x)\D x +
	2\vartheta\int_{\mathbb R^d}\xi(x)S^{\xi}_\star(x)\D x = 1 ,
\end{equation*}
so $\nu_\star$ has unit mass.

\begin{proof}
Define the ratio $\varrho := \nu_\star/\tilde\pi$, positive because both densities are, so that \eqref{eq: surrogate-stationary} divided by $\tilde\pi$ reads
\begin{equation}
    \varrho(x) = 1 + 2\lambda\,S^{\tilde\pi}_\star(x) + 2\vartheta\,\frac{\xi(x)}{\tilde\pi(x)}\,S^{\xi}_\star(x) .
    \label{eq: appendix-varrho}
\end{equation}
Multiply \eqref{eq: appendix-varrho} by $\tilde\pi w_\star$, with $w_\star=\pi/\nu_\star$, and integrate. In the last term $\tilde\pi\cdot(\xi/\tilde\pi)=\xi$, so that integral is carried by the mixture rather than by the surrogate; subtracting $\int_{\mathbb R^d}\tilde\pi w_\star\D x$ from both sides,
\begin{equation}
    \int_{\mathbb R^d}\!\tilde\pi\,w_\star(\varrho-1)\D x
    = 2\lambda\!\int_{\mathbb R^d}\!\tilde\pi\,w_\star S^{\tilde\pi}_\star\D x
    + 2\vartheta\!\int_{\mathbb R^d}\!\xi\,w_\star S^{\xi}_\star\D x .
    \label{eq: appendix-rho-minus-one}
\end{equation}
Both correlations are averages of signs under probability densities, so $|S^{\tilde\pi}_\star|\Le1$ and $|S^{\xi}_\star|\Le1$, whence
\begin{equation*}
	\int_{\mathbb R^d}\!\tilde\pi\,w_\star\bigl|S^{\tilde\pi}_\star\bigr|\D x \Le \mathbb E_{\tilde\pi}[w_\star] < \infty ,
	\qquad
	\int_{\mathbb R^d}\!\xi\,w_\star\bigl|S^{\xi}_\star\bigr|\D x \Le \mathbb E_{\xi}[w_\star] < \infty
\end{equation*}
by the two hypotheses of Theorem~\ref{thm: accuracy}. The right side of \eqref{eq: appendix-rho-minus-one} therefore converges absolutely.

It remains to relate $\KL(\pi\|\tilde\pi)$ to $\KL(\pi\|\nu_\star)$. Taking logarithms in $\nu_\star=\tilde\pi\varrho$ gives $\log(\pi/\tilde\pi) = \log(\pi/\nu_\star) + \log\varrho$, and integrating against $\pi$,
\begin{equation*}
	\KL(\pi\|\tilde\pi) = \KL(\pi\|\nu_\star) + \int_{\mathbb R^d}\!\pi\log\varrho\D x .
\end{equation*}
In the last integral $\pi = \nu_\star w_\star = \tilde\pi\,\varrho\,w_\star$, so
\begin{equation}
    \KL(\pi\|\tilde\pi) = \KL(\pi\|\nu_\star) + \int_{\mathbb R^d}\!\tilde\pi\,w_\star\,\varrho\log\varrho\D x .
    \label{eq: appendix-kl-split}
\end{equation}
The inequality $\varrho\log\varrho\Ge\varrho-1$ for $\varrho>0$ bounds the last integral of \eqref{eq: appendix-kl-split} below by the left side of \eqref{eq: appendix-rho-minus-one}, which is finite, so both summands of \eqref{eq: appendix-kl-split} are well defined in $(-\infty,+\infty]$ and
\begin{equation*}
    \KL(\pi\|\nu_\star) \Le \KL(\pi\|\tilde\pi) - \int_{\mathbb R^d}\!\tilde\pi\,w_\star(\varrho-1)\D x.
\end{equation*}
Substituting \eqref{eq: appendix-varrho} into the remaining integral returns the right side of \eqref{eq: appendix-rho-minus-one}, and Lemma~\ref{lem: appendix-correlation}, applied at $\omega=\tilde\pi$ and at $\omega=\xi$ with $\nu=\nu_\star$, absorbs the factor two in each term, giving \eqref{eq: accuracy-bound}.
\end{proof}

Identity \eqref{eq: appendix-rho-minus-one} also settles how much \eqref{eq: accuracy-bound} can subtract. Since $\tilde\pi\varrho=\nu_\star$ and $\nu_\star w_\star=\pi$, the left side of \eqref{eq: appendix-rho-minus-one} evaluates exactly, and Lemma~\ref{lem: appendix-correlation} turns its right side into the two discrepancies, so the two subtracted terms sum to $1-\mathbb E_{\tilde\pi}[w_\star]$, the value quoted after Theorem~\ref{thm: accuracy}. Both are nonnegative, which forces $\mathbb E_{\tilde\pi}[w_\star]\Le1$, and $\mathbb E_{\tilde\pi}[w_\star]>0$ then leaves the subtraction strictly less than $1$. Equality holds at $\lambda=\vartheta=0$, and when either coefficient is positive only in the degenerate case $\tilde\pi=\nu_\star=\pi$: the equality case of Lemma~\ref{lem: appendix-correlation} makes $w_\star$ constant $\tilde\pi$-almost everywhere when $\lambda>0$ and $\xi$-almost everywhere when $\vartheta>0$, both densities are positive, and $\int\pi=\int\nu_\star=1$ then forces $w_\star\equiv1$, whence $S^{\tilde\pi}_\star=S^{\xi}_\star=0$ and \eqref{eq: surrogate-stationary} gives $\nu_\star=\tilde\pi=\pi$.

\begin{remark}
\label{rem: appendix-domination}
Two remarks on the hypotheses. The following condition implies both of them, and is recorded for that reason, although it excludes the coefficients $(\lambda,\vartheta)=(1,1)$ used in every KLXX run. Suppose the mixture is dominated by the surrogate, $\xi\Le c\,\tilde\pi$ pointwise for some constant $c$, and $\lambda+c\vartheta<\tfrac12$; note that $\xi\Le c\tilde\pi$ implies $c\Ge1$, since both densities integrate to one. Both correlations lie in $[-1,1]$, being averages of signs under probability densities, so \eqref{eq: appendix-varrho} gives $\varrho\Ge1-2\lambda-2c\vartheta>0$, that is $\nu_\star\Ge(1-2\lambda-2c\vartheta)\tilde\pi>0$. Dividing $\pi$ by this bound and integrating gives $\mathbb E_{\tilde\pi}[w_\star]\Le(1-2\lambda-2c\vartheta)^{-1}$ and, using the domination once more, $\mathbb E_{\xi}[w_\star]\Le c\,(1-2\lambda-2c\vartheta)^{-1}$. The condition is sufficient but conservative, and positivity does not in fact require it. If $\nu_\star(x)$ approaches zero at some $x$ with $\pi(x)>0$, then $w_\star(x)\to\infty$, so $w_\star(x)$ exceeds $w_\star(y)$ for almost every $y$ and both $S^{\tilde\pi}_\star(x)$ and $S^{\xi}_\star(x)$ tend to $+1$; the right-hand side of \eqref{eq: surrogate-stationary} then tends to $\tilde\pi(x)(1+2\lambda)+2\vartheta\xi(x)>0$. A solution cannot therefore approach zero where $\pi$ is positive, at any $\lambda,\vartheta\Ge0$; this indicates that positivity is not restrictive, but it is not an existence proof, and the theorem assumes a positive solution.
\end{remark}

\section{Proof of the propagation factor analysis}
\label{sec: appendix-propagation}

Section~\ref{sec: propagation-analysis} states the error bound and the propagation factor $F_\Sigma$ read from it. This appendix proves Theorem~\ref{thm: propagation} for the inference scheme of Algorithm~\ref{alg: inference}, with a general Markov rejuvenation kernel.

The accepted schedule from $\pi_0$ to $\pi_K$ and the trained maps $G_1,\dots,G_K$ are held fixed. The theorem concerns inference with a trained Boltzmann generator, on a fresh particle set, with no further selection or training. The proof proceeds one stage at a time. It uses the triple $(\pi_k,\nu_k,w_k)$ and the identities \eqref{eq: appendix-transfer} that relate consecutive stages. The only property of the rejuvenation kernel that enters is invariance.

\subsection{Notations and what one stage does}
\label{subsec: appendix-notation}

We write $\lVert\cdot\rVert_2$ for the $L^2$ norm over all the randomness of the scheme, namely the $N$ initial draws and, at every stage, the resampling draws and the kernel draws. As in Section~\ref{sec: propagation-analysis}, $\mu\langle \varphi\rangle:=\int_{\mathbb R^d}\varphi\,\D\mu$ denotes the integral. The Markov kernel $Q_k$ acts on a bounded measurable function from the left side and on a distribution from the right side,
\begin{equation}
    Q_kh(x) := \int_{\mathbb R^d} h(y)\,Q_k(x,\D y) ,
    \qquad
    (\pi_kQ_k)(A) := \int_{\mathbb R^d} Q_k(x,A)\,\pi_k(\D x) ,
    \label{eq: appendix-kernel}
\end{equation}
so that the invariance assumed of $Q_k$ reads $\pi_kQ_k=\pi_k$. Each $G_k$ is a diffeomorphism of the domain onto itself, so $\nu_k=(G_k^{-1})_{\#}\pi_{k-1}$ has a density whenever $\pi_{k-1}$ does, and $w_k=\pi_k/\nu_k$ is a ratio of densities. Since $w_k>0$, the measures $\nu_k$ and $\pi_k$ are equivalent. As in Section~\ref{sec: propagation-analysis}, $\lVert w_k\rVert_\infty$ is the $\nu_k$-essential supremum of $w_k$, and the essential infimum and supremum of a bounded observable at stage $k$ are taken with respect to $\pi_k$. Write the two empirical measures produced at each stage as
\begin{equation}
    \pi_k^N := \frac1N\sum_{i=1}^{N}\delta_{X_k^i}
    \quad (k=0,1,\dots,K),
    \qquad
    \nu_k^N := \frac1N\sum_{i=1}^{N}\delta_{Y_k^i}
    \quad (k=1,\dots,K) .
    \label{eq: appendix-empirical}
\end{equation}
Here $\pi_k^N$ is the empirical measure of the particles leaving stage $k$, and $\nu_k^N$ is the empirical measure of the pushforward particles that stage $k$ reweights. In particular $\pi_0^N$ is the empirical measure of the $N$ initial draws, and $\pi_K^N$ is the estimate the scheme returns. The proof bounds the $L^2$ distance between each empirical measure and the corresponding exact distribution.

Each stage maps $\pi_{k-1}^N$ to $\pi_k^N$ in three steps, which contribute two terms in the recurrence \eqref{eq: appendix-recursion} proved in Appendix~\ref{subsec: appendix-proof}. The \emph{flow pushforward} is deterministic, because $G_k^{-1}$ adds no randomness:
\begin{equation}
    \nu_k^N = (G_k^{-1})_{\#}\pi_{k-1}^N ,
    \qquad
    \nu_k = (G_k^{-1})_{\#}\pi_{k-1} ,
    \label{eq: appendix-pushforward}
\end{equation}
the second identity being the definition of $\nu_k$. The \emph{importance sampling} contributes both terms of the recurrence. Weighting by $w_k$ is deterministic, but it distorts the measure by an amount controlled by the error already present; resampling then adds a Monte Carlo error of order $N^{-1/2}$. The \emph{rejuvenation}, together with the resampling, is one independent draw per particle. Given the pushforward particles, the conditional mean of the result is the weighted average of $Q_kh$, written $\mathcal A_k(Q_kh)$ in \eqref{eq: appendix-weighted} after Lemma~\ref{lem: onestep}.

By the definition of $(G_k^{-1})_{\#}$, integrating $h$ against $\nu_k$ is integrating $h\circ G_k^{-1}$ against $\pi_{k-1}$, and the same holds for the empirical measures because $Y_k^i=G_k^{-1}(X_{k-1}^i)$; taking $h$ to be an indicator matches the null sets of the two distributions, so $h$ and $h\circ G_k^{-1}$ have the same essential range. Hence
\begin{equation}
    \nu_k^N\langle h\rangle = \pi_{k-1}^N\langle h\circ G_k^{-1}\rangle,
    \qquad
    \nu_k\langle h\rangle = \pi_{k-1}\langle h\circ G_k^{-1}\rangle,
    \qquad
    \operatorname{osc}(h\circ G_k^{-1}) = \operatorname{osc}(h) ,
    \label{eq: appendix-transfer}
\end{equation}
for every bounded measurable $h$. In the third identity, the oscillation of $h$ is with respect to $\nu_k$ and that of $h\circ G_k^{-1}$ is with respect to $\pi_{k-1}$; they agree because the pushforward matches null sets. An error incurred before stage $k$ is therefore transferred to stage $k$ with the same oscillation.

Two facts will be used repeatedly. The first concerns the location of the particles. Every $X_0^i$ is drawn from $\pi_0$. If the law of every $X_{k-1}^i$ is dominated by $\pi_{k-1}$, then the law of every $Y_k^i=G_k^{-1}(X_{k-1}^i)$ is dominated by $\nu_k$. Resampling only selects among the $Y_k^i$. If $A$ is $\pi_k$-null, then $\pi_kQ_k(A)=\pi_k(A)=0$, so $Q_k(x,A)=0$ for $\pi_k$-almost every $x$, and rejuvenation preserves this domination. By induction on $k$, almost surely no particle lies in a null set of its own stage, and essential suprema and infima may be used for particle averages as for the exact averages. The second fact is that $Q_k$ is a Markov kernel leaving $\pi_k$ invariant, and that this is the only property of $Q_k$ used below. Invariance means the kernel does not change a $\pi_k$-average,
\begin{equation}
    \pi_k\langle Q_kh\rangle = \pi_k\langle h\rangle
    \qquad\text{for every bounded measurable } h ,
    \label{eq: appendix-invariance}
\end{equation}
and $Q_kh(x)$ is an average of values of $h$ over a distribution charging no $\pi_k$-null set, so
\begin{equation}
    \operatorname{ess\,inf}h \Le Q_kh(x) \Le \operatorname{ess\,sup}h
    \quad\text{for } \pi_k\text{-almost every } x ,
    \qquad\text{hence}\qquad
    \operatorname{osc}(Q_kh) \Le \operatorname{osc}(h) .
    \label{eq: appendix-contraction}
\end{equation}
The first identity makes the error decomposition of Appendix~\ref{subsec: appendix-proof} exact. The second prevents the inherited error from growing.

\subsection{Proof of Theorem~\ref{thm: propagation}}
\label{subsec: appendix-proof}

The proof is an induction on the stage index. Two facts about the importance weight $w_k$ are used throughout. First, $\nu_k\langle w_k\rangle = 1$, because $w_k=\pi_k/\nu_k$ and $\pi_k$ is a probability distribution. Second, $\nu_k^N\langle w_k\rangle$ lies almost surely in $(0,\lVert w_k\rVert_\infty]$, by the null sets recorded in Appendix~\ref{subsec: appendix-notation}, so no normalization vanishes or diverges.

The next lemma bounds the error of the weighted average $\nu_k^N\langle w_kh\rangle/\nu_k^N\langle w_k\rangle$ in terms of the errors of $\nu_k^N$ against $\nu_k$. To control it, we center $h$ at its exact mean and weight the result, obtaining the function $g_k$ below. Then $\nu_k\langle g_k\rangle=0$, and both the numerator error and the denominator error become ordinary deviations of $\nu_k^N$ from $\nu_k$.

\begin{lemma}
\label{lem: onestep}
Let $h$ be bounded and measurable, and define $g_k := w_k\bigl(h-\pi_k\langle h\rangle\bigr)$. Then $\nu_k\langle g_k\rangle=0$, $\operatorname{osc}(g_k)\Le2\lVert w_k\rVert_\infty\operatorname{osc}(h)$, and, almost surely,
\begin{equation}
    \biggl|\frac{\nu_k^N\langle w_kh\rangle}{\nu_k^N\langle w_k\rangle}-\pi_k\langle h\rangle\biggr|
    \Le \bigl|\nu_k^N\langle g_k\rangle-\nu_k\langle g_k\rangle\bigr|
    + \operatorname{osc}(h)\,\bigl|\nu_k^N\langle w_k\rangle-\nu_k\langle w_k\rangle\bigr| .
    \label{eq: appendix-onestep}
\end{equation}
\end{lemma}

\begin{proof}
Since $w_k=\pi_k/\nu_k$, weighting by $w_k$ converts a $\nu_k$-average into a $\pi_k$-average, so $\nu_k\langle w_kh\rangle=\pi_k\langle h\rangle$, and the same identity holds with any other bounded function in place of $h$. Using $\nu_k\langle w_k\rangle=1$, we obtain
\begin{equation*}
    \nu_k\langle g_k\rangle = \pi_k\langle h\rangle - \pi_k\langle h\rangle\,\nu_k\langle w_k\rangle = 0 .
\end{equation*}
Write $D_k := \nu_k^N\langle w_kh\rangle/\nu_k^N\langle w_k\rangle-\pi_k\langle h\rangle$ and $q_k := \nu_k^N\langle w_k\rangle-1$. Then
\begin{equation*}
    \nu_k^N\langle g_k\rangle = \nu_k^N\langle w_kh\rangle - \pi_k\langle h\rangle\,\nu_k^N\langle w_k\rangle
    = \nu_k^N\langle w_k\rangle\,D_k = D_k(1+q_k) ,
\end{equation*}
hence $D_k = \nu_k^N\langle g_k\rangle-D_kq_k$, and the triangle inequality gives
\begin{equation*}
    |D_k|\Le|\nu_k^N\langle g_k\rangle|+|D_k|\,|q_k| .
\end{equation*}
Both $\nu_k^N\langle w_kh\rangle/\nu_k^N\langle w_k\rangle$ and $\pi_k\langle h\rangle$ are weighted averages of $h$ and so lie in $[\operatorname{ess\,inf}h,\operatorname{ess\,sup}h]$, whence $|D_k|\Le\operatorname{osc}(h)$. Using that bound on the right-hand side only, and keeping $|D_k|$ itself on the left, leaves
\begin{equation*}
    |D_k|\Le|\nu_k^N\langle g_k\rangle|+\operatorname{osc}(h)\,|q_k| ,
\end{equation*}
and inserting $\nu_k\langle g_k\rangle=0$ and $\nu_k\langle w_k\rangle=1$ into the two absolute values gives \eqref{eq: appendix-onestep}.

For the oscillations, $\pi_k\langle h\rangle\in[\operatorname{ess\,inf}h,\operatorname{ess\,sup}h]$ gives $\lVert h-\pi_k\langle h\rangle\rVert_\infty\Le\operatorname{osc}(h)$ and therefore $\lVert g_k\rVert_\infty\Le\lVert w_k\rVert_\infty\operatorname{osc}(h)$, so $\operatorname{osc}(g_k)\Le2\lVert g_k\rVert_\infty\Le2\lVert w_k\rVert_\infty\operatorname{osc}(h)$.
\end{proof}

Lemma~\ref{lem: onestep} motivates defining the weighted average in the $w_k$-reweighted pushforward distribution $\nu_k^N$.
\begin{definition}
Let $W_k^i=w_k(Y_k^i)$. For a bounded measurable function $h$,
\begin{equation}
    \mathcal A_k(h) := \sum_{i=1}^{N}\frac{W_k^i}{\sum_{l=1}^{N}W_k^l}\,h(Y_k^i)
    = \frac{\nu_k^N\langle w_kh\rangle}{\nu_k^N\langle w_k\rangle} ,
    \label{eq: appendix-weighted}
\end{equation}
so that \eqref{eq: appendix-onestep} is $\bigl|\mathcal A_k(h)-\pi_k\langle h\rangle\bigr|$ on the left.
\end{definition}
Applied to $Q_kh$, the number $\mathcal A_k(Q_kh)$ is the conditional mean of $\pi_k^N\langle h\rangle$ given the pushforward particles $Y_k^{1:N}$. The same ratio in \eqref{eq: appendix-weighted}, taken at $\nu_k$ rather than $\nu_k^N$, equals $\pi_k\langle h\rangle$, since $\nu_k\langle w_kh\rangle = \pi_k\langle h\rangle$ and $\nu_k\langle w_k\rangle=1$.

We now prove Theorem~\ref{thm: propagation} by deriving the error propagation recurrence \eqref{eq: appendix-recursion}.

\begin{proof}
Define $\varepsilon_k$ by the worst-case $L^2$ error of the stage-$k$ particles over observables of unit oscillation,
\begin{equation}
    \varepsilon_k := \sup\bigl\{\lVert\pi_k^N\langle h\rangle-\pi_k\langle h\rangle\rVert_2
    \;:\; h\text{ bounded measurable},\ \operatorname{osc}(h)\Le1\bigr\} .
    \label{eq: appendix-eps-def}
\end{equation}
Then $\varepsilon_k\Le1$, since $\pi_k^N\langle h\rangle$ and $\pi_k\langle h\rangle$ both lie in $[\operatorname{ess\,inf}h,\operatorname{ess\,sup}h]$. The definition of $\varepsilon_k$ implies
\begin{equation}
    \lVert\pi_k^N\langle h\rangle-\pi_k\langle h\rangle\rVert_2 \Le \operatorname{osc}(h)\,\varepsilon_k
    \qquad\text{for every bounded } h.
    \label{eq: appendix-eps}
\end{equation}
If $\operatorname{osc}(h)>0$, apply the definition to $\bigl(h-\operatorname{ess\,inf}h\bigr)/\operatorname{osc}(h)$; if $\operatorname{osc}(h)=0$ both sides vanish.
Combining \eqref{eq: appendix-eps} at stage $k-1$ with the pushforward identities \eqref{eq: appendix-transfer} transfers the same bound to the pushforward particles: for every bounded function $h$ and every $k\Ge1$,
\begin{equation}
    \lVert\nu_k^N\langle h\rangle-\nu_k\langle h\rangle\rVert_2
    = \lVert\pi_{k-1}^N\langle h\circ G_k^{-1}\rangle-\pi_{k-1}\langle h\circ G_k^{-1}\rangle\rVert_2
    \Le \operatorname{osc}(h)\,\varepsilon_{k-1} .
    \label{eq: appendix-transferred}
\end{equation}

Next we prove the error-propagation recurrence
\begin{equation}
    \varepsilon_k \Le \frac{1}{2\sqrt N} + 3\lVert w_k\rVert_\infty\,\varepsilon_{k-1},
    \quad\text{with the second term absent at } k=0 .
    \label{eq: appendix-recursion}
\end{equation}
At $k=0$, \eqref{eq: appendix-recursion} reads $\varepsilon_0\Le1/(2\sqrt N)$. Here $\pi_0^N$ is the empirical measure of $N$ independent draws from $\pi_0$,
\begin{equation*}
    \mathbb E\bigl[(\pi_0^N\langle h\rangle-\pi_0\langle h\rangle)^{2}\bigr]
    = \frac{\operatorname{Var}_{\pi_0}(h)}{N} .
\end{equation*}
Then Popoviciu's inequality on variances \citep{popoviciu1935equations} bounds $\operatorname{Var}_{\pi_0}(h)$ by $\frac14\operatorname{osc}(h)^2$. Hence $\lVert\pi_0^N\langle h\rangle-\pi_0\langle h\rangle\rVert_2 \Le \operatorname{osc}(h)/(2\sqrt N)$. Restricting to $\operatorname{osc}(h)\Le1$ gives $\varepsilon_0\Le 1/(2\sqrt N)$.

For $k\Ge1$, let $h$ be bounded with $\operatorname{osc}(h)\Le1$ and recall $\mathcal A_k$ from \eqref{eq: appendix-weighted}. Applying it to $Q_kh$ and using the invariance \eqref{eq: appendix-invariance} on the right, we obtain the exact decomposition
\begin{equation}
    \pi_k^N\langle h\rangle-\pi_k\langle h\rangle
    = \underbrace{\bigl[\pi_k^N\langle h\rangle-\mathcal A_k(Q_kh)\bigr]}_{\text{(I) resampling and rejuvenation error ($N^{-1/2}$)}}
    + \underbrace{\bigl[\mathcal A_k(Q_kh)-\pi_k\langle Q_kh\rangle\bigr]}_{\text{(II) pushforward sampling error (Lemma~\ref{lem: onestep})}} .
    \label{eq: appendix-split}
\end{equation}
Term (I) is the Monte Carlo error of resampling and rejuvenation at stage $k$. Term (II) is the error inherited from previous stages. Invariance replaces $\pi_k\langle h\rangle$ by $\pi_k\langle Q_kh\rangle$, so (II) compares two averages of the same function $Q_kh$. Since \eqref{eq: appendix-split} is an identity, Minkowski's inequality bounds the stage-$k$ $L^2$ error by $\lVert\text{(I)}\rVert_2+\lVert\text{(II)}\rVert_2$. The two terms are estimated separately.

We condition on the pushforward particles $Y_k^{1:N}=(Y_k^1,\dots,Y_k^N)$, which determine the weights $W_k^i=w_k(Y_k^i)$. Given $Y_k^{1:N}$, resampling a particle and then applying $Q_k$ has law $\sum_{j}\bigl(W_k^j/\sum_l W_k^l\bigr)Q_k(Y_k^j,\cdot)$, whose $h$-average is $\mathcal A_k(Q_kh)$. Repeating this separately for each $i$ yields $N$ i.i.d.\ stage-$k$ particles $X_k^1,\dots,X_k^N$.
Hence $\mathbb E\bigl[\pi_k^N\langle h\rangle\,\big|\, Y_k^{1:N}\bigr]=\mathcal A_k(Q_kh)$. Since each $h(X_k^i)$ lies almost surely in $[\operatorname{ess\,inf}h,\operatorname{ess\,sup}h]$, Popoviciu's inequality on variances gives
\begin{equation*}
    \mathbb E\Bigl[\bigl(\pi_k^N\langle h\rangle-\mathcal A_k(Q_kh)\bigr)^{2}\Bigm| Y_k^{1:N}\Bigr]
    \Le \frac{\operatorname{osc}(h)^{2}}{4N} .
\end{equation*}
This bound is almost surely uniform in $Y_k^{1:N}$, so the tower property gives
\begin{equation}
	\text{(I)} \quad
    \bigl\lVert\pi_k^N\langle h\rangle-\mathcal A_k(Q_kh)\bigr\rVert_2 \Le \frac{\operatorname{osc}(h)}{2\sqrt N} .
    \label{eq: appendix-resample}
\end{equation}

For (II), we use Lemma~\ref{lem: onestep} with $Q_kh$ in place of the observable. It is bounded and measurable, so with $g_k := w_k\bigl(Q_kh-\pi_k\langle Q_kh\rangle\bigr)$ we obtain the $L^2$ error bound
\begin{equation*}
    \bigl\lVert \mathcal A_k(Q_kh)-\pi_k\langle Q_kh\rangle\bigr\rVert_2
    \Le \bigl\lVert\nu_k^N\langle g_k\rangle-\nu_k\langle g_k\rangle\bigr\rVert_2
    + \operatorname{osc}(Q_kh)\,\bigl\lVert\nu_k^N\langle w_k\rangle-\nu_k\langle w_k\rangle\bigr\rVert_2 ,
\end{equation*}
where the two norms come from taking $\lVert\cdot\rVert_2$ in \eqref{eq: appendix-onestep} and applying the triangle inequality. Both $g_k$ and $w_k$ are deterministic: $Q_k$ and $\pi_k$ come from the schedule, so both are fixed functions. Then \eqref{eq: appendix-transferred} gives
\begin{equation*}
    \bigl\lVert \mathcal A_k(Q_kh)-\pi_k\langle Q_kh\rangle\bigr\rVert_2
    \Le \bigl[\operatorname{osc}(g_k)+\operatorname{osc}(Q_kh)\operatorname{osc}(w_k)\bigr]\varepsilon_{k-1} .
\end{equation*}
Finally $\operatorname{osc}(g_k)\Le2\lVert w_k\rVert_\infty\operatorname{osc}(Q_kh)$ by Lemma~\ref{lem: onestep}, and $\operatorname{osc}(w_k)\Le\lVert w_k\rVert_\infty$ because $w_k\Ge0$, so the two oscillations contribute $2\lVert w_k\rVert_\infty\operatorname{osc}(Q_kh)$ and $\lVert w_k\rVert_\infty\operatorname{osc}(Q_kh)$ respectively. The contraction \eqref{eq: appendix-contraction} then replaces $\operatorname{osc}(Q_kh)$ by $\operatorname{osc}(h)$, and
\begin{equation}
	\text{(II)} \quad
    \bigl\lVert \mathcal A_k(Q_kh)-\pi_k\langle Q_kh\rangle\bigr\rVert_2
    \Le 3\lVert w_k\rVert_\infty\operatorname{osc}(h)\,\varepsilon_{k-1} .
    \label{eq: appendix-inherited}
\end{equation}
Adding \eqref{eq: appendix-resample} and \eqref{eq: appendix-inherited} and taking the supremum over $\operatorname{osc}(h)\Le1$ yields \eqref{eq: appendix-recursion}.

It remains to unroll the recurrence. Write $a:=1/(2\sqrt N)$ and $b_j:=3\lVert w_j\rVert_\infty$, and induct on $k$ to prove $\varepsilon_k \Le a\sum_{i=0}^{k}\prod_{j=i+1}^{k}b_j$, the empty product being one. At $k=0$ the right-hand side is $a$, which is \eqref{eq: appendix-recursion} at $k=0$. Assuming it at $k-1$, \eqref{eq: appendix-recursion} gives
\begin{equation*}
    \varepsilon_k \Le a + b_k\,a\sum_{i=0}^{k-1}\prod_{j=i+1}^{k-1}b_j
    = a\Biggl[1+\sum_{i=0}^{k-1}\prod_{j=i+1}^{k}b_j\Biggr]
    = a\sum_{i=0}^{k}\prod_{j=i+1}^{k}b_j ,
\end{equation*}
the last equality because the $i=k$ term is the empty product. Taking $k=K$ and applying \eqref{eq: appendix-eps} with $h=\varphi$,
\begin{equation*}
    \bigl\lVert\pi_K^N\langle \varphi\rangle-\pi_K\langle \varphi\rangle\bigr\rVert_2
    \Le \operatorname{osc}(\varphi)\,\varepsilon_K
    \Le \frac{\operatorname{osc}(\varphi)}{2\sqrt N}\sum_{k=0}^{K}\prod_{j=k+1}^{K}3\lVert w_j\rVert_\infty .
\end{equation*}
Squaring gives \eqref{eq: propagation-bound}, concluding the proof of Theorem~\ref{thm: propagation}.
\end{proof}

\end{document}